%% file: main.tex
\documentclass{article}

\usepackage{wrapfig}
\usepackage[utf8]{inputenc} % allow utf-8 input
\usepackage[T1]{fontenc}    % use 8-bit T1 fonts
\usepackage{hyperref}       % hyperlinks
\usepackage{url}            % simple URL typesetting
\usepackage{booktabs}       % professional-quality tables
\usepackage{amsmath, amssymb, amsthm, amsfonts}       % blackboard math symbols
\usepackage{nicefrac, xfrac}       % compact symbols for 1/2, etc.
\usepackage{microtype}      % microtypography
\usepackage{xcolor}         % colors
\usepackage{bm}
\usepackage{graphicx}
\usepackage{algorithm}
\usepackage{algorithmic}
\usepackage{geometry}
\usepackage{authblk,textcomp}
\usepackage{commath}
\usepackage{physics}
\usepackage{libertine}
\usepackage{natbib}
\usepackage{enumerate}
\usepackage{latexsym}
\usepackage{verbatim}
\usepackage{mathtools}
\mathtoolsset{showonlyrefs=true}
\theoremstyle{plain}
\newtheorem{theorem}{Theorem}[section]
\theoremstyle{plain}

\theoremstyle{plain}

\newtheorem{proposition}[theorem]{Proposition}
\newtheorem{lemma}[theorem]{Lemma}
\newtheorem{corollary}[theorem]{Corollary}
\theoremstyle{definition}
\newtheorem{definition}[theorem]{Definition}
\newtheorem{assumption}[theorem]{Assumption}
\theoremstyle{remark}

\theoremstyle{definition}

\counterwithin{conjecture}{result}

\input{commands}
\input{math_commands}

\hypersetup{pdfauthor={},pdftitle={},%
            colorlinks, linktocpage=true, pdfstartpage=1, pdfstartview=FitV,%
    breaklinks=true, pdfpagemode=UseNone, pageanchor=true, pdfpagemode=UseOutlines,%
    plainpages=false, bookmarksnumbered, bookmarksopen=true, bookmarksopenlevel=1,%
    hypertexnames=true, pdfhighlight=/O,%
    urlcolor=orange, linkcolor=blue, citecolor=blue
        }

\title{A Random Matrix Theory Perspective on the Spectrum of \\Learned Features and Asymptotic Generalization Capabilities}

\author[1,2]{Yatin Dandi}
\author[2]{Luca Pesce}
\author[1,5]{Hugo Cui}
\author[2]{Florent Krzakala}
\author[3]{Yue M. Lu}
\author[4]{Bruno Loureiro}

\affil[1]{\small Statistical Physics Of Computation laboratory, 
\'Ecole Polytechnique F\'ed\'erale de Lausanne (EPFL), 1015 Lausanne, Switzerland}
\affil[2]{\small Information Learning \& Physics laboratory,
\'Ecole Polytechnique F\'ed\'erale de Lausanne (EPFL), 1015 Lausanne, Switzerland}
\affil[3]{\small  
John A. Paulson School of Engineering and Applied Sciences, Harvard University}
\affil[4]{\small D\'epartement d'Informatique, \'Ecole Normale Sup\'erieure (ENS) - PSL \& CNRS, 
F-75230 Paris cedex 05, France}
\affil[5]{\small Center of Mathematical Sciences and Applications, Harvard University}
\makeatletter
\newtheorem*{rep@theorem}{\rep@title}
\newcommand{\newreptheorem}[2]{%
\newenvironment{rep#1}[1]{%
 \def\rep@title{#2 \ref{##1}}%
 \begin{rep@theorem}}%
 {\end{rep@theorem}}}
\makeatother

\theoremstyle{plain}
\numberwithin{theorem}{section}
\newreptheorem{theorem}{Theorem}

\theoremstyle{remark}
\date{\today}

\newcommand{\change}[1]{\textcolor{black}{ #1}} 
\begin{document}
\date{}
\maketitle

%%%%%%%%%%%%%%%%%%%%%%%%%%%%%%%%%%%%%%%%%%%%%%%%%%%%%%%%%%%%%%%%%%%%%%%%%%%%%%%
\begin{abstract}
  A key property of neural networks is their capacity of adapting to data during training. Yet, our current mathematical understanding of feature learning and its relationship to generalization remain limited. In this work, we provide a random matrix analysis of how fully-connected two-layer neural networks adapt to the target function after a single, but aggressive, gradient descent step. We rigorously establish the equivalence between the updated features and an isotropic spiked random feature model, in the limit of large batch size. For the latter model, we derive a \emph{deterministic equivalent} description of the feature empirical covariance matrix in terms of certain low-dimensional operators.
  This allows us to sharply characterize the impact of training in the asymptotic feature spectrum, and in particular, provides a theoretical grounding for how the tails of the feature spectrum modify with training.  The deterministic equivalent further yields the exact asymptotic generalization error, shedding light on the mechanisms behind its improvement in the presence of feature learning.
  Our result goes beyond standard random matrix ensembles, and therefore we believe it is of independent technical interest. Different from previous work, our result holds in the challenging maximal learning rate regime, is fully rigorous and allows for finitely supported second layer initialization, which turns out to be crucial for studying the functional expressivity of the learned features. This provides a sharp description of the impact of feature learning in the generalization of two-layer neural networks, beyond the random features and lazy training regimes.
\end{abstract}

%%%%%%%%%%%%%%%%%%%%%%%%%%%%%%%%%%%%%%%%%%%%%%%%%%%%%%%%%%%%%%%%%%%%%%%%%%%%%%%
% MAIN
%%%%%%%%%%%%%%%%%%%%%%%%%%%%%%%%%%%%%%%%%%%%%%%%%%%%%%%%%%%%%%%%%%%%%%%%%%%%%%%

\input{sections/intro}
\input{sections/setting}
\input{sections/det_equiv}
\input{sections/discussion}

\input{sections/conclusion}
\input{sections/acknowledgements}

\newpage 

\appendix
\section*{Supplementary material}
\input{sections/appendix/proof}

% \bibliographystyle{plainnat}
\bibliographystyle{unsrtnat}
\bibliography{biblio}

\end{document}

%% file: commands.tex
\providecommand{\R}{\ensuremath{\mathbb{R}}}

\providecommand{\N}{\ensuremath{\mathbb{N}}}

\providecommand{\abs}[1]{\lvert#1\rvert}

\providecommand{\norm}[1]{\lVert#1\rVert}

\providecommand{\bydef}{\overset{\text{def}}{=}}

\renewcommand{\vec}[1]{#1}

\newcommand{\E}{\mathbb{E}}
\renewcommand{\P}{\mathbb{P}}
\newcommand{\Ea}[1]{\E\left[#1\right]}
\newcommand{\Eb}[2]{\E_{#1}\left[#2\right]}

\DeclareMathOperator{\diag}{diag}

%% file: math_commands.tex
\newcommand{\bA}{{\boldsymbol{A}}}

\newcommand{\bW}{{{\boldsymbol{{W}}}}}

\newcommand{\bV}{{\boldsymbol{V}}}

\newcommand{\bI}{{\boldsymbol{I}}}

\newcommand{\ba}{{\boldsymbol{a}}}

\newcommand{\polylog}{\operatorname{polylog}}
\newcommand{\btheta}{{\boldsymbol{\theta}}}

\newcommand{\bL}{{\boldsymbol{L}}}

\newcommand{\bx}{{\boldsymbol{x}}}

\newcommand{\bX}{{\boldsymbol{X}}}

\newcommand{\by}{{\boldsymbol{y}}}

\renewcommand{\vec}[1]{\boldsymbol{#1}}

  \providecommand{\R}{\mathbb{R}} % Reals
  \providecommand{\N}{\mathbb{N}} % Naturals

  \makeatletter
  \def\sign{\@ifnextchar*{\@sgnargscaled}{\@ifnextchar[{\sgnargscaleas}{\@ifnextchar{\bgroup}{\@sgnarg}{\sgn} }}}
  \def\@sgnarg#1{\sgn\rbr{#1}}
  \def\@sgnargscaled#1{\sgn\rbr*{#1}}
  \def\@sgnargscaleas[#1]#2{\sgn\rbr[#1]{#2}}
  \makeatother

\providecommand{\E}{\mathbb{E}}%
  \providecommand{\cO}{\mathcal{O}}


\newcommand{\speedup}[1]{{\color{gray}(\ifdim #1 pt > 0.3pt #1\else $< #1$\fi{}$\times$)}}
\usepackage{amsfonts}
\usepackage{amsmath}
\usepackage{amsthm}
\usepackage{amssymb}
\usepackage{dsfont}

\usepackage{xcolor}
\usepackage{color}
\usepackage{graphicx}

\usepackage{verbatim}
\usepackage{xspace} %

\usepackage{enumerate}
\usepackage{enumitem}

\makeatletter
\newsavebox{\@brx}
\newcommand{\llangle}[1][]{\savebox{\@brx}{\(\m@th{#1\langle}\)}%
  \mathopen{\copy\@brx\mkern2mu\kern-0.9\wd\@brx\usebox{\@brx}}}
\newcommand{\rrangle}[1][]{\savebox{\@brx}{\(\m@th{#1\rangle}\)}%
  \mathclose{\copy\@brx\mkern2mu\kern-0.9\wd\@brx\usebox{\@brx}}}
\makeatother

\providecommand{\abs}[1]{\left\lvert#1\right\rvert}
\providecommand{\norm}[1]{\left\lVert#1\right\rVert}

  \providecommand{\R}{\mathbb{R}} %
  \providecommand{\N}{\mathbb{N}} %

  \providecommand{\Eb}[1]{{\mathbb E}\left[#1\right] }

  \providecommand{\P}[1]{{\rm Pr}\left.#1\right. }

  \providecommand{\cO}{\mathcal{O}}

\providecommand{\mycomment}[3]{\todo[caption={},size=footnotesize,color=#1!20]{\textbf{#2: }#3}}%
\providecommand{\inlinecomment}[3]{%
  {\color{#1}#2: #3}}%
\newcommand\commenter[2]%
{%
  \expandafter\newcommand\csname i#1\endcsname[1]{\inlinecomment{#2}{#1}{##1}}
  \expandafter\newcommand\csname #1\endcsname[1]{\mycomment{#2}{#1}{##1}}
}

  \definecolor{mydarkblue}{rgb}{0,0.08,0.45}
  \usepackage{hyperref}
  \hypersetup{ %
    pdftitle={},
    pdfauthor={},
    pdfsubject={},
    pdfkeywords={},
    pdfborder=0 0 0,
    pdfpagemode=UseNone,
    colorlinks=true,
    linkcolor=mydarkblue,
    citecolor=mydarkblue,
    filecolor=mydarkblue,
    urlcolor=mydarkblue,
    pdfview=FitH}

\usepackage{url}

%% file: sections/intro.tex
\section{Introduction}

An essential property of neural networks is their capacity to extract relevant low-dimensional features from high-dimensional data. This \textit{feature learning} is usually signaled by an array of telltale phenomena ---  such as the improvement of the test error over non-adaptive methods \citep{Bach2021}, or the lengthening of the tails in the spectra of the network weights and activations \citep{martin2021implicit, martin2021predicting, wang2024spectral}. Yet, a precise theoretical characterization of the learned features, and how they translate into the aforementioned generalization and spectral properties, is still largely lacking, and arguably constitutes one of the key open questions in machine learning theory.

In this work, we provide a rigorous answer to these questions for two layer neural networks trained with a single but large gradient step. More precisely,
\begin{itemize}%[leftmargin=*, noitemsep]
%\looseness=-1
    \item we provide an exact characterization of statistics associated to the learned features, and in particular the achieved test error;
    \item we quantitatively characterize how feature learning results in modified tails in the spectrum of the feature covariance matrix, as illustrated in Fig.\,\ref{fig:spectrum}.
\end{itemize}
Our results provide a sharp mathematical description of feature learning in this context and allow us to explore the interplay between representational learning and generalization. Before exposing our main technical results, we first offer an overview of known results on feature learning (or the lack thereof) in two-layer neural networks, so as to put our work in context.

%\subsection*{Feature learning in two-layer neural networks}
\paragraph{Models ---}
The present manuscript addresses the simplest class of neural network architectures (used in Fig.~\ref{fig:spectrum}), namely fully-connected, shallow two-layer neural networks:
\begin{align}
\label{eq:def:2LNN}
    f( x; W, a) = \frac{1}{\sqrt{p}} \sum_{j=1}^p a_j \sigma(w_j^{\top}x),
\end{align}
where $W = \{w_j, j \in [p] \} \in \mathbb{R}^{p \times d}$ and $a = \{a_{j}, j\in[p]\} \in\mathbb{R}^{p}$ denote the first and second layer weights, respectively, and $\sigma$ is an activation function. Motivated by the lazy regime of large-width networks \citep{jacot2018neural,chizat2019lazy}, the generalization properties of two-layer neural networks have been thoroughly investigated in the simple case where only the second-layer weights $a$ are trained (typically by ridge regression), while the first-layer weights $w = w^{0}$ are fixed at (typically random) initialization. This model, which is equivalent to the \emph{Random Features} (RF) approximation of kernel methods introduced by \cite{rahimi2007random}, is particularly amenable to mathematical treatment. The reason is that, besides being a convex problem, in the asymptotic regime where $n,p=\Theta(d)$ with $d\to\infty$, the random feature map $\varphi(x) = \sigma(w^{\top}x)$ statistically behaves as a \textit{linear} function with additive noise --- a result often referred to as the Gaussian Equivalence Principle (GEP) \citep{goldt_gaussian_2021, hu2022universality, mei_generalization_2022}. While this surprising property makes the problem tractable with random matrix theory arguments, it implies that in this regime random features can learn, at best, a linear approximation of the underlying target function. This sets a benchmark for the fundamental limitation of not adapting the first-layer weights to the data.

\begin{figure}[t]
    \centering
    \includegraphics[width=0.48\textwidth]{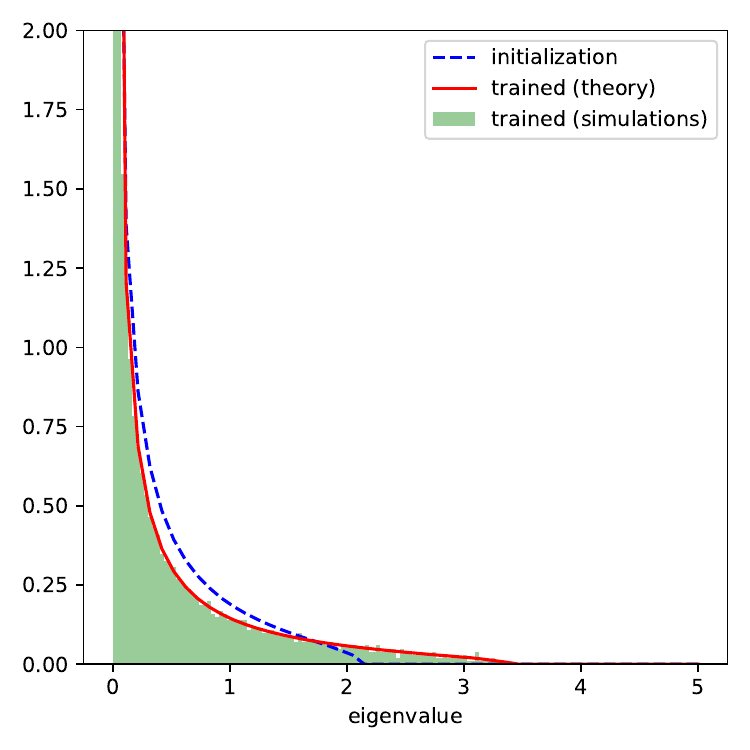}
    \caption{    
    {\bf Bulk spectrum of the empirical features covariance} 
    at initialization (dashed blue) and after training (green); the red line corresponds to the theoretical characterization derived in this manuscript.
    }
    \label{fig:spectrum}
\end{figure}

\paragraph{Gradient descent ---}
Going (literally) one step beyond random features, \cite{ba2022high} showed that training the first layer weights with a single Gradient Descent (GD) step on a batch of data $\{(x_{\mu}, y_{\mu}): \mu\in[n_{0}]\}$ from the same target distribution:
\begin{align}
    \label{eq:def:gdstep}
        w_{j}^1 &= w_{j}^0 - \eta \nabla_{w_{j}}\frac{1}{2n_{0}}\sum\limits_{\mu \in[n_{0}]}(y_{\mu}-f(x_{\mu}; a^{0},w^{0}))^{2} 
\end{align}
followed by ridge regression on the second layer weights $a$ with $n$ fresh samples can drastically change the story above, depending on the scaling of the learning rate $\eta$ where $n \propto d, n_0 \propto d$. More precisely, they showed that for $\eta = \Theta_{d}(\sqrt{d})$, Gaussian Equivalence asymptotically holds as $d\to\infty$, implying that an exact asymptotic treatment based on the GEP still holds. However, in the regime of an aggressive learning rate $\eta = \Theta_d(d)$ (the so-called Maximal Update parametrization \citep{yang2022tensor}) they showed that the first-layer weights adapt to the data distribution, translating into an improvement over the RF lower-bound.

This finding has sparked considerable interest in exactly characterizing the asymptotic generalization error achieved under this single, large step setting. \cite{dandi2023twolayer} proved a novel \textit{Conditional Gaussian Equivalence Principle} (cGEP), and derived more general lower-bounds for the performance of two-layer neural networks after the gradient step. However, these bounds do not provide a fine-grained description of what is learned in the feature learning regime. Reaching such a sharp description instead requires addressing the challenging Random Matrix Theory (RMT) problem of characterizing the non-linear transformation of a highly structured random matrix $\sigma\left(W - \eta G\right)$
with $G$ denoting the gradient matrix,  in the regime $\Theta_d(\norm{\eta G}_F)=\Theta_{d}(||W||_{F})$. \cite{moniri2023theory} provided the first result in this direction for the intermediate learning rate regime $\eta = \Theta_{d}(d^{\nicefrac{1}{2}+\zeta})$, with $\zeta\in (0,\nicefrac{1}{2})$, proving a polynomial GEP together with an exact RMT analysis of the asymptotic generalization error. Their findings, however, do not hold in the more challenging Maximal Update regime $\eta=\Theta(d)$. Leveraging the results on the low-rank approximation of the Gradient matrix \citep{ba2022high} and the cGEP from \citep{ dandi2023twolayer}, \cite{cui2024asymptotics} studied the latter regime approximating the two-layer network by a Spiked Random Features Model (SRFM) with the non-rigorous replica method \citep{Mezard1986}. Though such heuristic arguments are inspirational, they often lack interpretability when compared to other approaches. 

\paragraph{Summary of main results ---} In this work, we provide a rigorous RMT treatment of feature learning after a single GD step in the challenging Maximal Update step size regime. Beyond proving the conjectured results from \cite{cui2024asymptotics}, our analysis extends these findings in several key directions, enabling a quantitative exploration and deeper understanding of fundamental aspects of feature learning, particularly with respect to generalization and spectral properties. Specifically:
\begin{itemize}%[leftmargin=*, noitemsep]
    \item We prove a \emph{deterministic equivalent} description for the empirical feature matrix after the gradient step. This characterization is non-asymptotic in the problem dimensions, and in particular sharply holds in the Maximal-update scaling in the proportional regime where $n,p,\eta=\Theta(d)$ and $n_0=\Omega(d^{1+\epsilon})$ as $d\to\infty$, for some $\epsilon>0$. This result characterizes the spectral properties as well as the generalization error upon updating the second layer.
    \item Our proof proceeds through multiple stages of deterministic equivalences and the asymptotic description of the high-dimensional features through low-dimensional non-linear functions.
    \item We derive an exact asymptotic formula for the generalization error of ridge regression on features updated via a gradient step in the proportional high-dimensional regime, where $n,p,\eta=\Theta(d)$ with $d\to\infty$. Our result offers a rigorous proof of the conjectures in \cite{cui2024asymptotics}, while extending them in several directions, as it applies to finitely supported second-layer initialization and structured first-layer initialization.
    \item Using the deterministic equivalent, we demonstrate how the ``spikes'' in the weights resulting from feature learning alter the entire shape of the feature covariance spectrum and its tail behavior. This observation aligns with the empirical findings of \cite{wang2024spectral} and provides a rigorous foundation for them in the context of our setting.
    \item Finally, we precisely characterize the effect of the  variability in the second-layer initialization on the functional expressivity of the network after one GD step.
\end{itemize}
These findings provide a detailed understanding of the consequences of feature learning in our setting, helping to establish several widely accepted intuitions on a more quantitative and rigorous basis.

\subsection*{Further related works}

 % \paragraph{Spectral aspects of feature learning --} How to ascertain in practice if a neural network has learned meaningful features? Many (mostly empirical) studies have evidenced that feature learning is often accompanied by distinctive spectral phenomena. \cite{martin2021implicit, martin2021predicting, wang2024spectral} have observed that the weights spectrum of many trained neural networks departs from the Gaussianly initialized one (Marchenko-Pastur) and display \textit{heavier tails}. For the same architecture \eqref{eq:def:2LNN} and task as the present work, \cite{wang2024spectral} also empirically observe the emergence of longer tails at the level of the features kernel, when training with the Adam \citep{kingma2014adam} optimizer. In the present manuscript, we quantitatively and rigorously show that a similar phenomenon can be observed, in some cases, after a single large gradient step, and exactly characterize the features spectrum.
\paragraph{Fixed feature methods --} A plethora of works characterized the generalization capabilities of two-layer networks in the high-dimensional regime when the first hidden layer is not trained, with the most prominent example of such fixed feature method being kernel machines \citep{bordelon20a,Canatar2021,Cui2021, Cui2023error,Dietrich1999, Donhauser2021,Ghorbani2019, Ghorbani2020, Opper2001,xiao2022precise}. This class of algorithms is amenable to theoretical analysis and comes with sharp generalization guarantees. 
% To be more precise, if the number of input data scales as $N = \Theta(d^l)$ with $l \in \mathbb{N}$, kernel methods are able to approximate the target $f_\star$ with the best degree-$l$ polynomials. 
However, these methods adapt to relevant low-dimensional structures at much higher sample complexities than fully trained two-layer networks. In simple terms, the sample complexity of kernel methods is not driven by the presence (or lack thereof) of a low-dimensional target subspace. Identical considerations hold for Random Feature Models, where the number of samples $n$ in the generalization guarantees is replaced with $\min(n,p)$, with $p$ being the number of random features \citep{gerace_generalisation_2020, mei_generalization_2022, MEI20223, hu2024asymptotics, aguirre2024random}.

\paragraph{Feature learning --} The discussion above portrays the limitations of fixed feature methods. Inspired by this, a large body of work has studied the sharp separation between the generalization capabilities of such methods versus fully trained two-layer networks that learn features through gradient-based training.
Many of these works fall under the umbrella of the so-called mean-field regime \citep{chizat2018global, mei2018mean, rotskoff2018trainability, sirignano2020mean}. The authors mapped the optimization of two-layer shallow networks onto a convex problem in the space of measures on the weights and paved the way for understanding how features are learned in the high-dimensional regime. The arguably most popular data model in the theoretical community for addressing this question is the multi-index data model with isotropic Gaussian data. This setting has attracted considerable attention in the theoretical community with numerous works that have analyzed the feature learning capabilities of shallow networks trained with gradient-based schemes \citep{arous2021online, abbe2023sgd,ba2023learning,
bardone2024sliding, berthier2023learning, bietti2023learning,
damian2023smoothing, 
dandi2023twolayer,
Paquette2021SGDIT,veiga2022phase, zweig2023symmetric,dandi2024the}.
% In this setting, \cite{arous2021online} for the single-index and \cite{abbe2023sgd} for the multi-index scenario characterized the number of iterations needed for Stochastic Gradient Descent (SGD) to learn the relevant features in the target. These studies highlighted the presence of hard functions for which SGD requires a polynomial number of steps (of the input dimension $d$) for adapting to the low-dimensional target features. 

\paragraph{Deterministic equivalents --}
Deterministic equivalents of large empirical covariance matrices have been extensively studied, beginning with the seminal work of \cite{marchenko1967distribution}. This was extended by \cite{burda2004signal, knowles2017anisotropic} to separable data covariances, and further by \cite{bai2008large, louart2018concentration, chouard2022quantitative} for non-separable covariances. These methodologies have enabled precise asymptotic characterizations of the learning dynamics in single-layer neural networks \citep{louart2018random} and deep random feature models \citep{schroder2023deterministic, schroder2024asymptotics, bosch2023precise}.%\looseness=-1

%% file: sections/setting.tex
\section{Notations and Setting} 
\label{sec:main:setting}
%%%%%%%%%%%%%%%%%%%%%%%%%%%%%%%%%%%%%%%%%%
Consider a supervised learning problem with training data $\mathcal{D}=\{(x_{\mu}, y_{\mu})\in\mathbb{R}^{d+1}, \mu\in[N]\}$. As motivated in the introduction, our goal is to study the problem of feature learning with two-layer neural networks defined in \eqref{eq:def:2LNN}. A widespread intuition in the machine learning literature for why learning is possible despite the curse of dimensionality is that real data distributions typically exhibit low-dimensional latent structures \citep{bellman1957dynamic}. To reflect and model this intuition, we assume our training data have been independently drawn from an isotropic \emph{Gaussian single-index model}:
\begin{align}
\label{eq:def:data}
y_{\mu} = f_{\star}(x_{\mu}) = g(x_{\mu}^{\top}w^{\star}), \qquad x_{\mu}\sim\mathcal{N}(0,I_{d}),
\end{align}
where the \change{unit norm} vector $w^\star$ denotes the target weights and \change{ $g:\mathbb{R}\to\mathbb{R}$} is the link function. \change{ We refer to the appendix for the extension to more general stochastic mapping.} Note that in \ref{eq:def:data}, the high-dimensional covariates $X = (x_{\mu})_{\mu\in[n]}\in\mathbb{R}^{n\times d}$ are isotropic in $\mathbb{R}^{d}$, and therefore the structure in the data distribution is in the conditional distribution of the labels $y|x$, which depends on the covariates only through their projection onto a $1$-dimensional subspace of $\mathbb{R}^{d}$. Therefore, learning features in this model translate to learning the target weight $w^\star$. 

Given a batch of $N$ samples $\mathcal{D} = \{(x_{\mu}, y_{\mu}), {\mu \in [N]}\}$ independently drawn from model \eqref{eq:def:data}, we are interested in studying how our two-layer neural network \eqref{eq:def:2LNN} learns the target feature $w^\star$  through the \emph{Empirical Risk Minimization} (ERM) on the training data. We follow the same procedure as in \citep{ba2022high, dandi2023twolayer, moniri2023theory, cui2024asymptotics} and consider the following two-step training procedure:
\begin{enumerate}%[leftmargin=*, noitemsep]
    \item Let $W^{0}$ and $a^{0}$ denote the first and second layer 
    weights at initialization. Consider a partition of the training data $\mathcal{D} = \mathcal{D}_{0}\cup \mathcal{D}_{1}$ in two disjoint sets of size $n_{0}$ and $n\coloneqq N-n_{0}$, respectively. First, we apply a single gradient step on the square loss for the first-layer weights, keeping the $2^{\rm nd}$ layer ${a}^{0}$ fixed: 
    \begin{align}
        \label{eq:training_step}
         &w_{j}^1 \!\!= w_{j}^0 - \eta g_j^0 \\
         &g_j^0 = \frac{1}{n_0\sqrt{p}} \!\sum_{\mu \in [n_0]}\!\! \!\!\left(f(x_{\mu}; W^{0},a^{0})\! -\!y_\mu\right)a_j^0 x_\mu \sigma^\prime ({w_j^0}^{\top}\!\!x_{\mu}) \notag
    \end{align}
    In this first \textit{representation learning} step, the hidden layer weights adapt to the low-dimensional relevant features from the data.
    \item Given the updated weights $W^{1}$, we update the second-layer weights via  ridge regression on the remaining data $\mathcal{D}_{1}$:
    \begin{align}
        \label{eq:def:main:erm}
        \hat{a}_{\lambda}&=\underset{a\in\mathbb{R}^{p}}{\rm argmin} \sum\limits_{\mu\in[n]}\left(y_{\mu}-f(x_{\mu};a,W^{1})\right)^{2}+\lambda||a||_{2}^{2} \nonumber \\
        &=\left(\nicefrac{\Phi^{\top}\Phi}{p}+\lambda I_{n}\right)^{-1}\nicefrac{\Phi^{\top} y}{\sqrt{p}}
    \end{align}
    where we defined the feature matrix $\Phi \in\mathbb{R}^{n\times p}$ with elements $\phi_{\mu j}=\sigma(x_\mu^{\top}{w^{1}_{j}})$ and the label vector $y = (y_{\mu})_{\mu\in[n]}$.
\end{enumerate}
In the following, we will assume the following initial conditions for the training protocol above:
\begin{assumption}[Initialization]
    \label{ass:init}
     We assume the first layer weights are initialized uniformly at random from the hypersphere $w^{0}_{j}\sim {\rm Unif}(\mathbb{S}^{d-1}(1))$ and the second-layer weights read $a^{0}_{j} = \nicefrac{\tilde{a}^0_j}{\sqrt{p}}$, where the $\{\tilde{a}^0_j, j \in [p]\}$ are $O_d(1)$ scalars initialized i.i.d. by sampling from a dimension-independent vocabulary of size $k$, with probabilities $\pi = (\pi_{q})_{q\in[k]}$.
\end{assumption}
Our goal in the following is two-fold. First, to characterize the properties of the empirical feature matrix $\Phi^{\top}\Phi$. Second, to characterize the generalization error associated with the minimizer of equation \eqref{eq:def:main:erm}, which is defined as:
\begin{align}
\label{eq:def:generror}
    \varepsilon_{\rm gen} = \mathbb{E}_{\mathcal{D}_1} \mathbb{E}_{y_{\rm new}, \vec{x}_{\rm new}} \left[ \left( y_{\rm new} - f(\vec{x}_{\rm new}; W^{1}, \hat{a}_{\lambda})\right)^2\right]
\end{align}
where the expectation is over the joint distribution defined by the model in \ref{eq:def:data}. In particular, we will focus on the proportional high-dimensional regime with Maximal Update scaling, which we formalize in the following assumption.
\begin{assumption}[High-dimensional regime] 
\label{ass:highd}
We assume that $n_0=\Omega(d^{1+\epsilon})$ for some $\epsilon > 0$.
We work under the proportional regime with Maximal Update scaling, defined as the limit where $n,p,\eta, d\to\infty$ at fixed ratios:
\change{\begin{align}\label{eq:joint_limit}
    \alpha \coloneqq \frac{n}{d}, \qquad \beta \coloneqq \frac{p}{d}, \qquad \tilde{\eta}\coloneqq \frac{\eta}{d}
\end{align}}
\end{assumption}

\begin{assumption}[Activation function]
\label{ass:activation}
$\sigma$ is odd, uniformly Lipschitz such that $\sigma'', \sigma'''$ exist almost surely and
are bounded in absolute value by some constant $C$
almost surely with respect to the Lebesgue measure.
Furthermore, $g$ is uniformly bounded and Lipschitz with $\Eb{z \sim \mathcal{N}(0,1)}{g(z)}=0$ and $\Eb{z\sim \mathcal{N}(0,1)}{g'(z)} \neq 0$.
\end{assumption}

\section{Asymptotics of the First Gradient Step} 
We introduce in this section our first result that establishes a rigorous framework for studying the exact asymptotics after one GD step. This question has been the subject of intense theoretical scrutiny in recent years. First, \cite{ba2022high} proved that in the high-dimensional regime, specified in Assumption~\ref{ass:highd}, the hidden layer weights after one GD step are approximately low rank:
\begin{align}
\label{eq:main:rank_one_approx}
     % W^1 &= W^0  + uv^\top + \underbrace{\frac{\eta}{n} \diag\{\sqrt p a^0\} \sigma'_{>1}(W^0 X) \diag\{y\} X^T}_{\Delta_1} 
     W^1 &= W^0  + uv^\top + \Delta
\end{align}
where the spiked structure is identified by: i) $u = \nicefrac{\eta c_1c^\star_1 a^0}{\sqrt{p}}$ is proportional to the second layer at initialization $a^0$, the learning rate $\eta$, and the first Hermite coefficients of the network activation $\sigma$ and the target activation $g$, i.e, $c_1 = \Eb{\xi \sim \mathcal{N}(0,1)}{\sigma(\xi)\xi}$, $c^\star_1 = \Eb{\xi \sim \mathcal{N}(0,1)}{g(\xi)\xi}$; ii) $v$ lies along the first Hermite coefficient of the target $f_\star(\cdot)$, i.e., $v \propto \sum_{\mu \in [n]}y_\mu x_\mu$. The spike component $v$ is correlated with the target vector $w^\star$ in eq.~\eqref{eq:def:data}. It is precisely the presence of such correlated components that enables the trained model to surpass the Random Features performance \cite{damian2022neural,abbe2022merged,abbe2023sgd,dandi2023twolayer}. The ``noise'' term $\Delta$, arises from higher-order components of the activations (see details in \ref{sec:app:mappting_sfrm})
% The ``noise'' term $\Delta_1$, depends on higher order components of the activation, namely $\sigma^\prime(\cdot) = \Eb{z \sim \mathcal{N}(0,1)}{z\sigma(z)} + \sigma^\prime_{>1}(\cdot)$.
\\ The decomposition in eq.~\eqref{eq:main:rank_one_approx} underlies the analysis of \cite{ba2022high, moniri2023theory, cui2024asymptotics} in the propotional regime $p \propto d,p \propto n$.
Heuristically, one could hope to analyze the trained weights by mapping the problem to a Spiked Random Feature Model (SRFM), where the weights $F$ in a feature map $\sigma(Fx)$ --- can be decomposed as the sum of a random bulk $(F_0)$ and a spike:
\begin{align}\label{eq:spiked}
    F = F_0 + u v^\top
\end{align}
% SFRMs have been proposed in as a model for neural networks after one aggressive GD step. 
Along these lines, \cite{cui2024asymptotics} approximate the noise term $\Delta$ in eq.~\eqref{eq:main:rank_one_approx} as an isotropic Gaussian matrix to reach an asymptotic description of the equivalent SRFM model using non-rigorous tools from Statistical Physics. 
It however remained unclear whether
the uniform/Gaussian isotropic description of $W^0 + \Delta$ in eq.~\eqref{eq:main:rank_one_approx} is an accurate approximation of the bulk in the actual GD step. We offer a rigorous answer to the validity of the approximation. 
\paragraph{Anisotropic bulk covariance --} We provide the asymptotic description of the covariance for the bulk weights in eq.~\eqref{eq:main:rank_one_approx} after one GD step.
% \begin{align}
% \label{eq:main:bulk_weight_cov}
%     \Sigma_{\tilde{W}} = \mathbb{E}_{w_i^0}[\tilde{w}^0_i (\tilde{w}^0_i)^\top].
% \end{align} 
For $\nicefrac{n_0}{d} = \Theta(1)$, we unveil the presence of \textit{anisotropic} components that contrast with the uniform approximation considered in \cite{cui2024asymptotics}, which corresponds to taking a diagonal approximation of the covariance. We refer to Appendix~\ref{app:extra} for formal results and additional investigations.%\looseness=-1
\paragraph{Diagonal approximation regime --} On the other hand, we provably show that the diagonal approximation considered in \cite{cui2024asymptotics} for the covariance of the bulk weights is valid as soon as the number of samples $n_0$ used in the GD step is sufficiently large ($n_0 = \Theta(d^{1+\epsilon})$ for any $\epsilon >0$). In this regime,   the spike $v$ in Equation \ref{eq:main:rank_one_approx} can be further replaced by the signal $w^*$:
\begin{lemma}
\label{lemma:SRFM}
    Let $W^{(1)} \in \mathbb{R}^{p \times d}$ denote the weight matrix after the first gradient step. Then,
    under Assumptions \ref{ass:init}, \ref{ass:highd}, and \ref{ass:activation}:
    \begin{equation}
        \norm{W^{(1)}-\left(W^{(0)}+ u (w^\star)^\top \right)}_2  \xrightarrow[d \rightarrow \infty]{a.s} 0.
    \end{equation}
    where $w^\star$ is the target vector in eq.~\eqref{eq:def:data}, and $u = \nicefrac{\eta c_1c^\star_1 a^0}{\sqrt{p}}$ as defined in eq.~\eqref{eq:main:rank_one_approx}, with $\eta$ being the learning rate.
\end{lemma}
From eq.~\eqref{eq:main:rank_one_approx}, we see that the finite support assumption for the second layer (\ref{ass:init}) translates to finite support of the entries $u_i$ and we denote with  $A_{u} = \{\zeta^{u}_1, \dots, \zeta^{u}_{k}\}$ its vocabulary with the corresponding probabilities $\pi = (\pi_{q})_{q\in[k]}$.
The above Lemma rigorously characterizes the regimes in which the isotropic SRFM approximation of \cite{cui2024asymptotics} is justified and correctly describes the network after one GD step.

%% file: sections/det_equiv.tex
\section{Main Results }\label{sec:main:det_equiv}
We are now in the position to state our main technical results. Namely, a rigorous characterization of the empirical feature matrix after one gradient step through a \textit{deterministic equivalent} description. This picture enables the characterization of the features spectrum and the resulting asymptotic generalization error.
\paragraph{Extended features and resolvent --}
 As is well established in random matrix theory, the resolvent 
$G(z) = (\frac{1}{p}\Phi^\top\Phi -zI)^{-1}$, where  $\Phi = \sigma(X (W^1)^\top)$, allows the extraction of a large class of summary statistics related to the spectrum of $\Phi$ \citep{bai2008large,anderson2010introduction}.
To additionally characterize the generalization error, and capture the mean dependence of $\{\phi_\mu, \mu \in [n]\}$ on the spike components $\kappa_\mu \bydef  x_\mu^\top w^\star$,  we construct an augmented version of $G(z)$.
We find that the relevant statistics in our setup are captured by the resolvent of certain \textit{extended} features that we introduce below:

\begin{definition}[Extended resolvent] 
\label{def:resolvent} Let $(X,y)$ denote a batch of data drawn from the Gaussian single-index model in \ref{eq:def:data}, and consider the feature matrix $\Phi = \sigma(X (W^1)^\top)$ after the first gradient step (Eq. \ref{eq:training_step}), with $\kappa_\mu = x_{\mu}^\top w^{\star}$. Let $s_q$ denote the subset of coordinates such that $u_j = \zeta^u_q$ and the ``mean'' $\bar{\phi}^q_\mu = \frac{1}{\abs{s_q}}\sum_{j \in s_q} \Phi_{\mu,j}$. 
We define the \emph{extended features} $\phi^e_\mu \in \mathbb{R}^{(p+k+1)}$ and the  \emph{extended resolvent} ${G}_e(z) \in \mathbb{R}^{(p+k+1)\times (p+k+1)}$ for , $z \in \mathbb{C}/\mathbb{R}^+$ as: 
\begin{equation}\label{eq:surr_feat}
    \phi^e_\mu = \begin{pmatrix}
        y_\mu\\
        \bar{\phi}_\mu \\
        \tilde{\phi}_\mu
    \end{pmatrix}, \  {G}_e(z) \bydef \left(\frac{(\Phi^e)^{\top}\Phi^e}{p}-z{I}\right)^{-1}
\end{equation}
where $\tilde{\phi}_{\mu,j} = \phi_{\mu,j}-\bar{\phi}_j, \Phi^e = \{\phi^e_{\mu} \in \mathbb{R}^{p+k+1}, \mu \in [n]\}$
\end{definition}

A few comments about the definition of the extended features \emph{extended features} $\phi^e_\mu$ and the resolvent ${G}_e(z)$ in \ref{def:resolvent} are in place. 
First, due to the extensive spike and the finite support over u by Assumption \ref{ass:init},  each subset $s_q$ possesses a non-zero mean $\bar{\phi}^q_\mu = \frac{1}{\abs{s_q}}\sum_{j \in s_q} \Phi_{\mu j}$, asymptotically converging to $c_0(\kappa_\mu,\zeta^u_q)$ defined below.
\begin{definition}[Shifted Hermite coefficient] 
\label{def:main:shifted_hermite}
We define the shifted Hermite coefficient $c_{\ell}(\kappa,\zeta)$ of the activation $\sigma(\cdot)$
\begin{align}
    \label{eq:main:shifted_hermite}
    c_{\ell}(\kappa,\zeta) &= \mathbb{E}_{z \sim \mathcal{N}(0,1)}[\sigma(z + \kappa \zeta) h_\ell(z)]
\end{align}
where $(h_\ell)_{\ell>0}$ denote the Hermite polynomials.
\end{definition}

Therefore, unlike typical random-matrix ensembles, the means $\bar{\phi}^q_\mu$ have $\mathcal{O}(1)$ fluctuations due to a large dependence on $\kappa$. Note that the means
$\bar{\phi}^q_\mu$ can be equivalently described as the projections of $\phi_\mu$ along the directions
$(e^1,\cdots,e^k)$ defined as:
\begin{align}
\label{eq:def:eq}
e^{q}_{j} = \frac{1}{\sqrt{p}}
    \begin{cases}
        1 & \text{ if } u_{j} = \zeta^{u}_{q}\\ 
        0 & \text{otherwise}
    \end{cases}, \qquad j\in[p], \quad q\in[k].
\end{align} 
By decomposing $\phi_\mu$ as $\sum_{q=1}^k \bar{\phi}^q_\mu  e^{q} + \tilde{\phi}_\mu$, we realize that the first term varies only along a $k$-dimensional subspace with variations governed by $\kappa_\mu$, while the second term contributes to the ``bulk" of the feature covariance.
The surrogate form $\phi^e_\mu$ splits the features into $\bar{\phi}_\mu, \tilde{\phi}_\mu$ precisely to account for these different scales of fluctuations. To exclude degeneracies arising from the leading-order contributions from $\bar{\phi}$, we introduce the following additional assumption.

\begin{assumption}[non-degeneracy]\label{ass:non-deg}
 The activation $\sigma$ and the vocabulary over $u$, $A_u = \{\zeta^u_1, \cdots, \zeta^u_k\}$ are such that the set of vectors $[c_1(\kappa,\zeta^u_1), \cdots c_1(\kappa,\zeta^u_k)]$ span $\mathbb{R}^k$ as $\kappa$ varies over $\mathbb{R}$.
\end{assumption}

\paragraph{Deterministic Equivalent --}
The extended resolvent ${G}_e(z)$ is a high-dimensional random matrix, inheriting the randomness from the training data $(X, y)$ and the initialization weights $W^0$. 
To reach the deterministic equivalent ${\mathcal{G}}_e$ for the above extended resolvent ${G}_e(z)$, our proof proceeds by subsequently addressing and removing the randomness over the data ${X}$, and the weights $W^0$, eventually obtaining an equivalent dependent only on the coefficients $u_i$ and the projections of $W^0$ on $w^\star$ denoted as $\theta \coloneqq W^0 w^\star \in \mathbb{R}^p.$
This ``special" dependence on $u_i,\theta_i$ is expected, since by Lemma \ref{lemma:SRFM}, these determine the component along the spike in the updated matrix $W^1$, while the remaining directions in the weights maintain isotropic dependence and are averaged out.
The resulting description is characterized through low-dimensional kernels and functions. Concretely, we show that for a large class of functions $\mathcal{F}$ associated to the feature matrix covariance $\Phi^\top\Phi$ and labels $y \in \R^{n}$, with entries $\{y_\mu\}_{\mu=1}^n$, $
   \mathcal{F}(\Phi^\top\Phi, y) \xrightarrow{a.s}   \mathcal{F}^\star(\theta,u)$
In contrast to the high dimensional matrices $(X,W^1)$, $\mathcal{F}^\star(\theta,u)$ depends only on sequences of scalars $(\theta,u)$, turning $\mathcal{F}^\star(\theta,u)$ into finite-dimensional expectations.

% While our proof technique can accommodate bulk matrices with a broad class of spectral distribution, we consider $W$ with rows sampled from the sphere for simplicity. Moreover, without loss of generality, we fix the row norms to unity, since different scale factors could be absorbed in the spike $u$.
% indeed the ``bulk norm'' parameter defining the mapping from one gradient step to equivalent sRFM and the rescaled learning rate $(\Tilde{\eta})$ can be absorbed in the quantity $ u$, see eq.~\eqref{eq:main:rank_one}. 
As stated in Assumption \ref{ass:init}, we consider a finitely supported second layer, leading to the entries of ${u}$ being supported on finitely-many values $A_{u}=\{\zeta^{u}_{1},\cdots, \zeta^{u}_{k}\}$ with probabilities $\pi = (\pi_{q})_{q\in[k]}$. From Lemma \ref{lemma:SRFM}  and equation \ref{eq:main:rank_one_approx}, we see that neurons with identical values of $u_i$ contain identical contributions from the spike. This 
leads to the deterministic equivalent ${\mathcal{G}}_e$ of the extended resolvent inheriting a block structure, with blocks corresponding to different values of $u_i$. 
Let $p_1,\cdots, p_k$ denote the number of neurons with $u_i$ taking values $\zeta^{u}_{1},\cdots, \zeta^{u}_{k}$ respectively. Then, by the strong law of large numbers
 $\frac{p_q}{p} \xrightarrow{a.s} \pi_q$ as $p \rightarrow \infty$.
Without loss of generality, we assume that the neurons are arranged such that:
\begin{equation}\label{eq:neuron}
    [u_1,\cdots, u_p]= [\zeta^{u}_{1} {1}_{1\times p_1},  \zeta^{u}_{2} 1_{1 \times p_2}, \dots \zeta^{u}_{k} 1_{1\times p_k}]
\end{equation}

To compactly express this block structure, we introduce a notation for block-structured matrices and vectors:
% \\ \lp{{\bf TO-DO}: cancel?
% Therefore, the mean of ${\varphi}_{\rm eq}$, denoted as $ c_0$ and defined in eq.~\eqref{eq:main:c0} belongs to the span of a finite number of directions 
% ${e}^1, \cdots, {e}^k\in\mathbb{R}^{p}$ defined as: 
% \begin{align}
% \label{eq:def:eq}
% e^{q}_{j} = \frac{1}{\sqrt{p}}
%     \begin{cases}
%         1 & \text{ if } u_{j} = \zeta^{u}_{q}\\ 
%         0 & \text{otherwise}
%     \end{cases}, \qquad j\in[p], \quad q\in[k]
% \end{align}
% We further define the matrix ${E}\in\mathbb{R}^{p\times k}$ with entries $e^{q}_{j}$. These privileged directions require a separate treatment in our analysis. The other directions will be unifyingly addressed, and we introduce the matrix ${P} \in \mathbb{R}^{(p-k) \times p}$,  with rows spanning the orthogonal complement to $ e^1, \cdots,  e^k$. 
% }

% where the matrices $D', \bL\in\mathbb{R}^{(p+1)\times (p+1)}$ are given by:
% \begin{align}
%     {L} = \begin{bmatrix}
%        {0}_{k+1,k+1} & {0}_{k+1, p}\\
%         {0}_{p, k+1}&  {P
%         } {C}\odot {W} {W}^{\top} {P}^\top
%     \end{bmatrix},&&
%     {D}' = \begin{bmatrix}
%         {0}_{k+1, k+1} & {0}_{k+1, p} \\
%         {0}_{p,k+1} & BDB^{\top}
%     \end{bmatrix},
% \end{align}
% with ${C} = \E_\kappa\left[{c_1(\kappa, {u})c_1(\kappa, {u})^\top}\right]$ and ${D}$ is a diagonal matrix with elements $D_{ii} = \Eb{\kappa}{\sum_{\ell\ge 2} c_{\ell}^2(\kappa, u_i)}$, and the coefficients $c_{\ell}(\kappa, u)$ are defined in \ref{eq: cGET_phieq}. 

\begin{definition}
\label{def:block_ext}
Let $p_1, \cdots p_k$ be the sequence defined above with $\frac{p_q}{p} \xrightarrow{a.s} \pi_q$. Let $C \in \mathbb{R}^{k \times k}$ be a fixed matrix. We define the extended matrix ${C}_e$ as:
\begin{align*}
    {C}_e &= \begin{pmatrix}
      C_{11} {1}_{p_1 \times p_1}, &\cdots  & \cdots &C_{1k} {1}_{p_1 \times p_k} \\
       C_{21} {1}_{p_2 \times p_1}, & C_{22} {1}_{p_2 \times p_2}, & \cdots &C_{2k} {1}_{p_2 \times p_k}\\
       \vdots & \vdots & \vdots &\vdots\\
    \end{pmatrix},
\end{align*}
\end{definition}
We are now ready to state the definition of the extended deterministic equivalent:
\begin{definition}[Deterministic equivalent]
\label{def:equivalent}
 Let $\mathbb{C}^+,\mathbb{C}^-$ denote the set of complex numbers with positive and negative imaginary parts respectively.
 Suppose that $z \in \mathbb{C}/\mathbb{R}^+$. Let $V^{\star} \in \mathbb{C}^{k\times k}, \nu^{\star} \in \mathbb{C}^k, b^\star \in \mathbb{C}^k$ be uniquely defined through the following conditions:
\begin{itemize}[leftmargin=*,noitemsep,wide=0pt]
    \item[(i)] $V^{\star},\nu^{\star},b^\star$ satisfy the following set of self-consistent equations:
    \begin{align*}
       &V^{\star}_{qq'}(z) = \Eb{\kappa}{\alpha\frac{c_1(\kappa, \zeta^{u}_{q})c_1(\kappa, \zeta^{u}_{q'})}{1+\chi\left(z;\kappa\right)}}\\
    &\nu^{\star}_{q}(z) =\Eb{\kappa}{\sum_{\ell \geq 2}\frac{\alpha c^2_{\ell}(\kappa, \zeta^{u}_{q})}{1+\chi(z;\kappa)}}\\
    & 
    b^\star_q(z) = \pi_{q} \beta \left(L_{q,q}(z) 
    + (\operatorname{diag}(\nu^{\star}(z)) -zI_k))_{q,q}\right)^{-1},
    \end{align*}
where $\kappa\sim\mathcal{N}(0,1)$, $(c_{\ell}(\kappa,\zeta))_{\ell >0}$ are defined in~\ref{def:main:shifted_hermite} and $(\chi(z;\kappa), L(z))$ read as follows:
    \begin{align*}
        &\beta \chi(z; \kappa) = \sum_{q,q' \in[k]} \psi_{qq'} c_1(\kappa,\zeta^{u}_{q})c_1(\kappa,\zeta^{u}_{q'}) + \sum_{q\in[k]}b^\star_{q}\sum_{\ell\ge 2}  c_{\ell}^2(\kappa,\zeta^{u}_{q}), \\
        &L(z) =\left(V^\star(z)^{-1} + \operatorname{diag}({b^\star(z)})\right)^{-1},
    \end{align*}
    where $ \psi(z) \in \mathbb{R}^{k \times k}$ is defined as:
    \begin{align}
        \psi(z)&=b^\star(z) - L(z) \odot (b^\star(z) (b^\star(z))^\top), \notag \\
    \end{align}
\item[(ii)] $V^{\star}, \nu^{\star}, b^\star$ are analytic mappings satisfying $V^{\star}_{i,j}: \mathbb{C}^+ \rightarrow \mathbb{C}^-$ for $i,j \in [k], \nu^{\star}_i : \mathbb{C}^+ \rightarrow \mathbb{C}^-$ for $i\in [k]$,  $b^\star:\mathbb{C}^+_i \rightarrow \mathbb{C}^+$ for $i\in [k]$. For $z \in \mathbb{C}^+$ with imaginary part $\zeta > 0$, $\abs{b^{\star}_q(z)} \leq \frac{\pi_q}{\beta\zeta} \ \forall q \in [p]$ 
\end{itemize}
We define the \emph{deterministic equivalent extended resolvent} ${\mathcal{G}}_e(z) \in\mathbb{R}^{(p+1)\times (p+1)}$ as:
\begin{align}
     \label{eq:def:detequiv}
        &{\mathcal{G}}_e(z) =\begin{bmatrix}
            A^\ast_{11}-zI_{k} & (A^\ast_{21})^\top \odot \theta^\top \\
          \theta \odot {A}^{\ast}_{21} & A^\ast_{22}+\alpha S_e^* \odot \theta \theta^\top
        \end{bmatrix}^{-1}, 
\end{align}

    % For $\lambda>0$, we define the \emph{deterministic equivalent extended resolvent} ${\mathcal{G}}_e(\zeta, \rho)\in\mathbb{R}^{(p+1)\times (p+1)}$ as:
    %  \begin{align}
    %  \label{eq:def:detequiv}
    %     {\mathcal{G}}_e(\zeta, \rho) =\left(\begin{bmatrix}
    %         A^\ast_{11} & (A^\ast_{21})^\top {E}^\top {B}\\
    %       {B}^\top {E}{A}^{\ast}_{21} & {B}^\top \left[\scriptstyle (V_e^{\star}(\zeta) \odot W W ^\top \!+ (D_e^*(\rho)+\lambda I_p))+\alpha {\Delta} \right] B
    %     \end{bmatrix} \!\!+\!\! \lambda I\right)^{-1}, 
    % \end{align}
% $V_e^{\star}(\zeta), D_e^*(\rho)$ denote extended versions of $V^{\star}, D^\star$ as per definition \ref{def:block_ext}
where $\theta = W w^\star \in \mathbb{R}^p$, and $S_e^\star \in \mathbb{R}^{p \times p}$, $A^\ast_{11} \in \mathbb{R}^{(k+1) \times (k+1)}, A^\ast_{21} \in \mathbb{R}^{p \times (k+1)}$ are defined as:
% \begin{bmatrix}
%         g(s)^{2} & g(s) c_0(\kappa,\zeta^{u})^\top\\
%         g(s)  c_0(\kappa,\zeta^{u}) & c_0(\kappa,\zeta^{u}) c_0(\kappa,\zeta^{u})^\top
%     \end{bmatrix}
\begin{align*}\label{eq:A11}
    &S^* = \Eb{\kappa}{(\kappa^2-1)\frac{c_1(\kappa, \zeta^u_i)c_1(\kappa, \zeta^u_j)}{1+\chi(z;\kappa)}}\\
    &A^\ast_{11} =  \, \Eb{\kappa}{\frac{\alpha}{1 + \chi(z;\kappa)}  
    \iota \,\, \iota^\top},\\
    &A^\ast_{21}[j,:] = \alpha \, \Eb{\kappa}{\frac{c_1(\kappa,u_j)}{1 + \chi(z;\kappa)}\kappa \iota^\top},  \forall j \in [p] \\
    &A^\ast_{22} = \left(\operatorname{diag}\left(\frac{\pi}{\beta b^\star}\right) -
    \operatorname{diag}(\frac{b^\star}{\pi\beta}) (\bV^{-1}_\star+\operatorname{diag}(\frac{b^\star}{\pi \beta}))^{-1}_e \odot \theta\theta^\top
    \operatorname{diag}(\frac{b^\star}{\pi\beta})
    \right)_e,
\end{align*}
where $\kappa \sim \mathcal{N}(0,1),\iota = (g(\kappa), c_0(\kappa, \zeta^u_1), \cdots c_0(\kappa, \zeta^u_k))^\top$ and the subscript $e$ in $S^\star_e$, $A^\star_{22}$ refers to the block matrix notation in Definition \ref{def:block_ext}.
\end{definition}

We are now in a position to state our main result, which states that ${\mathcal{G}}_e(z)$ approximates ${G}_e(z)$ for ``typical" linear functionals. %\looseness=-1
\begin{theorem}[Deterministic equivalent]
\label{thm:det_eq} Consider the extended resolvent ${G}_e(z)$ (\ref{def:resolvent}) associated to a batch of training data $(X,y)$ \ref{eq:def:data} after one gradient step. \change{Let $\mathcal{G}_e(z)$} denote the deterministic equivalent  (\ref{def:equivalent}). Then, under Assumptions \ref{ass:init}, \ref{ass:highd}, \ref{ass:activation} and \ref{ass:non-deg}, with neurons arranged as Equation \ref{eq:neuron}, for any $z \in \mathbb{C}/\mathbb{R}^+$ and sequence of deterministic matrices ${A} \in \mathbb{C}^{(p+1) \times (p+1)}$ with $\norm{A}_{\operatorname{tr}}=\tr(({A}{A}^*)^{1/2})$ uniformly bounded in $d$:
\begin{equation}
{\rm Tr}(A G_e(z))  \xrightarrow[d \rightarrow \infty]{a.s} {\rm Tr}(A \mathcal{G}_e(z)). 
\end{equation}
\end{theorem}
The class of linear functionals $A$ characterized above includes weighted traces as well as low-rank projections \cite{rubio2011spectral}. A direct consequence of Theorem \ref{thm:det_eq} is that it yields the Stieltjes transform of the bulk covariance:
\begin{corollary}[Stieltjes transform]
Let $\mu_d$ denote the empirical spectral measure of the bulk covariance
$\nicefrac{\tilde{\Phi}^\top\tilde{\Phi}}{p}$. 
Let $m_d(z)$ denote the Stieltjes transform $m_d(z) = \int \frac{1}{\lambda-z} d\mu_d(\lambda)$.
Let $b^\star(z)$ be as defined in~\ref{def:equivalent}, then:
\begin{equation}
    m_d(z)\xrightarrow[d \rightarrow \infty]{a.s} \beta \sum_{q=1}^k b^\star_q(z).
\end{equation}
\end{corollary}

Furthermore, as discussed previously, the deterministic equivalent ${\mathcal{G}}_e(z)$ contains all the necessary summary statistics for a full asymptotic characterization of the generalization error. This is the objective of the following theorem.

\begin{theorem}[Generalization Error] Under the proportional asymptotics \ref{ass:highd}, the generalization error \ref{eq:def:generror} is given by the following low-dimensional, deterministic formula:
\label{thm:gen_e}
\begin{equation}  
\lim\limits_{n,d,p\to\infty}\mathbb{E}[\change{\varepsilon_{\rm gen}}] =  \mathbb{E}_{\kappa}\left[\Lambda_{\kappa}(\{\tau_{0,q}, \tau_{1,q}\}_{q \in [k]}, \tau_2, \tau_3) \right]
\end{equation}
where $\{\tau_{0,q}, \tau_{1,q},, q \in [k] \}, \tau_2, \tau_3$ are  certain scalar deterministic functions of  $ V^\star(-\lambda), \nu^\star(-\lambda)$ reported in Appendix~\ref{sec:gen} along with the precise expression of the function $\Lambda_{\kappa}(\cdot)$.
\end{theorem}
\paragraph{Intepretation and proof Sketch --}
\label{sec:proof}

Unlike the high-dimensional interactions in the true feature covariance $\Phi^\top\Phi$, the interactions between neurons $i,j$ in $\mathcal{G}_e(z)$ depend only on the scalars $u_i, \theta$. $\mathcal{G}_e(z)$ therefore reflects the structure of a non-linear low-dimensional Kernel on $u_i, \theta$. Note that the dimensions of the order-parameters $V^{\star} \in \mathbb{C}^{k\times k}, \nu^{\star} \in \mathbb{C}^k, b^\star \in \mathbb{C}^k$ grows with the support size $k$. In the continuous support limit $k \rightarrow \infty$, we expect $V^{\star}$ and $\nu^{\star},b^\star$ to converge to certain limiting Kernels and functions respectively, satisfying \emph{functional} fixed point equations. The precise characterization of this regime constitutes an interesting avenue for future research.
% We expect such an analysis to be, however, challenging and leave it to future work.
 
As mentioned earlier, our proof proceeds through two stages of deterministic equivalent, successively eliminating the randomness over ${X}$ and ${W}^{0}$ respectively. A crucial aspect of our analysis is to decouple the randomness of ${X}, W^{0}$ along the spike $w^*$ and the orthogonal subspace --  this is due to the fact that the dependence of the features $\phi_\mu$ on $\kappa_\mu = x_\mu^\top w^\star$ is of a larger order than on the components of ${x}_\mu$ in the orthogonal space. 
Conditioned on $\kappa_\mu$, we show that the covariance of $\phi_\mu$ can be well-approximated through an equivalent linear model. However, due to the variability in $\kappa$, the description of the resolvent does not reduce to a standard random matrix theory ensemble. Lastly,  we obtain the generalization error through the introduction of certain perturbation terms into the deterministic equivalent, with the resolvent acting as a ``generating function" for additional relevant statistics. The detailed proofs are provided in the Appendices.

%% file: sections/discussion.tex
\section{Consequences of the Main Results} 
\label{sec:main:discussion}
\begin{figure}[t]
    \centering
    \includegraphics[width=0.48\textwidth]{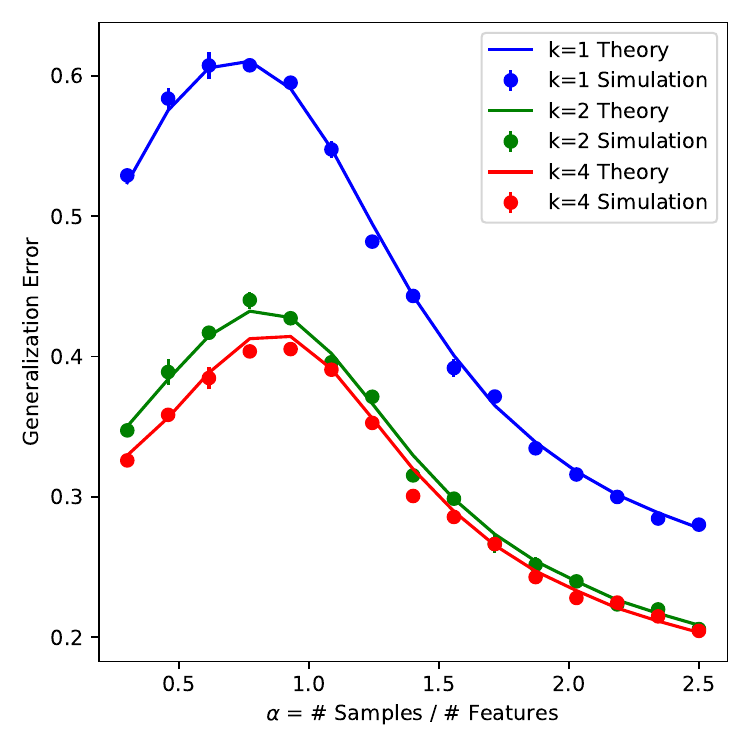}
    \caption{\textbf{Increase fitting accuracy through second layer variability:} Illustration of the benefits of larger support for the $2^{\rm nd}$ layer values $\sigma = \mathrm{ReLu}, \sigma_\star = \mathrm{tanh}$. Theoretical (continuous lines) and numerical (dots) predictions for the generalization error as a function of the number of samples per dimension $\alpha$ for different values of the second layer vocabulary size $k \in (1,2,4)$. The numerical simulations are averaged over 5 seeds and fixed hyper-parameters $\lambda = 0.01, \gamma = 0.5, \beta = 1.5, p = 2048$. Note the significant drop in the generalization error for $k>1$.
    \change{The choice of the probabilities ${\bf \pi} = \{\pi_q\}_{q \in [k]}$ and the vocabulary ${\bf \zeta} = \{\zeta_q\}_{q \in [k] }$ for the numerical illustration are: a) ${\bf k = 1}: {\bf \pi} = \{1\}, {\bf{\zeta}} = \{1\}$; b) ${\bf k = 2}: {\bf \pi} = \{0.9,0.1\}, {\bf{\zeta}} = \{1,-1\}$; c) ${\bf k = 4}: {\bf \pi} = \{0.7,0.1,0.1,0.1\}, {\bf{\zeta}} = \{1,-0.5,1.5-2\}$}
    }
    \label{fig:main:role_of_k}
\end{figure}

\paragraph{Tight characterization of the feature covariance spectrum --} 

The singular values of the features after one, but non-maximal gradient step -- $\eta \asymp p^{\zeta+\nicefrac{1}{2}}$ with $\zeta<\nicefrac{1}{2}$ -- have been characterized in \cite{moniri2023theory}, showing how a series of $\ell$ spikes appear after the step, in addition to the RF bulk corresponding to the features at initialization, as characterized in e.g. \citep{pennington2017nonlinear, benigni2021eigenvalue, benigni2022largest, fan2020spectra, louart2018random}. The number of spikes is given by the integer $\ell$ such that $\nicefrac{\ell-1}{2\ell}<\zeta <\nicefrac{\ell}{2\ell+2}$. Note that in the maximal step size limit $\zeta=\nicefrac{1}{2}$, the corresponding $\ell$ diverges, and it is largely unclear how the bulk and spikes recombine into the limiting features covariance spectrum.  

Our results further provide a tight asymptotic characterization of the bulk spectrum (represented in red in Fig.\,\ref{fig:spectrum} for $k=1, \alpha=0.8, \sigma=\mathrm{ReLu},g=\mathrm{sin},\Tilde{\eta}=3.3$). Note how this bulk is modified from the unspiked RF spectrum (the dashed blue line in Fig.\,\ref{fig:spectrum}), displaying in particular a wider support and longer tails in this particular instance. The spectrum also exhibits outlying eigenvalues of order $\Theta(d)$, arising from the means $\bar{\phi}_\mu$, not represented in Fig.\,\ref{fig:spectrum} for readability. This theoretical result ties in with numerous previous empirical observations \citep{martin2021implicit, martin2021predicting, wang2024spectral} that heavier tails can emerge after feature learning, a behaviour which further tends to correlate with better generalization abilities. Interestingly, this phenomenon persists even when the network is trained with \textit{multiple} large stochastic GD steps, or with adaptive optimizers such as Adam \citep{kingma2014adam}, as empirically observed by \cite{wang2024spectral}.

On a qualitative level, the departure of the bulk from its untrained shape can be intuitively seen as a result of the recombination between the untrained bulk and some of the spikes predicted by \cite{moniri2023theory}, as they proliferate when $\zeta\to \nicefrac{1}{2}$. We stress that in addition to these spikes, the feature covariance spectrum includes  $k$ additional spikes related to the non-zero mean property of the features $\Phi$. Such ``spurious spikes'' are not illustrated in Fig.~\ref{fig:spectrum} and are present due to the finite support assumption for the second layer (Assumption~\ref{ass:init}). It is an interesting avenue of future research to study the recombination of such spurious spikes with the bulk in the continuously supported second layer limit ($k \to \infty$). 
\paragraph{Precise characterization of the learned features --} While two-layer neural networks are known to be universal approximators \citep{cybenko1989approximation, hornik1989multilayer}, a precise characterization of the approximation space spanned by a given trained neural network feature maps remains to a large extent elusive. As we discuss in this paragraph, for features resulting from a maximal update on the first layer weights, the \textit{diversity} of the second-layer initialization $a^0$, namely the number of different values its components take, plays a crucial role in allowing for expressive feature maps. 
% \sout{Our result sharply characterizes the generalization performance of two-layer networks with a second layer finitely supported with values drawn from a finite vocabulary $V = \{\zeta_q\}_{q \in [k]}$, thereby extending the sharp characterization of \cite{cui2024asymptotics}, restricted to uniform second layer initializations, namely $k=1$.} 
% As can be read from the equivalent cGET features \eqref{eq: cGET_phieq} 
Indeed, there is a net increase of the expressivity of the neural network for larger vocabulary sizes $k$. More precisely, the network is able to express non-linear functions in $\kappa$ --the projection on the spike.
This is reflected in Definition \ref{def:equivalent} containing non-linear dependence on $\kappa$  along the functional basis $\{c_0(\cdot,\zeta_q),c_1(\cdot, \zeta_q)\}_{q\in[k]}$, with $c_1(\kappa ,\zeta_q)\equiv \kappa c_1(\kappa, \zeta_q)$. This functional basis is larger, and thus the neural network more expressive, for larger vocabulary sizes $k$, i.e. when the second layer is initialized with more variability.

To illustrate this point concretely,  consider for definiteness the case of an error function activation $\sigma(\cdot)=\mathrm{erf}(\cdot)$. Then,  $\{c_0(\cdot, \zeta_q)\}_{q\in[k]}=\{\mathrm{erf}(\nicefrac{\zeta_q}{\sqrt{3}}\cdot)\}_{q\in[k]}$. While a uniform second layer initialization $a^0\propto 1_p$ ($k=1$) only allows the network to express monotonic sigmoid functions, allowing the second initialization to take $k=2$ values already allows to express non-monotonic function with a derivative which can change sign twice. Pushing the second layer diversity to vocabularies of size $k\ge 3$ further enriches the pool of expressible functions with further non-monotonic functions. Depending on the functional form of the target activation $\sigma_\star$, the variability of the second layer can thus prove particularly instrumental in reaching a good approximation and learning. In Fig.~\ref{fig:main:role_of_k}, we illustrate the role of the vocabulary size $k$, for the simplest possible setting $(\text{single-index target}, \sigma = \mathrm{relu}, \sigma_\star = \mathrm{tanh})$, by plotting the test error as a function of the sample complexity $\alpha$ for varying vocabulary sizes $k\in\{1,2,4\}$. We observe a net decrease in the test error performance with increasing second layer variability $k$ at initialization. Furthermore, the theoretical results (plotted as continuous lines) are in accordance with numerical simulations (dots).

% \begin{figure}
%     \centering
%     \includegraphics[height=0.45\textwidth]{figures/[gen_error(k)]fig1a.pdf}
%     \includegraphics[height=0.45\textwidth]{figures/[toy_example]fig1b.pdf}

%     \caption{\textbf{Increase fitting accuracy through second layer variability:} The plots exemplify the benefits of having a larger support for the second layer values $(s=1, \sigma = \mathrm{ReLu}, \sigma_\star = \mathrm{tanh})$.  \textbf{Left:} Theoretical (continuous lines) and numerical (dots) predictions for the generalization error as a function of the number of samples per dimension $\alpha$ for different values of the second layer support size $k \in (1,2,4)$. The numerical simulations are averaged over 5 seeds and fixed hyperparameters $\lambda = 0.01, \gamma = 0.5, \beta = 1.5, p = 2048$. There is a net benefit in considering $k>1$ as the generalization error drops significantly. The plot shows also a plateau in the improvement for $k>2$. \textbf{Right:} Toy figure with $1-$dimensional input data providing a qualitative explanation for the improvement from $k =1$ to $ k=2$ in the fitting efficiency with a consequent plateau for larger values $k>2$.}
%     \label{fig:main:multi_index_complete}
% \end{figure}

%% file: sections/conclusion.tex
\section{Conclusions and Limitations}

We present a rigorous random matrix theory analysis of feature learning in a two-layer neural network, when the first layer is trained with a single, but aggressive, gradient step, in the limit where the number of samples $n$, the input dimension $d$, the hidden layer width $p$ and the learning rate $\eta$, jointly tend to infinity at proportional rates. We rigorously justify how the trained neural network can be approximated by a spiked random features model in the limit of large batch sizes. We derive a deterministic equivalent for the empirical covariance matrix of the resulting features. We further provide a tight asymptotic characterization of the test error, when the second layer is subsequently trained with ridge regression. Our results provide a rigorous proof to the heuristic work of \cite{cui2024asymptotics}, while extending it in multiple aspects. In particular, we allow for non-uniform initialization for the second-layer weights. We discuss how the second-layer variability enhances the expressivity of the trained network, and its ability to fit a single-index target. Among the limitations are the finitely supported second layer initialization, the use of Gaussian data, and asymptotic nature of the results. We note, however, that numerical experiments shows that predictions are accurate even at fairly moderate sizes. Secondly, a number of universality results shows that such ensembles extend over larger datasets \citep{dudeja2023universality, Gerace2024, NEURIPS2021_9704a4fc, pesce2023gaussian, wang2022universality}.

%% file: sections/acknowledgements.tex
\section*{Acknowledgements}
We would like to thank Denny Wu, Edgar Dobriban, Lenka Zdeborov\'a and Ludovic Stephan for stimulating discussions. We also thank the Institut d'\'Etudes Scientifiques de Carg\`ese for the hospitality during the ``Statistical physics \& machine learning back together again'' workshop, where this work started. BL acknowledges funding from the \textit{Choose France - CNRS AI Rising Talents} program. YD and HC acknowledge support from the Swiss National Science Foundation grant SNFS SMArtNet (grant number 212049). YD, LP, and FZ acknowledge support from the SNFS grant OperaGOST (grant number 200390). The work of YML was supported in part by the Harvard FAS Dean's Fund for Promising Scholarship and by a Harvard College Professorship. HC acknowledges support from the Center of Mathematical Sciences and Applications (CMSA) of Harvard University.

%% file: sections/appendix/proof.tex
\section{Structure of the Appendix}
The Appendix is organized as follows:
\begin{itemize}
    \item In section \ref{sec:prelim}, we list some notations, definitions and preliminary results utilized throughout the proof.
    \item In section \ref{app:spiked_approx}, we prove the  the isotropic-spike approximation of the gradient and show that in the limit $\alpha \rightarrow \infty$, it suffices to establish the deterministic equivalent and generalization error under the isotropic-spike approximation.
    \item In section \ref{sec:det_eq}, we prove the main Theorem \ref{thm:det_eq} characterizing the asymptotic deterministic equivalent of the sample covariance of the extended features.
    \item In section \ref{sec:gen}, we show how the generalization error can be expressed through certain functionals of the deterministic equivalent and finally obtain Theorem \ref{thm:det_eq}.
    \item Finally in Section \ref{app:extra}, we provide proofs of certain auxiliary results used in the analysis along with additional theoretical investigations
\end{itemize}

\section{Preliminaries}\label{sec:prelim}
\subsection{Stochastic Domination}

Throughout the analysis, we use the following notation for controlling high-probability bounds over stochastic error terms:
\begin{definition}\label{def:stoch-dom}[Stochastic dominance \citep{lu2022equivalence}]
    We say that a sequence of real or complex random variables $X_d$ is stochastically dominated by another sequence $Y_d$ if for all $\epsilon > 0$ and $k$, the following holds for large enough $d$:
    \begin{equation}
        \Pr[\abs{X}_d > d^{\epsilon}{\abs{Y}}_d]  \leq d^{-k}.
\end{equation}

We denote the above relation through the following notation:
\begin{equation}
    X = \mathcal{O}_{\prec}(Y).
\end{equation}

We further denote:
\begin{equation}
    \abs{X-Y}  = \mathcal{O}_{\prec}(Z),
\end{equation}
as:
\begin{equation}
    X-Y  = \mathcal{O}_{\prec}(Z).
\end{equation}

Similarly for vectors $X,Y \in \mathbb{R}^{k_d}$ for some sequence of dimensions $k_d$, we use the shorthand:
\begin{equation}
    X-Y  = \mathcal{O}_{\prec}(Z),
\end{equation}
to denote:
\begin{equation}
    \norm{X-Y}  = \mathcal{O}_{\prec}(Z),
\end{equation}

It is easy to check that stochastic dominance is closed under unions of polynomially many events in $d$. We will often exploit this while taking unions over $p=\mathcal{O}(d)$ neurons and $n=\mathcal{O}(d)$ samples. Furthermore, $\prec$ absorbs polylogarithmic factors i.e:
\begin{equation}
     X = \mathcal{O}_{\prec}(Y) \implies X = \mathcal{O}_{\prec}((\polylog d) Y)
\end{equation}

Furthermore, it subsumes exponential tail bounds of the form:
\begin{equation}
    \Pr[X_d > t Y_d]  \leq e^{-t^\alpha},
\end{equation}
for some $\alpha >0$, as well as polynomial tails of arbitrarily large degree:
\begin{equation}
     \Pr[X_d > t Y_d]  \leq \frac{C_k}{t^k},
\end{equation}
for some sequence of constants $C_k$ dependent on $k$.
\end{definition}

\begin{definition}[Hermite Expansion]
Let $f:\R \rightarrow \R$ be square-integrable function w.r.t the Gaussian measure. Then, $f$ admits a series expansion in the orthonormal basis of Hermite polynomials given by:
\begin{equation}
    f(x) = \sum_{i=0}^\infty a_i h_i(x),
\end{equation}
where the convergence holds in $L_2$ w.r.t the Gaussian measure.
\end{definition}
\begin{lemma}[Resolvent Identity]\label{lem:diff_inv}
Let, $A, B \in \R^{p \times p}$ be two invertible matrices, then: 
   \begin{equation}
       A^{-1}-B^{-1} = A^{-1}(B-A)B^{-1}
   \end{equation} 
\end{lemma}

\begin{lemma}\label{lem:op_norm}
Let $z \in \mathbb{C}/\R^+$ be arbitary and let $A \in \R^{N\times N}$ denote a p.s.d matrix. Let $\zeta = \max(\abs{\text{Im}(z)},-\text{Re}(z)$.
Then:
\begin{equation}
    \norm{(A-zI_N)} \leq \frac{1}{\zeta}
\end{equation}

\end{lemma}

\begin{lemma}[Burkholder's inequality (Lemma 2.12 in \cite{bai2008large})]\label{lem:burk}
    Let $X_i, i =1,\cdots n$ be a complex-valued Martingale difference sequence w.r.t a filtration $\mathcal{F}_i$, then for any $p \geq 1$, there exists a constant $K_p$ such that:
    \begin{equation}
        \Ea{\abs{\sum_{i=1}^n X_i}^p} \leq K_p (\Ea{\sum_{i=1}^n X_i^2})^{p/2}
    \end{equation}
    
\end{lemma}

\begin{definition}[Lipschitz concentration \citep{louart2018concentration}] 

A random variable $X \in \R^d$ is said to be $\alpha$-Lipschitz concentrated if for any $1$-Lipschitz function $f$:
\begin{equation}
    \P[f(X) \geq t] \leq \alpha(t)
\end{equation}
    
\end{definition}

\begin{definition}[Trace norm]
For any $A \in \mathbb{C}^{p \times p}$, the Trace norm is defined as:
\begin{equation}
   \norm{A}_{tr} =\Tr(\sqrt{A^\star A}),
\end{equation}
where $A^\star$ denotes Hermitian conjugate of $A$.
\end{definition}

\begin{lemma}[Trace-norm inequality] \label{lem:trace_in}
    For any $A, B \in \mathbb{C}^{p \times p}$:
    \begin{equation}
        \norm{AB}_{tr} \leq \norm{A}_{tr}\norm{B},
    \end{equation}
where $\norm{B}$ denotes the operator norm.
\end{lemma}

\begin{definition}[Schur complement]
\label{def:schur_complement}
    Let $p,q$ two non-negative integers such that $p+q >0$, consider the matrix $A \in \mathbb{R}^{(p+q) \times (p+q)}$:
    \[
A = \begin{pmatrix}
A_{11} & A_{12} \\
A_{21} & A_{22}
\end{pmatrix}
\]
If $A_{22}$ is invertible, the Schur complement of the block $A_{22}$ of the matrix $A$ is the matrix $A / A_{22} \in \mathbb{R}^{p \times p}$ defined as:
\begin{align}
    A / A_{22} = A_{11} - A_{12}A_{22}^{-1}A_{21}
\end{align}
Similarly,  the Schur complement of the block $A_{11}$ of the matrix $A$ is the matrix $A / A_{11} \in \mathbb{R}^{q \times q}$ defined as:
\begin{align}
    A / A_{11} = A_{22} - A_{21}A_{11}^{-1}A_{12}
\end{align}
\end{definition}

\begin{lemma}\label{lem:gauss-lipschitz}[Theorem 5.2.2 in \cite{vershynin2018high}]
    Let $z \sim \mathcal{N}(0,I_N)$ denote a standard Gaussian random vector in $R^N$ and let $f:\mathbb{R}^N \rightarrow \mathbb{R}$ be a Lipschitz-function with Lipschiz-constant bounded by $L$.
    Then:
    \begin{equation}
        \Pr[\abs{f(z)-\Ea{f(z)}} \geq t] \leq 2e^{-c t^2/L} 
    \end{equation}
\end{lemma}

\begin{lemma}\label{lem:sphere_lips}[Theorem 5.1.3 in \citep{vershynin2018high}]
    Let $z \sim  \mathcal{U}(\mathbb{S}^{N-1}(\sqrt{d}))$ denote a random vector uniformly sampled from the $N$ dimensional sphere of radius $\sqrt{N}$. Then, for any Lipschitz-function 
 $f:\mathbb{R}^N \rightarrow \mathbb{R}$ with Lipschiz-constant bounded by $L$:
    \begin{equation}
        \Pr[\abs{f(z)-\Ea{f(z)}} \geq t] \leq 2e^{-c t^2/L} 
    \end{equation}
\end{lemma}

\begin{lemma}\label{lem:operator-norms}  [
Theorem 5.3 in \cite{vershynin2010introduction} and  Theorem 4.3.5 in \citep{vershynin2018high}]
Under assumptions \ref{ass:init}, the operator norms $\norm{W_0}, \norm{X}$ satisfy, for some constants $C_1,C_2$
    \begin{align*}
 \norm{W_0} &\leq \frac{1}{\sqrt{d}}C_1(\sqrt{d}+\sqrt{p})+\mathcal{O}_\prec{(\frac{1}{\sqrt{d}})}\\
 \norm{X} &\leq C_2(\sqrt{n}+\sqrt{d})+\mathcal{O}_\prec(1)\\
    \end{align*}
\end{lemma}

\begin{lemma}\label{lemma:Hermite}[O'Donnell, Proposition 11.33, p. 338] The Hermite polynomials $(h_\alpha)_{\alpha \in \N}$ form a complete orthonormal basis. Further, for any $\rho \in [-1, 1]$ and two standard normal Gaussian random variables $z, z'$ that are $\rho$-correlated, we have
\begin{align}
    \Eb{z, z'}{h_\alpha(z)h_\beta(z')} = \begin{cases}
    \rho^{\alpha} & \text{if } \alpha = \beta,\\
    0 &\text{if } \alpha \neq \beta.
    \end{cases}
\end{align}
\end{lemma}
% \begin{definition}[High-probabilty events]
% Throughout the paper, we say that a sequence of events $\mathcal{E}_d$ occurs with high-probability if:
% \begin{equation}
%     \Pr[\mathcal{E}_d] \geq 1 - \operatorname{poly}(d)e^{-c \log(d)^2},
% \end{equation}
% for some polynomial $\operatorname{poly}(d)$ and constant $c >0$.
% Note that under the above definition, high-probality events are closed under  unions of polynomially-many events in $d$.
% \end{definition}

\begin{lemma}\label{lemma:weak_correlation} Let $f_1(z)$ and $f_2(z)$ be two twice-differentiable functions with the first and second derivatives bounded almost surely. Suppose that $f_1(z), f_2(z)$ admit the following Hermite expansions:
\begin{align}
    f_1(z) = \sum_{k\ge 0} c_k h_k(z) \qquad f_2(z) = \sum_{k \ge 0} \tilde c_k h_k(z),
\end{align}
where $\set{h_k(z)}_k$ denote the set of normalized Hermite polynomials. Given three unit norm vectors $w_1, w_2, w'$ such that:
\begin{equation}
    \abs{w_i^\top w'} =\mathcal{O}_\prec{(\frac{1}{\sqrt{d}})} 
\end{equation}
for $i = 1, 2$. Then for $g \sim \mathcal{N}(0, I)$, the following holds:
\begin{equation}\label{eq:f1f2_linear}
    \Ea{f_1(w_1^\top g) f_2(w_2^\top g)  g^\top w'} = c_0 \tilde c_1 (w_2^\top w') + \tilde c_0  c_1 (w_1^\top w') + \mathcal{O}_\prec(\frac{1}{d}),
\end{equation}
\begin{equation}\label{eq:f1f2_quadratic}
    \Ea{f_1(w_1^\top g) f_2(w_2^\top g) (g^\top w')^2} = c_0 \tilde c_0 + \mathcal{O}(d^{-1/2}),
\end{equation}
\end{lemma}

\section{Spiked isotropic approximation}\label{app:spiked_approx} 

In this section, we justify the assumption of an isotropic bulk in the regime $\alpha_0 \rightarrow \infty$, We first establish Lemma \ref{lemma:SRFM} and subsequently control the resulting approximation error for deterministic equivalence and generalization errors.

\subsection{Proof of Lemma \ref{lemma:SRFM}}

Let $\hat{y}_1, \cdots, \hat{y}_n$ denote the outputs of the network at initialization, i.e:
\begin{equation}
    \hat{y}_\mu = f( x_\mu; W^{(0)}, a) 
\end{equation}
We start by expressing the gradient as:
 \begin{equation}
     G = \frac{1}{n \sqrt p } \diag\{a_1, \ldots a_p\} \sigma'(W^0 X) \diag\{y_1-\hat{y}_1, \ldots, y_n-\hat{y}_n \} X \in \R^{p \times d},
 \end{equation}
where $X \in n\times d$ denotes the matrix with rows the data vectors in the first batch, $\set{y_i}_{i \in [n]}$ are the corresponding labels, and $\sigma'(\cdot)$ is the derivative of the student activation function. 

Next, we decompose $\sigma'$ as:
 \begin{equation}\label{eq:dsigma_decomp}
     \sigma'(x) = c_1 + \sigma'_{>1}(x),
 \end{equation}
 where
 \begin{equation}
     \sigma'_{>1}(x) \bydef \sum_{k \ge 2} \sqrt{k!}\,\mu_{k} \, h_{k-1}(x),
 \end{equation}
 and $c_{k}$ denote the Hermite coefficients of $\sigma$.
 Using the decomposition in \eqref{eq:dsigma_decomp}, we can rewrite the weight matrix $W^1$ as
 \begin{align*}
     W^1 &= W^0  + uv^\top + \underbrace{\frac{\eta}{n} \diag\{\sqrt p a_1, \ldots \sqrt p a_p\} \sigma'_{>1}(W^0 X) \diag\{y_1, \ldots, y_n\} X^T}_{\Delta_1}\\ &- \underbrace{\frac{\eta}{n} \diag\{\sqrt p a_1, \ldots \sqrt p a_p\} \sigma'_{>1}(W^0 Z) \diag\{\hat{y}_1, \ldots, \hat{y}_n\} Z^T}_{\Delta_2}
\end{align*}
 where $u = \mu_1 \eta \sqrt{p} a$ and $v = X^\top y / n$. 

Our goal is therefore to bound the contributions from $\Delta_1, \Delta_2$. We start by expressing $\Delta_1$ as a sum of rank one terms:
\begin{equation}
    \Delta_1 = \frac{1}{n} \sum_{\mu=1}^{n_0} b_\mu x_\mu^\top,
\end{equation}
where $b^\top_\mu = y_\mu\sigma'(W^0X)  [\sqrt p a_1, \ldots \sqrt p a_p]$.

Let $\{c^\star_k\}_{k \in \mathbb{N}}$ denote the Hermite coefficients of $g$. An application of Stein's Lemma and Hermite expansion yields (See Lemma 7 in \cite{damian2022neural} or Lemma 4 in \cite{dandi2023twolayer}) yields the following expansion for the $i_{th}$ row of $\Delta_1$:
\begin{equation}
    \Eb{x_\mu}{b^{i}_\mu x_\mu} =  \sum_{k=1}^\infty c_{k+1}\mu^\star_{k+1} \langle (w^\star)^{\otimes {k+1}}, w_i^{\otimes {k}} \rangle + \sum_{k=1}^\infty c_{k+2}c^\star_{k} \langle (w^\star)^{\otimes {k}}, w_i^{\otimes {k}}w_i \rangle
\end{equation}
By assumption $c_{2}=0$, therefore the term $c_{2}, c^\star_{2} w^\star \langle w^\star, w_i \rangle$ vanishes.
Since $\langle w^\star, w_i \rangle = \mathcal{O}_\prec(\frac{1
}{\sqrt{d}})$ for all $i \in [p]$, the remaining terms are bounded in term by $\mathcal{O}_\prec(\frac{1
}{d^{1.5}})$. We obtain:
\begin{equation}
    \norm{\Eb{x_\mu}{b^{i}_\mu x_\mu}}_F = \mathcal{O}_\prec(\frac{1
}{d}),
\end{equation}
which results in the following bound on the operator norm:
\begin{equation}
  \norm{\Eb{x}{\Delta_1}} =   \mathcal{O}_\prec(\frac{1
}{\sqrt{d}}).
\end{equation}

Next, note that the uniform boundedness of $\sigma', \sigma^*$ imply that  $\norm{b_\mu}^2 < Cp$ for some constant $C>0$.
Consider the symmetrized matrix:
\begin{equation}
    B_\mu = \begin{bmatrix}
        0 & b_\mu x_\mu^\top\\
        x_\mu b_\mu^\top  & 0
    \end{bmatrix} \in \mathbb{R}^{(p+d) \times (p+d)}
\end{equation}

Then, for all $k \geq 1$:
\begin{align}\label{eq:b_k}
     B^{2k+1}_\mu &= \begin{bmatrix}
        0 & \norm{x_\mu}^{2k}\norm{b_\mu}^{2k} b_\mu x_\mu^\top\\
    \norm{x_\mu}^{2k}\norm{b_\mu}^{2k} x_\mu b^\top_\mu & 0
    \end{bmatrix}\\
     B^{2k}_\mu &= \begin{bmatrix}
    \norm{x_\mu}^{2k}\norm{b_\mu}^{2k-2} b_\mu b_\mu^\top & 0\\
   0& \norm{x_\mu}^{2k-2}\norm{b_\mu}^{2k} x_\mu x^\top_\mu\\
    \end{bmatrix}
\end{align}

Now, let $z \in \mathbb{R}^{p+d}$ with
$\norm{z} = 1$ and let  $z_p = z[:p]$, $z_d = z[p:]$ denote the first $p$ and the remaining components of $z$ respectively.
From Equation \eqref{eq:b_k}, we have:
\begin{equation}
    \Ea{z^\top B^{2k+1}_\mu z} = \Ea{\norm{x_\mu}^{2k}\norm{b_\mu}^{2k} \langle b_\mu, z_p\rangle\langle x_\mu, z_d\rangle}+ \Ea{\norm{x_\mu}^{2k}\norm{b_\mu}^{2k}  \langle b_\mu, z_p\rangle\langle x_\mu, z_d\rangle}.
\end{equation}
Applying Cauchy–Schwarz to each term in the RHS yields:
\begin{equation}\label{eq:cauchysh}
     \Ea{z^\top B^{2k+1}_\mu z} \leq \Ea{\norm{x_\mu}^{4k}\norm{b_\mu}^{4k}}^{1/2} \Ea{(\langle b_\mu, z_p\rangle)^2(\langle x_\mu, z_d\rangle)^2}^{1/2}+ \Ea{\norm{x_\mu}^{4k}\norm{b_\mu}^{4k}}^{1/2} \Ea{(\langle b_\mu, z_p\rangle)^2(\langle x_\mu, z_d\rangle)^2}^{1/2}
\end{equation}

Now, the boundedness of $\sigma', \sigma^\star$ imply that $\norm{b_\mu}^{4k} \leq C_1^{4k} p^{2k}$ for some constant $C_1$. While $\norm{x_\mu}^2$ is a sub-exponential random variable with parameter $d$ and therefore \citep{vershynin2018high}:
\begin{equation}
    \Ea{\norm{x_\mu}^{4k}} \leq C_2^{2k} d^{2k} {2k}^{2k},
\end{equation}
for some constant $C_2>0$. 
Therefore:
\begin{equation}
\Ea{\norm{x_\mu}^{4k}\norm{b_\mu}^{4k}}  \leq C_2^{2k} d^{2k} {2k}^{2k}\times C_1^{4k} p^{2k}
\end{equation}

Finally, by assumption \ref{ass:activation}, $\sigma',g$ are uniformly-lipschitz. Furthermore,
with high probabilty over $W$, $\norm{W} \leq C_3$ for some constant $C_4$. We therefore obtain that $x \rightarrow  \langle b_\mu, z_p\rangle$ is uniformly lipschitz in $x$ with high probabilty over $W$.

Furthermore, applying Lemma \ref{lemma:Hermite}  to $\sigma',g$ and using $c_2=0$, yields:
\begin{equation}
    \Ea{\langle b_\mu, z_p\rangle} = \mathcal{O}(\frac{1}{\sqrt{d}}).
\end{equation}
Therefore, $\Ea{(\langle b_\mu, z_p\rangle)^2(\langle x_\mu, z_d\rangle)^2}^{1/2}$ is further bounded by some constant $C_4>0$.

Since $p/d=\beta$ is a constant, substituting in Equation \eqref{eq:cauchysh} we obtain:
\begin{equation}
    \Ea{z^\top B^{2k+1}_\mu z} \leq  (C_4d)^{2k-1} d.
\end{equation}
Similarly, 
\begin{equation}
    \Ea{z^\top B^{2k}_\mu z} \leq  (C_5d)^{2k-2} d,
\end{equation}
for some constant $C_5 >0$.

Therefore:
\begin{equation}
    \Ea{B^k_\mu} \prec (C_4d)^{k-2} d I_{p+d}, 
\end{equation}
for some constant $C_4$.

Subsequently, we apply the matrix-Bernstein inequality for self-adjoint matrices with subexponential tails (Theorem 6.2 in \cite{tropp2012user}) to obtain that:
\begin{equation}
   \Pr[\norm{\frac{1}{n} \sum_{\mu=1}^n B_\mu - \norm{\Ea{B_\mu}}} \geq t] \leq d e^{-t(\log d)^2/(c_1+c_2t)},
\end{equation}
for some constants $c_1,c_2 >0$. Borel-Cantelli Lemma then implies:
\begin{equation}
     \norm{\Delta_1} \xrightarrow[a.s]{d \rightarrow \infty} 0.
\end{equation}

Now, to bound $\Delta_2$, we use:
\begin{equation}
  \norm{\Delta_2} \leq  \frac{\eta}{n} \norm{\diag\{\sqrt p a_1, \ldots \sqrt p a_p\}} \norm{\sigma'_{>1}(W^0 Z)} \norm{\diag\{\hat{y}_1, \ldots, \hat{y}_n\}} \norm{X}.
\end{equation}

We show that:
\begin{equation}
    \norm{\diag\{\hat{y}_1, \ldots, \hat{y}_n\}} = \mathcal{O}_\prec(\frac{1}{\sqrt{d}})
\end{equation}

Recall that:
\begin{equation}
    \hat{y}_\mu = \frac{1}{\sqrt{p}} \sum_{j=1}^p a_j \sigma(w_j^{\top}x_\mu)
\end{equation}

By Theorem 3.1.1 in \cite{vershynin2018high},
 $\abs{\norm{x_\mu}-\sqrt{d}}$ for $\mu \in [n]$ are independent sub-Gaussian random variables. Therefore:
 \begin{equation}
     \sup_{\mu \in [n]} \norm{x_\mu} = \sqrt{d}+\mathcal{O}_\prec(\sqrt{\log n}).
 \end{equation}

 Therefore, we may condition on the high-probability event:
 \begin{equation}
    \mathcal{E}_n = \{\sup_{\mu \in [n]} \norm{x_\mu} \leq C\sqrt{d}\},
 \end{equation}
 for some $C>1$. Conditioned on the event $\mathcal{E}_n$, Lemma \ref{lem:sphere_lips} and assumption \ref{ass:activation} imply that for each $\mu \in [n]$, $\sqrt{p}a_j\{\sigma(w_j^{\top}x_\mu), j \in [p]\}$ are independent sub-Gaussian random variables with mean $0$. (Recall that by assumption \ref{ass:activation}, $\sigma$ is odd).
 Therefore $\forall \mu \in [n]$:
 \begin{equation}
     \hat{y}_\mu = \mathcal{O}_\prec(\frac{1}{\sqrt{d}}),
 \end{equation}
 where we absorbed polylogarithmic factors through the stochastic domination notation (Definition \ref{def:stoch-dom}).

Lastly, it remains to show that the spike $v$ converges to $w^\star$.
Recall that $v \coloneqq \frac{1}{c^\star_0}\frac{X^\top y}{n}$. By assumptions \ref{ass:activation}, $y_\mu x^{i}_\mu$ are independent sub-Gaussian random variables for $i \in [d]$. Therefore:
\begin{equation}
    \norm{v-w^\star} = \mathcal{O}_\prec(\sqrt{d/n_0}).
\end{equation}
Since, $n_0 = \mathcal{O}(d^{1+\epsilon})$, we obtain:
\begin{equation}
    v \xrightarrow[a.s]{} w^\star
\end{equation}

It remains to show that the equivalence in the sense of Lemma \ref{lemma:SRFM} extends to generalization error and deterministic equivalence:
\begin{proposition}
  Let $G^e,(G_\star)^e$ denote the extended resolvents with features $W^{(1)}$ and $\tilde{W}=W^0+\eta u (w^\star)^\top$ respectively. Analogously, let $\hat{a},\hat{a}_\star$ denote the ridge-regression predictors corresponding to
  $W^{(1)}$ and $\tilde{W}$ respectively. Then,for
  sequence of deterministic matrices ${A} \in \mathbb{C}^{(p+1) \times (p+1)}$ with $\norm{A}_{\operatorname{tr}}=\Tr(({A}{A}^*)^{1/2})$ uniformly bounded in $d$:
\begin{equation}\label{eq:tr_eq}
\abs{\Tr(A G_e(z))-\Tr(A G_\star)^e}  \xrightarrow[d \rightarrow \infty]{a.s} 0. 
\end{equation}
\begin{equation}
\abs{e_{\text{gen}}(\hat{a},W^{(1)})- e_{\text{gen}}(\hat{a}_\star,W^{(1)})}\xrightarrow[d \rightarrow \infty]{a.s} 0.
\end{equation}
\end{proposition}
\begin{proof}
    Consider the matrix:
    \begin{equation}
        \Xi =\sigma(XW^{(1 \top)})-\sigma(X\tilde{W}^\top) \in \mathbb{R}^{n \times p}.
    \end{equation}
A straightforward consequence of the proof of Lemma \ref{lemma:SRFM} and Assumption \ref{ass:activation}, Lemma \ref{lem:gauss-lipschitz} is that each row of $\Xi$ is a sub-Gaussian random vector with sub-Gaussian norms $\mathcal{O}(\frac{\polylog d}{d}^{\epsilon})$. Therefore, through tail bounds on the operator norms of matrices with sub-Gaussian rows \citep{vershynin2010introduction}, we obtain:
\begin{equation}\label{eq:xi}
    \frac{1}{\sqrt{p}}\norm{\Xi} \xrightarrow[a.s]{} 0
\end{equation}
The claim in Eq \eqref{eq:tr_eq} then follows since by Lemma \ref{lem:trace_in}:
\begin{align*}
    \abs{\Tr(A G_e(z))-\Tr(A G_\star)^e} &\leq \norm{A}_{\Tr}  \norm{G_e(z)-G_\star}\\
    &\leq \norm{A}_{\Tr}  \frac{C}{\zeta^2}\norm{\sigma(XW^{{(1)}^\top})-\sigma(X\tilde{W}^\top)},
\end{align*}
for some constant $C$. In the last line, we used Lemma \ref{lem:diff_inv} and $\zeta$ is defined as in Lemma \ref{lem:op_norm}. Similarly, Equation \eqref{eq:xi} implies the almost sure convergence for the generalization error.
\end{proof}
 
\section{Deterministic Equivalent}\label{sec:det_eq}

\label{app:proof:one}
The proof of Theorem \ref{thm:det_eq} proceeds in three parts:
\begin{itemize}
    \item Approximation of the covariance $\Ea{\phi \phi^\top}$ through the covariance $R^\star_\kappa$ of a conditional-Gaussian equivalent model at fixed values of $\bW$.  $R^\star_\kappa$ possesses a block-structure due to the finite-support assumption over $u_i$
    \item First-stage deterministic equivalent: here we average over the randomness in the inputs $\bX$ and construct a deterministic equivalent dependent on $\bW$ and expressed as a functional of  $R^\star_\kappa$.
    \item Second-stage deterministic equivalent: Here we further average over the randomness in $\bW$ to establish deterministic equivalents for $R^\star_\kappa$ and associated transforms.
\end{itemize}
We describe each of these stages below:
\subsection{Gaussian-equivalent covariance approximation}

We start by showing that conditioned on $\kappa$, the covariance of the extended features $\phi^e_\mu$ can be well-approximated through that of an equivalent linear model. This approximation is similar to the \emph{conditional Gaussian equaivalence} in \cite{dandi2023twolayer,cui2024asymptotics}. However, we emphasize that we do not directly utilize any universality result on generalization errors established in recent works \citep{lu2022equivalence,Montanari2022,dandi2023twolayer}. Instead, we will directly utilize the approximation of the
covariance in our analysis of the spectrum and generalization errors in the subsequent sections.

Recall the definition of the surrogate feature vectors in equation \eqref{eq:surr_feat}:

\begin{equation}
    \phi^e_\mu = \begin{pmatrix}
        y_\mu\\
        \bar{\phi}_\mu \\
        \tilde{\phi}_\mu
    \end{pmatrix}
\end{equation}

We observe that the sample covariance matrix of $z_\mu$ posseses the following block structure:
\begin{equation}
    (\Phi^e_\mu)^\top (\Phi^e_\mu) = \begin{bmatrix}
        \by\by^\top & \by\bar{\Phi}^\top & \by\tilde{\Phi}^\top\\
    \bar{\Phi}\by^\top & \bar{\Phi}\bar{\Phi}^\top &\tilde{\Phi}\bar{\Phi}^\top  \\
    \tilde{\Phi}\by^\top & \bar{\Phi}\tilde{\Phi}^\top &\tilde{\Phi}\tilde{\Phi}^\top \\
    \end{bmatrix}
\end{equation}

Define:
\begin{equation}\label{eq:def+dec}
  \Sigma_{11} \coloneqq \begin{bmatrix}
      \by\by^\top & \by\bar{\Phi}^\top\\
      \bar{\Phi}\by^\top & \bar{\Phi}\bar{\Phi}^\top
  \end{bmatrix} \quad
  \Sigma_{21}  \coloneqq \begin{bmatrix}
      \tilde{\Phi}y^\top \\
\tilde{\Phi}\bar{\Phi}^\top
  \end{bmatrix} \quad 
   \Sigma_{22}  \coloneqq \begin{bmatrix}
      \tilde{\Phi}\tilde{\Phi}^\top
  \end{bmatrix}
\end{equation}

\begin{proposition}\label{prop:phi_moments}
For a vector $u \in \mathbb{R}^m$ , let $c_j(u,\kappa) \in \mathbb{R}^m$ denote the vector with entries $c_j(u_i,\kappa)$.
    Let $\Eb{\kappa}{\cdot}$ denote the expectation w.r.t $x$ conditioned on the sigma-algebra generated by $\kappa \coloneqq x^\top w^*$. Define:
    \begin{equation}
    R^\star_{11}(\kappa) = \begin{bmatrix}
       g^2(\kappa) & g(\kappa) c_0(\kappa,u_\pi)^\top\\
        g(\kappa) c_0(\kappa,u_\pi) & c_0(\kappa,u_\pi) c_0(\kappa,u_\pi)^\top
    \end{bmatrix},
\end{equation}
where $u_\pi$ denotes the vector $(u_1, \cdots, u_k)$
\begin{equation}
    R^\star_{21}(\kappa) =  c_1(\kappa, u) \odot \btheta \kappa\begin{bmatrix}
    \sigma_\star(\kappa) & c_0(\kappa, u_1) \cdots & c_0(\kappa, u_k)
    \end{bmatrix},
\end{equation}
    \begin{align*}
        R^\star_{22}(\kappa) &= (c_1(\kappa, u)c_1(\kappa, u)^\top) \odot W W ^\top ) + \operatorname{diag}(\bigg(\sum_{k \ge 2} c^2_k(\kappa,u)\bigg) + (\kappa^2-1)c_1(\kappa, u)c_1(\kappa, u)^\top \odot \btheta\btheta^\top\\
        &+ \frac{1}{d}c_2(\kappa,u)c_2(\kappa,u)^\top
    \end{align*}

Then, there exists an $\ell \in \mathbb{N}$ such that, the following holds almost surely over $\kappa$:
\begin{equation}\label{eq:sig11}
    \sup_{\kappa \in \mathbb{R}/\mathcal{E}_0}\norm{\Eb{\kappa}{\Sigma_{11}} - R^\star_{11}(\kappa)} = \mathcal{O}_\prec(\frac{\kappa^\ell}{\sqrt{d}})
\end{equation}

\begin{equation}\label{eq:sig21}
      \sup_{\kappa \in \mathbb{R}/\mathcal{E}_0}\norm{\Eb{\kappa}{\Sigma_{21}}- \R^\star_{21}(\kappa)} = \mathcal{O}_\prec(\frac{\kappa^\ell}{\sqrt{d}})
\end{equation}
\begin{equation}\label{eq:sig22}
     \sup_{\kappa \in \mathbb{R}/\mathcal{E}_0}\norm{\Eb{\kappa}{\Sigma_{22}} -R^\star_{22}(\kappa)}  = \mathcal{O}_\prec(\frac{\kappa^\ell}{\sqrt{d}})
\end{equation}

\end{proposition}

\begin{proof}

Let $w_i$ denote the $i$th row of $W$ and $x \sim \mathcal{N}(0, I_d)$. By the rotational invariance of the Gaussian measure, we can express $x$ as:
\begin{equation}
    x = \kappa w^* + (I_d-w^* (w^*)^\top) \xi,
\end{equation}
where $\xi \sim \mathcal{N}(0, I_d)$ is independent of $x$

We obtain:
\begin{equation}
    x^\top w^1 = \kappa u_i + w_i^\top \xi + w_i^\top w^* (\kappa-\xi^\top w^*)
\end{equation}
Therefore,
\begin{align}
    \phi_i(x) = \sigma(\kappa u_i + w_i^\top \xi + w_i^\top \xi ( - w_i^\top g)  + w_i^\top \xi( u - \xi^\top g)).
\end{align}
Since
\begin{equation}
    \abs{w_i^\top w^* (\kappa-\xi^\top w^*)} = \mathcal{O}_\prec(\frac{\kappa}{\sqrt{d}}),
\end{equation}
by assumption \ref{ass:activation} an
application of the Taylor's theorem yields:
\begin{equation}\label{eq:phi_i_2nd}
    \begin{aligned}
    \phi_i &= \sigma(\kappa u_i + w_i^\top \xi) + \sigma'(\kappa u_i+ w_i^\top \xi) \big[w_i^\top w^* (\kappa-\xi^\top w^*)\big]\\
    &\qquad\qquad\quad + \frac 1 2 \sigma''(\kappa u_i +  w_i^\top \xi) \big[w_i^\top w^* (\kappa-\xi^\top w^*)\big]^2 + \mathcal{O}_\prec(\kappa^3 d^{-3/2}).
    \end{aligned}.
\end{equation}

To simplify the expectations of the second, third terms, we leverage the Hermite expansion and Lemmas \ref{lemma:Hermite}, \ref{lemma:weak_correlation}.

We begin by expanding $\sigma$ along the Hermite-basis for a fixed value of $\kappa u_i$:
\begin{equation}\label{eq:sigma_Hermite}
    \sigma(\kappa u_i + z) = c_0(\kappa u_i) + c_1(\kappa u_i) h_1(z) + c_2(\kappa u_i) h_2(z) + \ldots + 
\end{equation}

Using the identity $h_k'(z) = \sqrt{k} h_{k-1}(z)$, we also have
\begin{equation}\label{eq:dsigma_Hermite}
    \sigma'(u_i\kappa + z) = c_1(\kappa,u_i)  + \sqrt{2}c_2(\kappa,u_i) h_1(z) + \sqrt{3} c_3(\kappa,u_i) h_2(z) + \ldots
\end{equation}
and
\begin{equation}\label{eq:d2sigma_Hermite}
    \sigma''(u_i\kappa + z) = \sqrt{2}c_2(\kappa,u_i) + \sqrt{6}c_3(\kappa,u_i) h_1(z) + \ldots
\end{equation}

Consider the second term in Equation \eqref{eq:phi_i_2nd}:
\begin{equation}
    \Eb{\kappa}{\sigma'(\kappa u_i+ w_i^\top \xi) \big[w_i^\top w^* (\kappa-\xi^\top w^*)\big]}= \kappa(w_i^\top w^*)\Ea{\sigma'(\kappa u_i+ w_i^\top \xi)}- (w_i^\top w^*)\Ea{\sigma'(\kappa u_i+ w_i^\top \xi) \xi^\top w^*}
\end{equation}

From \eqref{eq:dsigma_Hermite}, the the first term in the RHS equals $\kappa(w_i^\top w^*)c_1(\kappa,u_i)$ while from Lemma~\ref{lemma:Hermite}, the second term equals
\begin{equation}
    \Eb{\kappa}{\sigma'(\kappa u_i+ w_i^\top \xi) \xi^\top w^*} = \sqrt{2}c_2(\kappa,u_i)w_i^\top w^*.
\end{equation}

Combining, we obtain:
\begin{equation}
    \Eb{\kappa}{\sigma'(\kappa u_i+ w_i^\top \xi) \big[w_i^\top w^* (\kappa-\xi^\top w^*)\big]} = \kappa(w_i^\top w^*)c_1(\kappa,u_i)- \sqrt{2}c_2(\kappa,u_i)(w_i^\top w^*)^2.
\end{equation}

Next, consider the third term in Equation \eqref{eq:phi_i_2nd}:
\begin{equation}
    \Eb{\kappa}{\frac 1 2 \sigma''(\kappa u_i +  w_i^\top \xi) \big[w_i^\top w^* (\kappa-\xi^\top w^*)\big]^2} =  (w_i^\top w^*)^2\frac 1 2 \Eb{\kappa}{ \sigma''(\kappa u_i +  w_i^\top \xi) \big[\kappa^2-2\xi^\top w^* + (\xi^\top w^*)^2\big]}
\end{equation}

Again applying Equation \eqref{eq:d2sigma_Hermite} and Lemma \ref{lemma:Hermite} to each term in the RHS yields, simplifies as follows:
\begin{align*}
    \Eb{\kappa}{ \frac{1}{2} \kappa^2(w_i^\top w^*)^2 \sigma''(\kappa u_i +  w_i^\top \xi)} = \frac{1}{\sqrt{2}}\kappa^2c_2(\kappa,u_i)\\
    \Eb{\kappa}{(w_i^\top w^*)^2 \sigma''(\kappa u_i +  w_i^\top \xi)\xi^\top w^*}= \sqrt{6}c_3(\kappa,u_i)(w_i^\top w^*)^3 = \mathcal{O}_\prec(\frac{1}{d^{3/2}})\\
     \Eb{\kappa}{\frac{1}{2} (w_i^\top w^*)^2 \sigma''(\kappa u_i +  w_i^\top \xi)(\xi^\top w^*)^2}= \frac{1}{\sqrt{2}}c_2(\kappa,u_i)(w_i^\top w^*)^2 +\mathcal{O}_\prec(\frac{1}{d^{2}}),
\end{align*}
where in the second equation, we used that $c_3(\kappa,u_i)$ is uniformly bounded in $\kappa$ by Assumption \ref{ass:activation} and in the last equation, we used $(\xi^\top w^*)^2 = \sqrt{2}h_2(\xi^\top w^*) + 1$.

Combining, we obtain: 

\begin{equation}\label{eq:E_phi_0}
    \Eb{\kappa}{\phi_i} = c_0(\kappa, \vec{u}) + c_1(\kappa)k w_i^\top w^\star  +\frac{c_2(\kappa,\vec{u})}{\sqrt{2}}\left[(w_i^\top w^\star)^2(\kappa^2-1)\right] + \mathcal{O}_\prec(d^{-3/2}).
\end{equation}
 gives us:

 \begin{equation}
     \Eb{\kappa}{\phi} = c_0(\kappa, u) 1_p + c_1(\kappa, u) \odot W\big(\kappa w^\star)
         + \mathcal{O}_\prec(d^{-1})
 \end{equation}
Now, since $\langle w_1, w^\star \rangle, \cdots, \langle w_p, w^\star \rangle \in \mathbb{R}^d$ are i.i.d. sub-Gaussian random variables, we obtain that:
\begin{equation}
    \Eb{\kappa}{\bar{\phi}}_q = c_0(\kappa, \zeta^u_q).
\end{equation}

Furthermore, by Assumption \ref{ass:activation} and since $\norm{W}_2 < C$ a.s for large enough $C$, we obtain that $\bar{\phi}_q$ is an $\mathcal{O}(\frac{1}{\sqrt{d}})$-Lipschitz function of $x$, we obtain:
\begin{equation}
    \bar{\phi} -  \Eb{\kappa}{\bar{\phi}} = \mathcal{O}_\prec(\frac{1}{\sqrt{d}})
\end{equation}
Therefore, in what follows, we will utilize:
\begin{equation}\label{eq:app}
    \tilde{\phi_i} = \phi_i-c_0(\kappa,u_i) + \mathcal{O}_\prec(\frac{1}{\sqrt{d}}).
\end{equation}

Combining the above approximation with Equation \eqref{eq:E_phi_0}  
and using the uniform boundedness of $g$
directly yields the estimates for $\Sigma_{11},\Sigma_{21}$ in Equations \eqref{eq:sig11}  and \eqref{eq:sig21}.

To study the block $\Sigma_{22}$, we separate the case of diagonal and off-diagonal entries. For the former, we  apply Lemma \ref{lemma:Hermite} to Equation \eqref{eq:phi_i_2nd} to obtain:
    \begin{align}\label{eq:diag_terms}
        \Eb{\kappa}{\tilde{\phi}_i^2} = \sum_{k\ge 1}c_k^2 (\kappa,u_i) + \mathcal{O}_\prec(\kappa/\sqrt{d}).
    \end{align}
    For the off-diagonal terms with $i \neq j$, we start by noting that
    \begin{equation}
        \Eb{\kappa}{\tilde{\phi}\tilde{\phi}^\top} = (I_p-P) \Eb{\kappa}{\phi\phi^\top}(I_p-P)^\top,
    \end{equation}
where $P$ denotes the projection on the directions $e^1,\cdots e^k$ defined in Equation \eqref{eq:def:eq}.
    The Taylor expansion in \eqref{eq:phi_i_2nd} and the approximation \eqref{eq:app} then gives us:
\begin{align}
    &\Eb{\kappa}{\tilde{\phi}_i \tilde{\phi}_j} = \underbrace{\Eb{\kappa}{(\sigma(\kappa u_i + w_i^\top \xi)-\bar{\phi}_i)(\sigma(\kappa  u_j + w_j^\top \xi)-\bar{\phi}_j))}}_{(A)}\\
    &+ \underbrace{\Eb{\kappa}{(\sigma(\kappa u_i + w_i^\top \xi)-c_0(\kappa,u_i))\sigma'( \kappa u_j + w_j^\top \xi) \big[w_j^\top w^\star (\kappa - \xi^\top w^\star ) \big]}_{(B)}} + (B)_{i \leftrightarrow j}\\
    &+ \underbrace{\frac 1 2 \Eb{\kappa}{(\sigma(\kappa u_i + w_i^\top \xi)-c_0(\kappa,u_i))\sigma''( \kappa u_j + w_j^\top \xi) \big[w_j^\top w^\star (\kappa - \xi^\top w^\star ) \big]^2}}_{(C)} + (C)_{i \leftrightarrow j}\\
    &+\underbrace{\Eb{\kappa}{\sigma'(\kappa u_i + w_i^\top \xi) \big[w_i^\top w^\star (\kappa - \xi^\top w^\star)\big]\sigma'(\kappa u_j + w_j^\top \xi) \big[w_j^\top w^\star (\kappa - \xi^\top w^\star)\big]}}_{(D)}\\
    &+ \mathcal{O}(\kappa^3 d^{-3/2}),
\end{align}
where $(B)_{i \leftrightarrow j}, (C)_{i \leftrightarrow j}$ denote the corresponding terms with the roles of $i,j$ interchanged.
Next, we consider each term on the right-hand side. Applying Lemma~\ref{lemma:Hermite}, we obtain that the term $A$ results in:
\begin{equation}
    (A) = (I_p-P)\left( c_1(\kappa, u) c_1(\kappa, u)^\top \odot W^\top W  + c_2(\kappa, u) c_2(\kappa, u)^\top \odot (W^\top W)^2  \right ) (I_p-P)^\top + \mathcal{O}_\prec(d^{-3/2}).
\end{equation}
For terms in $(B)$, we apply Lemma~\ref{lemma:weak_correlation} with $w_1,w_2=w_i,w_j$ and $w'=w^\star$ to obtain:
\begin{equation}
\begin{aligned}
    (B) + (B)_{i \leftrightarrow j} &= \left( \sqrt{2}c_1(\kappa,u_i)c_2(\kappa,u_j) w_i^\top w_j\right)(w_i + w_j)^\top (\kappa w^\star)\\
    &\qquad\qquad-2 c_1(\kappa,u_i)c_1(\kappa,u_j)\big[(w_i^\top w^\star)(w_j^\top w^\star)] + \mathcal{O}_\prec(d^{-3/2}).
    \end{aligned}
\end{equation}
Similarly, for terms in $C$, Lemma~\ref{lemma:weak_correlation} directly implies that:
    \begin{equation}
    \begin{aligned}
        &(C) + (C)_{i \leftrightarrow j}= \mathcal{O}_\prec(d^{-3/2}).
        \end{aligned}
    \end{equation}
Finally,
\begin{equation}
\begin{aligned}
    (D) = c_1(\kappa,u_i)c_1(\kappa,u_j) (\kappa^2+1) (w_i^\top w^\star) (w_j^\top w^\star)+\mathcal{O}_\prec(d^{-3/2})
\end{aligned}
\end{equation}

We are now ready to establish an approximation of the correlation matrix $\Ea{\tilde{\phi} \tilde{\phi}^\top}$, up to an error term whose operator norm is of size $\mathcal{O}_\prec(d^{-1/2})$.  Starting with $(A)$, since $w_i^\top e^q = \mathcal{O}_\prec(\frac{1}{\sqrt{d}})$ for $i \in [p], q \in [k]$, we have:
\begin{equation}
    \norm{(I_p-P)\left( c_1(\kappa, u) c_1(\kappa, u)^\top \odot W^\top W \right)(I_p-P)^\top- c_1(\kappa, u) c_1(\kappa, u)^\top \odot W^\top W} = \mathcal{O}_\prec(\frac{1}{\sqrt{d}}).
\end{equation}
Therefore, the extra term compared to $R^\star_{22}(\kappa)$ is the one containing
$(w_i^\top w_j)^2$. Write
\begin{equation}\label{eq:wiwj_2}
    (w_i^\top w_j)^2 = \frac 1 d + \sqrt{\frac{2p}{d^2}}(H_2)_{ij},
\end{equation}
$H_2 \in \R^{p \times p}$ is a symmetric matrix such that
\begin{equation}
    (H_2)_{ij} = \frac{(\sqrt{d}w_i^\top w_j)^2 - 1}{ \sqrt{2p}} \, 1_{i \neq j}.
\end{equation}
This is in fact a Kernel gram matrix where the ``kernel'' function is the second Hermit polynomial $h_2(z)$. As studied in e.g. [Zhou \& Montanari, 2013], [Lu \& Yau, 2023], the spectrum of this kernel matrix is asymptotically equal to the semicircle law. Moreover, the operator norm of this matrix is bounded by $\mathcal{O}_\prec(1)$. It follows that the term $\sqrt{\frac{2p}{d^2}}(H_2)_{ij}$ in \eqref{eq:wiwj_2} can be ignored due to its vanishing operator norm. Thus, what is left is the contribution from the term $\frac{1}d$, which vanishes since $ (I_p-P)c_2(\kappa,u)c_2(\kappa,u)^\top (I_p-P)^\top = 0$. 

For each $k$, let $h_{k,w^\star}$  be the vector in $\R^p$, defined as
\begin{equation}\label{eq:hk_theta_xi}
    (h_{k,w^\star})_i \bydef h_k(\sqrt{d} w_i^\top w^\star)
\end{equation}
with $h_k$ being the $k$th Hermite polynomial. By substituting $(w_i^\top w^\star)^2 = \frac{1}{d}(h_2(\sqrt{d}d w_i^\top w^\star)-1)$

Next, we examine $(B) + (B)_{i \leftrightarrow j}$ and capture its main terms. First, note that $(w_i^\top w_j) w_i^\top (\kappa w^\star)$ can be ignored, since it corresponds to a matrix
\begin{equation}
    W\diag\set{\kappa w_i^\top w^\star } W^\top,
\end{equation}
whose operator norm is bounded by $\mathcal{O}_\prec(d^{-1/2})$ by Lemma \ref{lem:operator-norms}. 

Let $\theta_i = W^\star w_i$ for $i \in [p]$
The remaining components in $(B) + (B)_{i \leftrightarrow j}$ can be expressed as:
\begin{align*}
    -2c_1(\kappa,u)c_1(\kappa,u)  \theta_i \theta_j 
\end{align*}

It is also easy to verify that the last component $(D)$ can be expressed as follows:
\begin{equation}
    (D) = c_1(\kappa,u)c_1(\kappa,u) (\kappa^2+1) \theta_i \theta_j ,
\end{equation}

The terms $B$ and $D$ combine to give the component:
\begin{equation}
    (\kappa^2-1)c_1(\kappa, u)c_1(\kappa, u)^\top \odot \btheta\btheta^\top
\end{equation}
Combining with the contribution from $A$ and the diagonal terms in Equation \eqref{eq:diag_terms}, we obtain Equation \eqref{eq:sig22}.
\end{proof}

Before moving further, we note down a bound on $R^\star_\kappa$ that will prove useful late:
\begin{proposition}
    Let $R^\star_\kappa$ be as defined in Proposition \ref{prop:phi_moments}. Then, $\exists \ell \in \mathbb{N}$ such that:
    \begin{equation}
        \norm{R^\star_\kappa} = \mathcal{O}_\prec(\kappa^\ell)
    \end{equation}.
\end{proposition}
\begin{proof}
    The above is a direct consequence of Proposition \ref{prop:phi_moments} and Assumption \ref{ass:activation}, with the later implying that $\abs{c_o(\kappa, u)} \leq K\abs{u\kappa}$ for some constant $K$ while $\abs{c_j(\kappa, u)}$ are uniformly bounded for $j=1,2,3$.
    \end{proof}

\subsection{First-Stage Equivalent}

We now proceed with establishing the first stage of deterministic equivalence, aiming to remove the randomness w.r.t $X$. Concretely, our goal in this section is to construct a sequence of matrices $\mathcal{G}^e_W$ dependent on $W$ but not $X$ such that $\mathcal{G}^e_W$ is  deterministically equivalent to $\mathcal{G}^e$ for all suitably-bounded linear functions independent of $X$ (but possibly depending on $W$). Ideally $\mathcal{G}^e_W$ will possess a simpler structure than $\mathcal{G}^e$, allowing us to further simplify it in the next stage.

For subsequent usage, we shall establish a more quantitative version of deterministic equivalent than Theorem \ref{thm:det_eq}, with an explicit error bound of $\mathcal{O}_\prec(\frac{1}{\sqrt{d}})$.

\begin{proposition}\label{prop:first_stag}
Let $X_d, W_d$ be a sequence of random matrices generated as in Assumption \ref{ass:init}. Let $z \in \mathbb{C}/\mathbb{R}^+$ be arbitrary.
Let $\mathcal{G}_W(z)$ denote the the following sequence of random matrices dependent only on $W$:
    \begin{equation}
       \mathcal{G}_W(z) \coloneqq  \left(\frac{\alpha}{p} 
        \Sigma^\star+ \zeta L  -z I\right)^{-1},
    \end{equation}
where:
\begin{equation}\label{eq:chi_mu}
   \chi_d(\kappa)\bydef \frac{\Eb{W,X}{\Tr(G^\mu_e R^\star_\kappa)}}{p},
\end{equation}
and the expectation is w.r.t both $W, X$ conditioned on the projection $\kappa_\mu$ along the spike $\vec{v}$ as fixed. Since $G^\mu_e$ is independent of $\kappa$, $\chi_d(\kappa)$ is a fixed function of $\kappa$ alone. The subscript $d$ is to remind us that $\chi_d(\kappa)$ defined as in Equation \eqref{eq:chi_mu} is an unknown dimension-dependent function.

We further define:
\begin{equation}
   \Sigma^\star= \Eb{\kappa \sim \mathcal{N}(0,1)}{\frac{1}{1+\chi(\kappa)} R^\star(\kappa_\mu)},
\end{equation}
with $R^\star(\kappa)$ denoting the matrix:
\begin{equation}
    R^\star(\kappa) = \begin{pmatrix}
       R^\star(\kappa)_{11} &  R^\star(\kappa)_{12}\\
       R^\star(\kappa)_{21} & R^\star(\kappa)_{22} 
    \end{pmatrix}
\end{equation}
and $R^\star(\kappa)_{11},R^\star(\kappa)_{12},R^\star(\kappa)_{21},R^\star(\kappa)_{22}$ are as defined in Proposition \ref{prop:phi_moments}.

Then for any sequence of matrices $A$, possibly dependent on $W$, but not $X$:
\begin{equation}\label{eq:det_eq}
\abs{\Tr{A\mathcal{G}_W(z)}-\Tr{AG_e(z)}} = \mathcal{O}_\prec(\frac{\norm{A}_{\Tr}}{\sqrt{d}})
\end{equation}
\end{proposition}

\begin{proof}

We start by defining:
\begin{equation}
    \chi_{X,W}(\mu) \coloneqq \frac{1}{p}(\phi^e_\mu)^\top G_e^\mu \phi^e_\mu,
\end{equation}
Note that unlike $\chi(\kappa)$ which is solely a function of $\kappa$, $\chi_{X,W}(\mu)$ is a random variable depending on $X,W, z_\mu$. 
We further define:
\begin{equation}
    R_\mu = \Ea{\phi^e_\mu (\phi^e_\mu)^\top|\kappa=\kappa_\mu}
\end{equation}

We proceed through a leave-one-out argument, analogous to the one for Wishart-type matrices \citep{bai2008large}. By definition, $G_e$ satisfies:
\begin{equation}\label{eq:g_eq}
     G_e (\frac{1}{p}\sum_{\mu=1}^n \phi^e_\mu (\phi^e_\mu)^\top -z I) = I.
\end{equation}

Let $G_e^\mu \coloneqq (\sum_{\nu \neq \mu} \phi^e_\mu (\phi^e_\mu)^\top + \lambda I)^{-1}$ for $\mu \in [n]$ denote the resolvent with the sample $\mu$ removed.

Using the Sherman-Morrison formula, we can express $G_e$ for any $\mu \in [n]$ as:
\begin{equation}
    G_e = G_e^\mu - \frac{1}{1+ \frac{1}{p}(\phi^e_\mu)^\top G_e^\mu \phi^e_\mu} \phi^e_\mu (\phi^e_\mu)^\top.
\end{equation}
Next, for each $\mu \in [n]$, we replace $G_e$ in the term $G_e (\phi^e_\mu (\phi^e_\mu)^\top )$ by the above decomposition to obtain:

\begin{equation}\label{eq:g_eq_2}
     \sum_{\mu=1}^n \frac{1}{p}(G_e^\mu - \frac{1}{1+\frac{1}{p}(\phi^e_\mu)^\top G_e^\mu \phi^e_\mu)}\phi^e_\mu (\phi^e_\mu)^\top ) \phi^e_\mu (\phi^e_\mu)^\top -z I G_e  = I.
\end{equation}
The first set of terms simplify as:
\begin{align*}
    G_e (\phi^e_\mu (\phi^e_\mu)^\top) &= (G_e^\mu - \frac{1}{1+\frac{1}{p}(\phi^e_\mu)^\top G_e^\mu \phi^e_\mu} G_e^\mu \phi^e_\mu (\phi^e_\mu)^\top  G_e^\mu)\phi^e_\mu (\phi^e_\mu)^\top  \\
    &=  \frac{1}{1+\chi_{X,W}(\mu)}G_e^\mu \phi^e_\mu (\phi^e_\mu)^\top.
\end{align*}
Substituting in \eqref{eq:g_eq_2} yields:

\begin{equation}
     \frac{1}{p} \sum_{\mu=1}^n \frac{1}{1+\chi_{X,W}(\mu)}G_e^\mu \phi^e_\mu (\phi^e_\mu)^\top -z I G_e  = I,
\end{equation}

By concentration of $\chi_{X,W}(\mu)$ and averaging over $\phi^e_\mu (\phi^e_\mu)^\top$, we expect 
the first term to be well-approximated by $\sum_{\mu \in [n]} (\frac{1}{1 + \chi_d(\kappa)} G_e^\mu  \phi^e_\mu (\phi^e_\mu)^\top )$. We therefore, isolate the errors in $\chi_{X,W}(\mu)$ to obtain:

\begin{equation}\label{eq:Ge_id0}
     \frac{1}{p}\sum_{\mu \in [n]} (\frac{1}{1 + \chi_d(\mu)} G_e^\mu \phi^e_\mu (\phi^e_\mu)^\top  + \mathcal{E}_{1,\mu} G_e^\mu  \phi^e_\mu (\phi^e_\mu)^\top  + \mathcal{E}_{2,\mu} G_e^\mu  \phi^e_\mu (\phi^e_\mu)^\top + \mathcal{E}_{3,\mu} G_e^\mu R^\star(\kappa_\mu)) -z G_e = I.
\end{equation}
where $\mathcal{E}_{1,\mu}, \mathcal{E}_{2,\mu}, \mathcal{E}_{3,\mu}$ account for the error in $\chi_{X,W}(\mu)$ due to randomness over $\phi_\mu$, approximation of  covariance $R_\mu$ by $R^\star_\mu$ and the concentration of $\Tr(G^\mu_e R^\star_\mu)$ to $\chi(\kappa)$ respectively.
\begin{align*}
    \mathcal{E}_{1,\mu} &= \frac{(\phi^e_\mu)^\top  G_e^\mu \phi^e_\mu) - \frac{\Tr(G^\mu_e R_\mu)}{p}}{(1+\frac{1}{p}(\phi^e_\mu)^\top  G_e^\mu \phi^e_\mu)(1+\frac{\Tr(G^\mu_e R_\mu)}{p})}\\
     \mathcal{E}_{2,\mu} &= \frac{ \frac{\Tr(G^\mu_e R_\mu)}{p}-\frac{\Tr(G^\mu_e R^\star_\mu)}{p}}{(1+\frac{\Tr(G^\mu_e R^\star_\mu)}{p})(1+\frac{\Tr(G^\mu_e R^\star_\mu)}{p})}\\
      \mathcal{E}_{3,\mu} &= \frac{ \frac{\Tr(G^\mu_e R^\star_\mu)}{p}-\chi(\kappa)}{(1+\chi(\kappa))(1+\frac{\Tr(G^\mu_e R^\star_\mu)}{p})}
\end{align*}
Next, due to the average over $n$ samples and the small error in replacing $G^\mu_e$ by $G_e$, we further expect $\sum_{\mu \in [n]} (\frac{1}{1 + \chi(\kappa)} G_e^\mu \phi^e_\mu (\phi^e_\mu)^\top)$  to be well-approximated (in the sense of deterministic equivalence) by 
$ G_e(\sum_{\mu \in [n]} \frac{1}{1 + \chi(\kappa)}R^\star_\mu)$.

Therefore, we explicitly introduce the above term into Equation \eqref{eq:Ge_id0} at the cost of additional errors:
\begin{equation}
    \frac{1}{p}G_e(\sum_{\mu \in [n]} \frac{1}{1 + \chi(\kappa)}R^\star_\mu) -z G_e = I+\Delta,
\end{equation}
where:
\begin{equation}
   \Delta =   \sum_{\mu \in [n]} \Delta_\mu,
\end{equation}
where:
\begin{equation}
  \Delta_\mu = \frac{1}{p}\left (\frac{1}{1 + \chi_W(\mu)} G_e R^\star_\mu - G^\mu_e\frac{1}{1 + \chi_W(\mu)} \phi^e_\mu (\phi^e_\mu)^\top -G^\mu_e\mathcal{E}_{1,\mu} \phi^e_\mu (\phi^e_\mu)^\top - G^\mu_e\mathcal{E}_{2,\mu}  \phi^e_\mu (\phi^e_\mu)^\top - G^\mu_e\mathcal{E}_{3,\mu}  \phi^e_\mu (\phi^e_\mu)^\top\right)
\end{equation}

With the above error terms defined, we're now ready to analyze the linear functional $\Tr(AG_e(z))$ in Equation \eqref{eq:det_eq}. We perform the analysis in two-stages:
\begin{itemize}
    \item Show that $\Tr{AG_e(z)} \rightarrow \Tr{A\mathcal{G}_w(z)}$.
    \item Show that $\Tr{AG_e(z)}$ concentrates around its expectation.
\end{itemize}

We start with the first stage. We have:
\begin{equation}
    \Ea{\Tr(AG_e)} = \Ea{\frac{1}{p}\Tr{(A(I+\Delta)(\alpha R^\star -z I)^{-1})}},
\end{equation} 
where $R^\star \coloneqq \sum_{\mu \in [n]} \frac{1}{1 + \chi_W(\mu)}R_\mu$.

Our next goal is therefore to bound the contribution from $\Delta$.

%By the trace-norm inequality (Lemma \ref{lem:trace_in}), it suffices to show that $\Ea{\Tr{(A\Delta)}} \rightarrow 0$.

The independence of $G^\mu_e$ and $\phi^e_\mu $, implies that:
\begin{equation}
 \Ea{G^\mu_e\frac{1}{1 + \chi_W(\mu)}  \phi^e_\mu (\phi^e_\mu)^\top} = \Ea{G^\mu_e\frac{1}{1 + \chi_W(\mu)} R_\mu}
\end{equation}
Thus,  we have:
\begin{equation}\label{eq:err_dec}
\begin{split}
    \Ea{\frac{1}{p}\Tr(A\Delta_\mu)} &=  \underbrace{\Ea{\Tr(A (\mathcal{E}_{1,\mu}+\mathcal{E}_{2,\mu}+\mathcal{E}_{3,\mu}) G^\mu_e  \phi^e_\mu (\phi^e_\mu)^\top)}}_{T_1}\\&+\underbrace{\frac{1}{1 + \chi_W(\mu)} \frac{1}{p}\Ea{\Tr(A G_e R^\star_\mu)}-  \frac{1}{p}\frac{1}{1 + \chi_W(\mu)} \Ea{\Tr(A G^\mu_e R_\mu)}}_{T_2} 
\end{split}
\end{equation}

We start with the term $T_1$. We proceeed by bounding the contributions from $\mathcal{E}_{1,\mu}, \mathcal{E}_{2,\mu}, \mathcal{E}_{3,\mu}$.

First, note that cyclicity of the trace implies that for $i=1,2,3$:
\begin{equation}\label{eq:ineq_e_i}
   \Tr(\frac{1}{p} A (\mathcal{E}_{i,\mu}) G^\mu_e \phi^e_\mu (\phi^e_\mu)^\top) = (\mathcal{E}_{i,\mu})\frac{1}{p}\Tr((\phi^e_\mu)^\top A  G^\mu_e \phi^e_\mu )
\end{equation}

Our strategy is to bound the tail behavior of $\abs{\mathcal{E}_{i,\mu}},\Tr(\frac{1}{p}(\phi^e_\mu)^\top A  G^\mu_e \phi^e_\mu )$ and translate them into expectation bounds on the product.
 
Let $\epsilon_{1,\mu} \coloneqq \frac{1}{p} (\phi^e_\mu)^\top G_e^\mu (\phi^e_\mu)- \frac{\Tr(G^\mu_e R_\mu)}{p}, \epsilon_{2,\mu} \coloneqq  \frac{\Tr(G^\mu_e R_\mu)}{p}-\frac{\Tr(G^\mu_e R^\star_\mu)}{p}, \epsilon_{3,\mu} \coloneqq \frac{\Tr(G^\mu_e R^\star_\mu)}{p}-\chi_d(\kappa)$.
Then:
\begin{align*}
    \mathcal{E}_{1,\mu} &= \frac{\epsilon_{1,\mu}}{(1+\frac{\Tr(G^\mu_e R_\mu)}{p}+\epsilon_{1,\mu})(1+\frac{\Tr(G^\mu_e R_\mu)}{p})}\\
    &=\frac{\epsilon_{1,\mu}}{(1+\frac{\Tr(G^\mu_e R_\mu)}{p})^2} + \frac{\epsilon^2_{1,\mu}}{(1+\frac{\Tr(G^\mu_e R_\mu)}{p}+\epsilon_{1,\mu})(1+\frac{\Tr(G^\mu_e R_\mu)}{p})},
\end{align*}
and analogous relations hold for $\mathcal{E}_{2,\mu}, \mathcal{E}_{3,\mu}$.

We start with bounding the term $\Tr(\frac{1}{p}(\phi^e_\mu)^\top A  G^\mu_e \phi^e_\mu )$. Since $\Tr(\frac{1}{p}(\phi^e_\mu)^\top A  G^\mu_e \phi^e_\mu )$ is a quadratic form in the features $\phi^e_\mu$, one could hope to exploit a Hanson-Wright type inequality. This requires $\phi^e_\mu$ to be ``well-concentrated" in some sense. Fortunately, assumption \ref{ass:activation} directly yields the following regularity property:

\begin{proposition}
Let $\xi_\mu \sim \mathcal{N}(0, I_d)$ be defined through the decomposition:
\begin{equation}
     x_\mu = \kappa w^\star + (I - w^\star (w^\star)^\top)\xi_\mu
\end{equation}
Then, w.h.p as $d \rightarrow \infty$,
    The map $\xi_\mu \rightarrow \phi^e_\mu$ is Lipschitz-continuous
\end{proposition}
The Lipschitz-continuity of $\phi_\mu$ allows us to apply a generalization of the Hanson-wright inequality for Lipschitz-maps of independent Gaussian i.i.d vectors \citep{louart2018concentration}, leading to the following result:
\begin{lemma}\label{lem:han_wright}
For any sequence of matrices $M \in \mathbb{C}^{p \times p}$, independent of $\phi^e_\mu$:
\begin{equation}
     \abs{(\phi^e_\mu)^\top M \phi^e_\mu - \Tr(M R_\mu)} =\cO_\prec(\norm{M}_F),
\end{equation}
\end{lemma} 

Applying the above Lemma with 
$M= A  G^\mu_e$ and noting that:
\begin{align*}
    \norm{A  G^\mu_e}_F &\leq 
     \norm{A  G^\mu_e}_{\Tr} \\ 
     &\leq   \norm{A}_{\Tr} \norm{G^\mu_e} \\
     &\leq   \frac{1}{\zeta}\norm{A}_{\Tr}
\end{align*}
Therefore, Lemma \ref{lem:han_wright} yields:
\begin{equation}\label{eq:err_bound1}
    \Tr{(\frac{1}{p}(\phi^e_\mu)^\top A  G^\mu_e \phi^e_\mu )} = \mathcal{O}_\prec(\frac{\norm{A}_{\Tr}\kappa^\ell}{p}).
\end{equation}

We next claim that to obtain the desired bound on $\Tr(A\Delta)$, it suffices to show that:
\begin{equation}
    \epsilon_{i,\mu} = \mathcal{O}_\prec(\frac{\kappa^\ell}{\sqrt{d}}),
\end{equation}
for $i=1,2,3$. To see this, note that $R^\star_\mu$ is positive semi-definite and $G^\mu_e$ has a positive Hermitian component. Therefore, the absolute values of the term $\frac{1}{(1+\frac{\Tr(G^\mu_e R_\mu)}{p}+\epsilon_{1,\mu})(1+\frac{\Tr(G^\mu_e R_\mu)}{p})}$ in $\mathcal{E}_{1,\mu}$ is uniformly bounded by $1$, and similarly for $\mathcal{E}_{2,\mu},\mathcal{E}_{3,\mu}$. Therefore, for $i=1,2,3$:
\begin{equation}
    \mathcal{E}_{i,\mu} = \mathcal{O}_\prec (\epsilon_{i,\mu}+\epsilon^2_{i,\mu}).
\end{equation}

Applying Cauchy-Shwartz inequality to the sequence of variables $\mathcal{E}_{i,\mu}, (z_\mu^\top A z_\mu)$ yields:
\begin{equation}
\Ea{\abs{\mathcal{E}_{i,\mu}, (z_\mu^\top A z_\mu)}} = \mathcal{O}(\kappa^\ell\frac{1}{p}\frac{1}{\sqrt{d}}),
\end{equation}
for some $\ell \in \mathbb{N}$ and $i=1,2,3$. Subsequently, combining with $\sup_\mu (\kappa_\mu) = \mathcal{O}(\sqrt{\log n})$ and summing over $\mu=1 \cdots n$, we obtain:
\begin{equation}
    \Ea{T_1} = \mathcal{O}(\frac{\polylog d}{\sqrt{d}})
\end{equation}

In light of the above discussion, we now move to bounding $\epsilon_{i,\mu}$ for $i=1,2,3$:

\begin{proposition}\label{prop:bounds}
Let $\epsilon_{1,\mu} \coloneqq \frac{1}{p}(\phi^e_\mu)^\top G_e^\mu (\phi^e_\mu)- \frac{\Tr(G^\mu_e R_\mu)}{p}, \epsilon_{2,\mu} \coloneqq  \frac{\Tr(G^\mu_e R_\mu)}{p}-\frac{\Tr(G^\mu_e R^\star_\mu)}{p}, \epsilon_{3,\mu} \coloneqq \frac{\Tr(G^\mu_e R^\star_\mu)}{p}-\chi_d(\kappa)$. Then for $i=1,2,3$, $\exists \ell \in \mathbb{N}$ such that:
 \begin{equation}
    \epsilon_{i,\mu} = \mathcal{O}_\prec(\frac{\kappa^\ell}{\sqrt{d}})
\end{equation}
\end{proposition}
\begin{proof}

We start with $\epsilon_{1,\mu}$. 
Since $\frac{1}{p}(\phi^e_\mu)^\top G_e^\mu (\phi^e_\mu)- \frac{\Tr(G^\mu_e R_\mu)}{p}$ corresponds to the deviation of the quadratic form $(\phi^e_\mu)^\top G_e^\mu (\phi^e_\mu)$ from it's mean, we may again apply the generalized Hanson-Right inequality (Lemma \ref{lem:han_wright}). Note that $\norm{G_e^\mu}_2 \leq \frac{1}{\zeta}$ by Lemma \ref{lem:op_norm} implying $\norm{G_e^\mu}_F \leq \frac{\sqrt{p}}{\zeta}$. Sine $G^\mu_e$ is independent of $\phi^e_\mu$, Lemma \ref{lem:han_wright} yields:
 \begin{equation}
    \epsilon_{1,\mu} = \mathcal{O}_\prec(\frac{\kappa^\ell}{\sqrt{d}})
\end{equation}

$\epsilon_{2,\mu}$ is directly bounded through Proposition \ref{prop:phi_moments}

Bounding the third term containing  $\epsilon_{3,\mu}$ is more challenging and will take up the bulk of the remaining discussion. 

To bound $\epsilon_{3,\mu}$, we require establishing the concentration of $\Tr(G^\mu_e R^\star_\mu)$ around $\chi_d(\kappa)$, both w.r.t $X,W$ (recall the definition of $\chi(\kappa)$ in Equation \eqref{eq:chi_mu}).

A standard way of achieving this is through the use of Martingale-based arguments \citep{bai2008large,cheng2013spectrum,hastie2022surprises}. We proceed through a similar technique, however, our analysis is complicated by the presence of structured covariance $R^\star_\mu$ and the joint-randomness over $X,W$.

\begin{lemma}\label{lem:chi_conc}
Almost surely over $\kappa_\mu \sim \mathcal{N}(0,1)$:
\begin{equation}
    \abs{\Tr(G^\mu_e R^\star_\mu)-\chi_d(\kappa_\mu)} = \mathcal{O}_\prec(\frac{\kappa_\mu^\ell}{\sqrt{d}}),
\end{equation}

\end{lemma}

\begin{proof}
   As mentioned above, the proof follows through a martingale argument. Concretely, proceed by succesively applying Burkholder's inequality w.r.t the doob martingales on filtrations generated by $X,W$ respectively. We first start by conditioning on the following high-probability event over $W$:
\begin{equation}
    \mathcal{E}_w = \{W: \norm{W}_2 = \mathcal{O}(1), \langle w_i,w_j \rangle = \mathcal{O}_\prec(\frac{1}{\sqrt{d}})\}.
\end{equation}

It is easy to check that the moments of $\chi_d(\kappa)$ are bounded, and therefore, the error in the expectation upon restriction to $\mathcal{E}_w$ can be bounded by $\mathcal{O}_\prec(\frac{1}{d}^k)$ for arbitrarily large $k$.

Let $\mathbb{E}_\mu$ denote the conditional expectation w.r.t the sigma-algebra generated by ${x_\mu}_{\mu' < \mu}$. We apply the following martingale decomposition:

\begin{align*}
&\Tr(G_e R^\star_\kappa) - \Ea{\Tr(G_eR^\star_\kappa)}\\
&= \sum_{\mu=1}^n \mathbb{E}_{\mu}{\Tr(G_eR^\star_\kappa)} - \mathbb{E}_{\mu-1}{\Tr(G_e R^\star_\kappa)}.
\end{align*}

Let $e_\mu = \mathbb{E}_{\mu}{\Tr(G_e R^\star_\kappa)} - \mathbb{E}_{\mu-1}{\Tr(G_e R^\star_\kappa)}$. We have:
\begin{equation}
    \mathbb{E}_{\mu}{(\Tr(G_e R^\star_\kappa)} - \mathbb{E}_{\mu-1}{\Tr(G_e R^\star_\kappa))}=   \mathbb{E}_{\mu}{(\Tr(G_e R^\star_\kappa)-\Tr(G^\mu_e R^\star_\kappa))}-  \mathbb{E}_{\mu-1}{(\Tr(G_e R^\star_\kappa)-\Tr(G^\mu_e R^\star_\kappa))},
\end{equation}
where we used that $\mathbb{E}_{\mu-1}{\Tr(G^\mu_e R^\star_\kappa))}= \mathbb{E}_{\mu}{\Tr(G^\mu_e R^\star_\kappa))}$ since $G^\mu_e$ does not depend on $x_\mu$ 

Applying the Sherman-Morrison formula yields:
\begin{equation}
    \mathbb{E}_{\mu}{\Tr(G_e R^\star_\kappa)-\Tr(G^\mu_e R^\star_\kappa)} = - \frac{1}{p}\mathbb{E}_{\mu}{\frac{1}{1+ (\phi^e_\mu)^\top G_e^\mu \phi^e_\mu}  (\phi^e_\mu)^\top R^\star_\kappa\phi^e_\mu}.
\end{equation}
Now, $\abs{\frac{1}{1+ (\phi^e_\mu)^\top G_e^\mu \phi^e_\mu}} \leq 1$ while $(\phi^e_\mu)^\top R^\star_\kappa\phi^e_\mu$ has moments bounded polynomially in $\kappa$ by  Lemma \ref{lem:han_wright}.

Therefore, Lemma \ref{lem:burk} combined with Markov's inequality for large enough $p$, implies:
\begin{equation}
    \abs{\frac{1}{p}\Tr(G^\mu_e R^\star_\mu)-\chi_W(\kappa)} = \mathcal{O}_\prec(\frac{1}{\sqrt{d}})
\end{equation}
Where:
\begin{equation}
\chi_W(\kappa) 
 \bydef \Eb{X}{\Tr(G_e R^\star_\kappa)},    
\end{equation}

where the expectation is only w.r.t the matrix $X$.

It remains to show by establishing concentration w.r.t $W$ that:
\begin{equation}
    \abs{\chi_W(\kappa)-\chi(\kappa)} = \mathcal{O}_\prec(\frac{1}{\sqrt{d}}),
\end{equation}

We again condition on the following high-probability event:
\begin{equation}
    \mathcal{E}_x = \{X: \norm{W}_2 = \mathcal{O}(\sqrt{d}), \langle w_i,w_j \rangle = \mathcal{O}_\prec(1)\}.
\end{equation}

 We again apply a martingale argument, except that unlike the concentration w.r.t $X$, both $G_e, R^\star$ now depend on $W$.

Let $\mathbb{E}_i$ denote the conditional expectation w.r.t the sigma-algebra generated by ${w_j}_{j < i}$. Let $G_e(-i)$ denote the extended resolvent obtained after the removal of the $(k+1+i)_{\text{th}}$ row and column in $(\Phi^e)^\top\Phi^e$, except for the diagonal $(k+1+i), (k+1+i)$ entry. Note that this corresponds to a finite-rank perturbation:
\begin{equation}
    \frac{1}{p}(\Phi^e)^\top\Phi^e-e_{(k+1+i)}b^\top_i-b_ie_{(k+1+i)}^\top, 
\end{equation}
where $b^\top_i$ denotes the (normalized) $(k+1+i)_{\text{th}}$ row of $(\Phi^e)^\top\Phi^e$ given by:
\begin{equation}
    b^\top_i = \frac{1}{p}(\Phi^e_{i:})^\top (\Phi^e)
\end{equation}

Analogously, let $R^\star_\kappa(-i)$ be obtained by the removal of the $(k+1+i)_{\text{th}}$ row and column in $R^\star_\kappa$ (except for the diagonal $(k+1+i), (k+1+i)$ entry).

Using the Woodbury-matrix identity, we have:
\begin{equation}
    \begin{split}
        G_e = G_e(-i)-G_e(-i) \begin{bmatrix}
            e_{(k+1+i)} & b_i 
        \end{bmatrix}
        (\begin{bmatrix}
            0 & 1\\
            1 & 0
        \end{bmatrix}+\Psi
         )^{-1}\begin{bmatrix}
            e_{(k+1+i)} & b_i 
        \end{bmatrix}^\top G_e(-i)
    \end{split},
\end{equation}
where $\Psi = \begin{bmatrix}
            e_{(k+1+i)} & b_i 
        \end{bmatrix}^\top G_e(-i) \begin{bmatrix}
            e_{(k+1+i)} & b_i 
        \end{bmatrix} \in \mathbb{R}^{2\times 2}$

Let $e_i = \mathbb{E}_{i}\Eb{x}{\Tr({G_e} R^\star_\kappa)} - \mathbb{E}_{i-1}{\Eb{x}{\Tr({G_e} R^\star_\kappa)}}$. We express $\Tr({G_e} R^\star_\kappa)$ as:
\begin{align*}
    \Tr({G_e} R^\star_\kappa) = \Tr(G_e(-i)R^\star_\kappa(-i))+\Delta, 
\end{align*}
where $\Delta$ arises from the low-rank projections.

To Apply Lemma \ref{lem:burk}, it remains to bound the second moments of the residual terms $\Delta$.  It is easy to check that conditioned on the high-probability events $\mathcal{E}_w,\mathcal{E}_x$,
for any uniformly Lipschitz $f:\mathbb{R}^p \rightarrow \mathbb{R}$, $f(\frac{1}{\sqrt{p}}b_i)$ is a 
uniformly Lipschitz map of $w_i$. Therefore, Lemma \ref{lem:sphere_lips} implies that $\sqrt{p}b_i$ is a Lipschitz-concentrated vector in the sense of Definition 9 in \cite{couillet2022random}.
Therefore the generalized Hanson-Wright inequality implies that:
\begin{equation}
    \abs{b_i^\top G_e(-i) b_i-\Ea{b_i^\top G_e(-i) b_i}} = \mathcal{O}_\prec(\frac{1}{\sqrt{d}}).
\end{equation}
Similarly, we obtain that the remaining entries of $\Psi$ concentrate around their expectations. As a result, $\delta$ decomposes into quadratic forms in $e_{(k+1+i)},b_i$ which can be bounded as $\mathcal{O}_\prec(\frac{1}{p})$. Applying Lemma \ref{lem:burk} then completes the proof.
\end{proof}

The proof of Lemma \ref{lem:chi_conc} completes the proof of Proposition \ref{prop:bounds} and thus bounds $T_1$ in Equation \eqref{eq:err_dec}
\end{proof}
We now return to bounding the term  $T_2$ in Equation \eqref{eq:err_dec}.
Let $\zeta$ be as defined in Lemma \ref{lem:op_norm}
First, Proposition \ref{prop:phi_moments} and Lemmas \ref{lem:op_norm}, \ref{lem:trace_in} imply that:
\begin{equation}
    \abs{\Tr(A G^\mu_e R_\mu)-\Tr(A G^\mu_e R^\star_\mu)} \leq \mathcal{O}_\prec(\frac{1}{\sqrt{d}}).
\end{equation}
Therefore, we may replace $R_\mu$ by $R^\star_\mu$ at the cost of an error $\mathcal{O}_\prec(\frac{1}{\sqrt{d}})$.

Next, we apply the Sherman-Morrison formula to $G_e-G^\mu_e$ to obtain
\begin{equation}
    \Ea{\Tr(A G_e R^\star_\mu)}-\Ea{\Tr(A G^\mu_e R^\star_\mu)} = \frac{1}{p}\Ea{\frac{1}{1+(\phi^e_\mu)^\top G^\mu_e \phi^e_\mu}\Tr(A \phi^e_\mu (\phi^e_\mu)^\top R^\star_\mu)}
\end{equation}
By the cyclicity of trace, the second term can be expressed as the quadratic form $(\phi^e_\mu)^\top R^\star_\mu A \phi^e_\mu$

We start by noting that, Lemmas \ref{lem:op_norm}  \ref{lem:trace_in} imply:
\begin{equation}
    \Ea{\Tr(A G_e R_\mu)} \leq  \norm{A G_e R_\mu}_{\Tr}   \leq \norm{A}_{\Tr} \norm{G_e R_\mu} \leq \frac{C}{\zeta},
\end{equation}
for some constant $C$.

Therefore,  using Lemmas \ref{lem:diff_inv} and \ref{lem:trace_in}, we obtain::
\begin{equation}
     \abs{\frac{1}{1 + \chi_W(\mu)} \Ea{\Tr(A G_e R_\mu)}-  \frac{1}{1 + \chi_W(\mu)} \Ea{\Tr(A G^\mu_e R_\mu)}} =\mathcal{O}_\prec(\frac{\norm{A}_{\Tr}}{p}).
\end{equation}

We therefore obtain:
\begin{equation}
     \Ea{\Tr(A\Delta_\mu)} = \mathcal{O}(\frac{1}{\sqrt{d}})
\end{equation}

The convexity of $\abs{}$ then implies that 

\begin{equation}
    \Ea{\Tr(AG_e)} = \Ea{\Tr{(A(\alpha R^\star + \lambda I)^{-1})}} +\mathcal{O}(\frac{\polylog d}{\sqrt{d}}).
\end{equation}

Next, we show that the averaging over $\kappa$ allows us to replace $R^\star$ with $\Sigma^\star$:
\begin{lemma}
There exists $\ell \in \mathbb{N}$ such that almost surely over $\kappa$:
\begin{equation}
    \chi_\mu(\kappa) = \mathcal{O}(\kappa^\ell)
\end{equation}
\end{lemma}
\begin{proof}
    Recall that $G^e_\mu$ doesn't depend on $\kappa$, while $R^\star_\mu$ depends on $\kappa$ only through scalars, $c_j(\kappa,u_i)$ for $j=1,2,3$. The claim then directly follows using assumption \ref{ass:activation}.
 \end{proof}

Since $\kappa_\mu$ for $\mu \in [n]$ are independent Gaussians, the above Lemma implies that the covariance $\sum_{\mu \in [n]} \frac{1}{1+\chi_W(\mu)} R^\star(\kappa_\mu)$ 
can further by replaced by $\Sigma^\star$ in proposition \ref{prop:first_stag} with additional error $\mathcal{O}(\frac{1}{\sqrt{d}})$.

It remains to show that $\Tr(G_eA)$ concentrates around its expectation w.r.t $X$. This follows from a martingale argument identical to the first part of the  proof of Lemma \ref{lem:chi_conc}, with the role of $R^\star$ being replaced by $A$.
Then, we have the following martingale decomposition:

\begin{align*}
&\Tr(G_eA) - \Ea{\Tr(G_eA)}= \sum_{\mu=1}^n \mathbb{E}_{\mu}{\Tr(G_eA)} - \mathbb{E}_{\mu-1}{\Tr(G_e A)}.
\end{align*}
Let $e_\mu = \mathbb{E}_{\mu}{\Tr(G_eA)} - \mathbb{E}_{\mu-1}{\Tr(G_e A)}$. We have:
\begin{equation}
    \mathbb{E}_{\mu}{(\Tr(G_eA)} - \mathbb{E}_{\mu-1}{\Tr(G_e A))}=   \mathbb{E}_{\mu}{(\Tr(G_eA)-\Tr(G^\mu_e A))}-  \mathbb{E}_{\mu-1}{(\Tr(G_eA)-\Tr(G^\mu_e A))}
\end{equation}
where $G^\mu_e$ denotes the resolvent with the $\mu_{th}$ example removed. Next, exactly as in Lemma \ref{lem:chi_conc}, we apply the generalized Hanson-Wright inequality to show that:
\begin{equation}
    (\Tr(G_eA)-\Tr(G^\mu_e A)) = \mathcal{O}_\prec(\frac{1}{p}).
\end{equation}

Lemma \ref{lem:burk} for $p=4$ and Markov's inequality then imply that, with high-probability:
\begin{equation}
    \Tr(G_eA) = \Ea{\Tr(G_eA)} +\mathcal{O}_\prec(\frac{\norm{A}_{\Tr}}{\sqrt{d}}),
\end{equation}
completing the proof of Proposition \ref{prop:first_stag}

\end{proof}

\subsection{Second stage of Deterministic Equivalent (Averaging over $W$)}
\label{app:2ndstage}

Substituting $R^\star_\mu$ from Proposition \ref{prop:phi_moments}, we obtain that $\Sigma^*$ in proposition \ref{prop:first_stag} posseses the following structure:

\begin{equation}
\Sigma^*  =  \begin{pmatrix}
\bA^\ast_{11} & (\bA^\ast_{21})^\top \odot \btheta^\top \\
           \btheta \odot \vec{A}^{\ast}_{21} & \scriptstyle (\bV_e \odot \bW \bW ^\top \!+ \operatorname{diag}(\nu_e)))+\alpha S_e \odot \theta\theta^\top ,
\end{pmatrix}
\end{equation}
where:
\begin{align*}
    \bA^\ast_{11} &= \Eb{\kappa}{\frac{\alpha}{1 + \chi(z;\kappa)}R^\star_{11}(\kappa)}\\
    &= \Eb{\kappa}{ \frac{\alpha}{1 + \chi(z;\kappa)}\begin{bmatrix}
        \sigma_\star^2(\kappa) & \sigma_\star(\kappa) c_0(\kappa,u_\pi)^\top\\
        \sigma_\star(\kappa) c_0(\kappa,u_\pi) & c_0(\kappa,u_\pi) c_0(\kappa,u_\pi)^\top
    \end{bmatrix}}\end{align*}
\begin{align*}
    \bA^\ast_{11} &=  \Eb{\kappa}{\frac{\alpha}{1 + \chi(z;\kappa)}R^\star_{21}(\kappa)}
\end{align*}

and $\bV, \nu_e, S$ are defined as:
\begin{align*}
    V^{(d)}_{i,j}(\zeta) &= \Eb{\kappa}{\alpha\frac{(c_1(\kappa, u_i)c_1(\kappa, u_j)}{1+\chi(s,m)}}\\
    \nu^{(d)}_{i}(\rho) &=\Eb{\kappa}{(\sum_{k \geq 2}\frac{\alpha c^2_k(\kappa, u_i)}{1+\chi(\kappa,W)}}\\
    S &= \Ea{(\kappa^2-1)\Eb{\kappa}{\frac{(c_1(\kappa, u_i)c_1(\kappa, u_j)}{1+\chi(\kappa)}}}, 
\end{align*}
with $\kappa \sim \mathcal{N}(0,1)$ throughout.

Note that $\bA^\ast_{11}, \bA^\ast_{21}$ are deterministic, while the terms dependent on $\btheta$, including $\delta$ contribute finite-rank spikes. Therefore, the bulk statistics of $\Sigma^\star$ arise out of the term $\bV_e \odot \bW \bW$.

To average-out the randomness over $\bW$ in , it will be convenient to first extract a deterministic equivalent for the following matrix:

\begin{equation}
    M^* = (\bV^{(d)}_e \odot \bW \bW ^\top \!+ \operatorname{diag}(\nu^{(d)}_{e})-z\bI_p)^{-1}
\end{equation}
where $V_e, \operatorname{diag}(\nu^{(d)})_e$ denote the extended block-structured matrices as per definition \ref{def:block_ext}.

Additionally, to express $\chi_W(\kappa)$ self-consistently in terms of $V,\nu$, we will require the following additional functional:
\begin{equation} 1/d*\Tr( e_i \odot W M^\ast W \odot e_j) 
\end{equation}

% \yd{Here also handle concentration w.r.t $W$}
\begin{lemma}[Deterministic Equivalent for Block-Structured Wishart]\label{lemma:block_res}
Let $C(z) \in \mathbb{C}^{k \times k}, D(z) \in \mathbb{C}^{k}$ be analytic mappings such that $C(z): \mathbb{C}^{+} \rightarrow \mathbb{C}^{-}$  and  $D(z): \mathbb{C}^{+} \rightarrow \mathbb{C}^{-}$ entry-wise with $\abs{C_{i,j}}, \abs{D_i}$ uniformly bounded by some constant independent of $\zeta$. Furthermore, suppose that $D(z)$ is diagonal.
Let ${C}_e(z), {D}_e \in \mathbb{C}^{p \times p}$ denote the extended matrices as defined in Definition \ref{def:block_ext}. Let $ R_{C,D}$ denote the block structured resolvent defined as:
\begin{equation}\label{eq:def_m^*}
  R_{C,D} \bydef (({C}_e) \odot \tilde{{W}}^0({\tilde{{W}}^0})^\top + \operatorname{diag}({D}_e)  -zI_p)^{-1}.
\end{equation}
Define ${\mathcal{M}^*}({C}_e, {D}_e)$ as the diagonal matrix:
\begin{equation}
    {\mathcal{M}^*}({C}_e, {D}_e) \coloneqq \operatorname{diag}(\frac{b^\star}{\pi \beta}),
\end{equation}
where above, $b^\star, \pi \in \mathbb{R}^k$ are divided element-wise and $b^\star$ satisfies the following self-consistent equation:
\begin{align}
b^\star(C,D)_q = \pi_{u_i} \beta ((C^{-1}+ \operatorname{diag}({b}^\star))^{-1}+ (\operatorname{diag}(D)-z I_p))^{-1}_{q,q},
\end{align}
for $q \in [p]$.
Then for any sequence of matrices $A \in \mathbb{C}^{p \times p}$:
\begin{equation}
    \abs{\Tr{{A\mathcal{M}^\star}}-\Tr{{AR_{C,D}}}} = \mathcal{O}_\prec(\frac{\norm{A}_{\Tr}}{\sqrt{d}})
\end{equation}

Furthermore, the expression:
\begin{equation}
   K^*_{i,j}(C,D) = 1/d*\Tr( e^i \odot W M^\ast W \odot e^j) 
\end{equation}
satisfies the following deterministic equivalence:
\begin{align*}
    K^* = \psi(C,D) +\mathcal{O}_\prec(\frac{1}{\sqrt{d}})
\end{align*}
where $\psi(C,D)$ is defined as:
\begin{equation}\label{eq:def_psi}
    \psi(C,D)=b^\star(C,D) - L(C,D) \odot (b^\star(C,D) (b^\star(C,D))^\top),
\end{equation}
with:
\begin{equation}
    L(C,D) = \left((C)^{-1} + \operatorname{diag}({b^\star(C,D)})\right)^{-1}.
\end{equation}
\end{lemma}

\begin{proof}
    The proof relies on the observation that due to the block structure in $V$, removing a coordinate in $\bW \bW ^\top$ results in a finite rank perturbation to $M^\star$. We then proceed through a leave-one out argument similar to section \ref{app:proof:one}. The proof is deferred to section \ref{app:extra}.
\end{proof}

In light of Lemma \ref{lemma:block_res}, we obtain the following candidate for the deterministic equivalent to $G_w(z)$:
\begin{equation}
     (\frac{\alpha}{p}\tilde{\Sigma^\star} - zI),
\end{equation}
where $\tilde{\Sigma^\star}$ is obtained by replacing $(\bV^{(d)}_e \odot \bW \bW ^\top \!+ \operatorname{diag}(\nu^{(d)}_{e})+\lambda \bI_p) = (M^*)^{-1}$ by $\mathcal{M}^\star(V,\diag(\nu))^{-1}$.

However, while obtaining $\mathcal{M}^\star(V,\diag(\nu))^{-1}$, we averaged over the dependence on $\theta=W w^\star$ in addition to the remaining components of $W$. This is undesirable since the dependence on $\theta$ captures the correlations amongst blocks $\Sigma_{22}, \Sigma_{21}$ in the extended resolvent which will be relevant for the characterization of the generalization error. We obtain back this dependence through in the following result:
\begin{lemma}\label{lem:mthet}
Let $M^\star_d =(\bV^{(d)}_e \odot \bW \bW ^\top \!+ \operatorname{diag}(\nu^{(d)}_{e})+\lambda \bI_p)^{-1}$. Define:
\begin{equation}\label{eq:mtheta}
    M^\theta_d \coloneqq (\operatorname{diag}(\frac{b^\star}{\pi\beta})_e  -(\bV_d^{-1}+\operatorname{diag}(b^\star))^{-1} \odot (\frac{b^\star}{\pi \beta})_e (\frac{b^\star}{\pi \beta})_e^\top  \odot \theta \theta^\top),
\end{equation}
where $\psi(V_d,\nu_d)$ is as defined in Lemma \ref{lemma:block_res}.
    For any sequence of matrices $A \in \mathbb{C}^{p \times p}$, possibly dependent on $\theta=W w^\star$:
    \begin{equation}\label{eq:det_eq_ms}
        \Tr{(M^\star_d A)} = \Tr{(M^\theta_d A)} + \mathcal{O}_\prec (\frac{\norm{A}_{\Tr}}{\sqrt{d}}). 
    \end{equation}
\end{lemma}

\begin{proof}
    
To obtain back the dependence on $\theta$, we write:
\begin{equation}
    M^*_d = (\bV^{(d)}_e \odot (\bW_\bot (\bW_\bot) ^\top)+ \bV^{(d)}_e \odot \theta\theta^\top) \!+ \operatorname{diag}(\nu^{(d)}_{e})+\lambda \bI_p)^{-1},
\end{equation}
where $\bW_\bot=\bW-\theta (w^\star)^\top$ denotes the components of the weights upon the removal of the components $\theta$ along $w^\star$. 

Next, note that $\bV^{(d)}_e \odot \theta\theta^\top$ is a finite-rank perturbation along directions $e^1,\cdots,e^k$ defined in Equation \eqref{eq:def:eq}. Concretely,
let $E \in \R^{p \times k}$ denote the matrix with columns $e^1, \cdots e^k$  and define $E_{\theta}= \theta \odot E$. Then, $\bV^{(d)}_e \odot \theta\theta^\top$ can be expressed as:

\begin{equation}\label{eq:blockmult}
    \bV^{(d)}_e \odot \theta\theta^\top = E_{\theta}\bV_d E_{\theta}^\top.
\end{equation}
Therefore, we apply the Woodbury matrix identity to obtain:
\begin{equation}
     M^* = \tilde{M}- \tilde{M}E_{\theta}(\bV_d^{-1}+ E^\top_{\theta}\tilde{M} E_{\theta})^{-1} E_{\theta}\tilde{M}
\end{equation},
where $\tilde{M}=(\bV^{(d)}_e \odot(\bW_\bot (\bW_\bot) ^\top)+ \operatorname{diag}(\nu^{(d)}_{e})-z\bI_p)^{-1}$ denotes the inverse upon the removal of components alog $\theta$. By the rotational invariance of $w_i$, $\tilde{M}$ asymptotically shares the deterministic equivalent for $M^\star$ described in Lemma \ref{lemma:block_res}. 
Since $\tilde{M}$ is independent of $\theta$, and $\theta_i$ are asymptotically distributed as $\mathcal{N}(0,\frac{1}{d})$, 
$E^\top_{\theta}\tilde{M} E_{\theta}$ in turn simplifies to $\operatorname{diag}(b^\star(V_d, \nu_d))$. We then replace the occurances of $\tilde{M}$ with $\mathcal{M}^\star(V_d,\nu_d)=\frac{b^\star}{\pi \beta}$ to obtain:
\begin{equation} \operatorname{diag}(\frac{b^\star}{\pi \beta})- \operatorname{diag}(\frac{b^\star}{\pi \beta})_e E_{\theta}(\bV_d^{-1}+\operatorname{diag}(b^\star))^{-1} E_{\theta}\operatorname{diag}(\frac{b^\star}{\pi \beta})_e.
\end{equation}
Subsequently, analogous to Equation \eqref{eq:blockmult}, we have the following equivalent representation of the middle-block:
\begin{equation}
    E_{\theta}(\bV_d^{-1}+\operatorname{diag}(b^\star))^{-1} E_{\theta} = (\bV_d^{-1}+\operatorname{diag}(b^\star))^{-1}_e \odot \theta\theta^\top.
\end{equation}
We thus obtain:
\begin{equation}
    M^\star \simeq \operatorname{diag}(\frac{b^\star}{\pi \beta})- (\bV_d^{-1}+\operatorname{diag}(b^\star))^{-1} \odot (\frac{b^\star}{\pi \beta})_e (\frac{b^\star}{\pi \beta})_e^\top \odot \theta \theta^\top
\end{equation}
where $\simeq$ denotes deterministic equivalence in the sense of Equation \eqref{eq:det_eq_ms}, which follows from proposition \ref{lemma:block_res}.
\end{proof}

We can finally claim the second level of deterministic equivalence, by replacing $\mathcal{G}_W(z)$ with a sequence of matrices depending only on $\theta, u$: 
\begin{proposition}\label{prop:sec_stag}
Consider the sequence of (random) resolvents defined by:
\begin{equation}
\mathcal{G}_W(z)=\left(\frac{\alpha}{p} 
        \Sigma^\star-zI\right)^{-1},
    \end{equation},
and:
\begin{equation}\label{eq:new_det_def}
  \tilde{\mathcal{G}}_e(z) =   \begin{pmatrix}
\bA^\ast_{11} -z I_{k+1}& (\bA^\ast_{21})^\top \odot \btheta^\top \\
           \btheta\odot \vec{A}^{\ast}_{21} & (\frac{b^\star}{\pi\beta}-\frac{b^\star}{\pi\beta} (\bV^{-1}_d+\operatorname{diag}(\frac{b^\star}{\pi \beta}))^{-1}_e \odot \theta\theta^\top
    \frac{b^\star}{\pi\beta})
        \end{pmatrix}
\end{equation}
Then, for any sequence of deterministic matrices $A$:
\begin{equation}
    \abs{\Tr(\mathcal{G}_W(z)A)-\Tr(\tilde{\mathcal{G}}_e(z)A)} = \mathcal{O}_\prec(\frac{\norm{A}_{\text{tr}}}{\sqrt{d}})
\end{equation}
\end{proposition}
\begin{proof}
    Applying Lemma \ref{lem:diff_inv} twice, we obtain:
    \begin{align*}
        \mathcal{G}_W(z)-\tilde{\mathcal{G}}_e(z) &=  \frac{\alpha}{p}\mathcal{G}_W(z)((M^*_d)^{-1}-M^{\theta^{-1}_d})\tilde{\mathcal{G}}_e(z)\\
        & \frac{\alpha}{p}\mathcal{G}_W(z)(M^*_d)^{-1}(M^*-M^{\theta^{-1}_d})M^{\theta^{-1}_d}\tilde{\mathcal{G}}_e(z)
  \end{align*}
By Lemma \ref{lem:op_norm}, $\norm{\mathcal{G}_W(z)}\norm{\tilde{\mathcal{G}}_e(z)}$ are bounded by $\frac{1}{\zeta}$ while the norms of $\norm{(M^*_d)^{-1}}, \norm{(M^\theta)^{-1}_d}$ are bounded by constants due to Lemma \ref{lem:operator-norms}.
Therefore:
\begin{equation}
    \abs{\Tr(\mathcal{G}_W(z)A)-\Tr(\tilde{\mathcal{G}}_e(z)A)} \leq \frac{C}{\zeta^2} \abs{\Tr(\mathcal{G}_W(z)M^\star)-\Tr(\tilde{\mathcal{G}}_e(z) M^\theta_d)}
\end{equation}
Applying Lemma \ref{lem:mthet} then completes the proof. 
\end{proof}

\subsection{Self-consistent Equation for $\chi(\kappa)$}

$\tilde{\mathcal{G}}(z)$ is almost identical to the desired equivalent matrix $\mathcal{G}_e(z)$ in Theorem \ref{thm:det_eq}, except for $\chi_d(\kappa)$ still being dimension-dependent unknowns. We resolve this by utilizing \ref{prop:first_stag} to obtain a self-consistent equation for $\chi(\kappa)$

\begin{lemma}\label{lem:chi_self}
Let $\psi,b^\star$ be as defined in Lemma \ref{lemma:block_res}. Then, almost surely over $\kappa \sim \mathcal{N}(0,1)$, the following holds:
\begin{align*}
   \chi_d(\kappa,z)=  \sum_{q,q' \in[k]} \psi_{qq'}(V_d,\nu_d) c_1(\kappa,\zeta^{u}_{q})c_1(\kappa,\zeta^{u}_{q'})  + \sum_{q\in[k]}b^\star_{q}((V_d,\nu_d))\sum_{\ell\ge 2}  c_{\ell}^2(\kappa,\zeta^{u}_{q})+\mathcal{O}_\prec(\frac{\kappa^\ell}{\sqrt{d}}).
\end{align*}
\end{lemma}

\begin{proof}
We first recall that $\norm{R^\star_\kappa}_2 = 
\mathcal{O}_\prec(\kappa^\ell)$ for some $\ell \in \mathbb{N}$. Now, note that $\chi_d(\kappa) = \Ea{\Tr( \tilde{\mathcal{G}}_e(z)R^\star_\kappa)}$ involves dependency on $W$ in $R^\star_\kappa$. Therefore, we cannot directly apply Proposition \ref{prop:first_stag} and must resort to Proposition \ref{prop:first_stag}.
 we have:
\begin{equation}
     \chi_d(\kappa) = \Ea{\frac{1}{p}\Tr( \mathcal{G}_W(z)R^\star_\kappa)} +\mathcal{O}_\prec(\frac{\kappa^\ell}{\sqrt{d}})
\end{equation}
    We start by expressing $\mathcal{G}_W(z)$ through a Schur-complement decomposition:
    \begin{equation}\label{eq:schur_chi}
    \mathcal{G}_W(z) = \begin{bmatrix}
        C_W & -C_W Q_W^\top\\
        -Q_W C_W & \tilde{G}_W + Q_QC_WQ_W^\top,
    \end{bmatrix},
    \end{equation},
where:
\begin{equation}
    \tilde{G}_W = ((\bV_e \odot \bW \bW ^\top \!+ \diag{\nu_d}+\lambda \bI_p))+\alpha S \odot \theta\theta^\top-zI)^{-1},
\end{equation}
and $C_W, Q_W$ in Equation \eqref{eq:schur_chi} contribute only finite-rank components of bounded operator norm. Since $\norm{R^\star_\kappa}_2 = 
\mathcal{O}_\prec(\kappa^\ell)$, such components contribute $\mathcal{O}(\frac{1}{p})$ to $\Ea{\frac{1}{p}\Tr{ \mathcal{G}_W(z)R^\star_{22}(\kappa)}}$ and can therefore be ignored. We're left with:
    \begin{equation}
        \frac{1}{p}\Ea{\Tr{\tilde{G}_W R^\star_{22}(\kappa)}}
    \end{equation},
which evaluates to:
\begin{align*}
&\underbrace{\Ea{\Tr{(c_1(\kappa, u)c_1(\kappa, u)^\top) \odot W W ^\top )(\bV_e \odot \bW \bW ^\top \!+ \diag{(\nu_d)}+\lambda \bI_p)^{-1}}}}_{T_1}\\&+ \underbrace{\Ea{\Tr{\operatorname{diag}(\bigg(\sum_{k \ge 2} c^2_k(\kappa,u)\bigg) (\bV_e \odot \bW \bW ^\top \!+ \diag{(\nu_d)}+\lambda \bI_p)^{-1}}}}_{T_2}+\mathcal{O}(\frac{1}{p}),
\end{align*}
where we again supressed contributions from finite-rank terms. Since $\operatorname{diag}(\bigg(\sum_{k \ge 2} c^2_k(\kappa,u)\bigg)$ is independent of $W$, by Lemma \ref{lemma:block_res}, $T_2$ simplifies to:
\begin{equation}
   \sum_{q\in[k]}b^\star_{q}((V_d,\nu_d))\sum_{\ell\ge 2}  c_{\ell}^2(\kappa,\zeta^{u}_{q})
\end{equation}
While $T_1$ can be decomposed into terms of the form $\Tr( e^i \odot W M^\ast W \odot e^j)$, which converge to $\psi(V,\nu)$ by Lemma \ref{lemma:block_res}.

\end{proof}

% By isotropy,  $\Sigma^*_{2,2}$  is asymptotically equivalent to $\vec{B}^\top$ along the additional spikes due to $\vec{\Delta}$.

\subsection{Self-consistent equations for $V^\star, D^\star$}

Using Lemma \ref{lemma:block_res}, we obtain a deterministic equivalent $\mathcal{G}'_e$ that doesn't depend on the realizations of $X,W$.
However, the quantities $V_d, D_d$
still depend on dimension $d$ due to the dimension dependent definition of $\chi_d(\kappa)$. While \ref{lem:chi_self} defines a self-consistent equation for $\chi(\kappa)$, it remains an infinite-dimensional (functional) order parameter. Instead, we show that we can directly construct dimension-independent self-consistent equations on $V_d, \nu_d$: 

\begin{proposition}\label{prop:fix_pt_sat}
For $C,D \in \R^{k,k}$, let $\psi_{i,j}$ be as defined in Theorem \ref{thm:det_eq}. Define:
\begin{equation}
    \chi_{C,D}(\kappa) = \frac{1}{\beta}\sum_{i=1,j=1}^{k} \psi_{i,j}(C,D)c_1(u_i,\kappa)c_1(u_j,\kappa)+\frac{1}{\beta}\sum_{i=1}^{k}(b_{ii}(\sum_{k\ge 2}  c_k^2(\kappa, u_i))
\end{equation}
Then,  $V, D$ defined as: 
\begin{align*}
    V = \Ea{\frac{\alpha}{1+\chi_d(\kappa)}c_1(u_i,\kappa) c_1(u_J,\kappa)},
\end{align*}
\begin{align*}
    D= \Ea{\frac{\alpha}{1+\chi_d(\kappa)} \sum_{k \geq 2}c^2_k(u_i,\kappa) }
\end{align*}
satisfy:
\begin{equation}
  V_{i,j}  = \Ea{\frac{\alpha}{1+\chi_{V, D}(\kappa)}c_1(u_i,\kappa) c_1(u_J,\kappa)} + \mathcal{O}_{\prec}(\frac{1}{\sqrt{d}})
\end{equation}
\begin{equation}
  D_{i,i}  = \Ea{\frac{\alpha}{1+\chi_{V, D}(\kappa)}\sum_{k \geq 2}c^2_k(u_i,\kappa)}+ \mathcal{O}_{\prec}(\frac{1}{\sqrt{d}}).
\end{equation}
\end{proposition}

\begin{proof}
Note that $\frac{\alpha}{1+\chi_{V, D}(\kappa)}$ is uniformly Lipschitz in $\kappa$. Consider independent sample $\kappa_1, \cdots, \kappa_n'$ for some $n' \propto d$. Then $\sup \abs{\kappa_i} = \sqrt{\log d}$ and:
\begin{equation}
    \Ea{\frac{\alpha}{1+\chi_{V, D}(\kappa)}c_1(u_i,\kappa) c_1(u_J,\kappa)} = \frac{1}{n'} \sum_{i=1}^{n'} \frac{\alpha}{1+\chi_{V, D}(\kappa_i)}c_1(u_i,\kappa) c_1(u_J,\kappa) + \mathcal{O}_\prec(\frac{1}{\sqrt{d}}).
\end{equation}
   Proposition \ref{prop:fix_pt_sat} then follows from Proposition \ref{prop:sec_stag} and the self-consistent equation for $\chi(\kappa)$ in Lemma \ref{lem:chi_self}.
\end{proof}

Next, we show that the above fixed point equations are contractive,
allowing us to translate approximate satisfiability to distance of $V_d, D_d$ from the unique fixed points
satisfied by $V^\star, D^\star$ defined in Theorem \ref{thm:det_eq}:
\begin{lemma}\label{lem:contrac}
Let $z \in \mathbb{C}^+$. Define $\mathcal{S}(z)$ as the set of $C,D \in {\mathbb{C}^{-}}^{k \times k} \times {\mathbb{C}^{-}}^k$ satisfying:
\begin{itemize}
    \item $b^\star(C,D) \in \mathbb{C}^+$.
    \item $\abs{b^\star(C,D)} \leq \frac{\pi_i}{\beta\zeta}$
\end{itemize}
Then, there exists $C>0$ such that for $\zeta = \text{Im}(z)> C$, the fixed point iteration defined in Proposition \ref{prop:fix_pt_sat} i.e:
    \begin{align*}
        &F: {\mathbb{C}^{-}}^{k \times k} \times {\mathbb{C}^{-}}^k \rightarrow {\mathbb{C}^{-}}^{k \times k} \times {\mathbb{C}^{-}}^k\\
        &F(C,D) = \Ea{\frac{\alpha}{1+\chi_{V, D}(\kappa)}c_1(\vec{u},\kappa) c_1(\vec{u},\kappa)^\top},  \Ea{\frac{\alpha}{1+\chi_{C, D}(\kappa)}\sum_{k \geq 2}c^2_k(\vec{u},\kappa)},
     \end{align*}
     is contractive in $\mathcal{S}(z)$, i.e for any $C, D$:
     \begin{equation}
         \norm{F(C_1,D_1)-F(C_2,D_2)} \leq \norm{(C_1, D_1)- (C_2, D_2)}
     \end{equation}
\end{lemma}

\begin{proof}
From the definition of $\chi(\kappa)$, we directly obtain that for any $C,D$ such that $\abs{b} \leq \frac{\beta\pi}{\zeta}$
 \begin{equation}\label{chi:cont}
     \abs{\chi(\kappa, b_1)-\chi(\kappa, b_2)} \leq C_1\norm{b_1-b_2} +C_2\abs{b}\norm{C_1, D_1- C_2, D_2}
 \end{equation}
The restriction $\abs{b^{\star}_i(z)} \leq \frac{\pi_i}{\beta\zeta}$ therefore implies that $\abs{\chi_{C,D}(\kappa)} \leq \frac{K}{\zeta}$ for some $K>0$. 
Now, since:
\begin{equation}
    \abs{\frac{1}{1+\chi_{V,D}(\kappa)}} \leq 1-\abs{\chi_{V,D}(\kappa)},
\end{equation}
for small enough $\abs{\chi_{V,D}(\kappa)}$, we obtain the following entry-wise upper -bound for any-feasible solution of $C,D$:
\begin{equation}\label{eq:lowerb}
   \abs{C}  \leq K', \abs{D} \leq K',
\end{equation}
for some $K' > 0$. This ensures that the boundedness condition on $C,D$ in Lemma \ref{lemma:block_res} applies. 

Recall the definition of $b^\star(C,D)$ :
\begin{equation}
    b^\star(C,D) = \pi_{u_i} \beta ((C^{-1}+ \operatorname{diag}({b}))^{-1}+ (\operatorname{diag}(D)-z I_p))^{-1}_{p,p}
\end{equation}

Subsequently, we may apply Lemma \ref{lem:diff_inv} to obtain: (See also Lemma \ref{lem:inner_contr} in the proof of Lemma \ref{lemma:block_res})

 \begin{equation}
    \norm{b(C_1,D_1)-b(C_2,D_2)} \leq \frac{K''}{\zeta^2}\norm{C_1, D_1- C_2, D_2},
 \end{equation},
 for some constant $K''>0$
Combining the above with \ref{chi:cont} yields:
\begin{equation}\label{chi:cont}
     \abs{\chi(\kappa,C_1,D_1)-\chi(\kappa, C_2,D_2)} \leq \frac{K'''}{\zeta}\norm{C_1, D_1- C_2, D_2}
 \end{equation},
 for some constant $K'''>0$.
    
\end{proof}

\subsection{Proof of Theorem \ref{thm:det_eq}}

Lemma \ref{lem:contrac} and the Banach-fixed point theorem imply that $F(C,D)$ admit unique fixed points $V^\star(z), \nu^\star(z)$ for $\zeta > C$ and that:
\begin{equation}
    V_d(z) \xrightarrow[a.s]{}V^\star(z), \quad \nu_d(z)\xrightarrow[a.s]
 \nu^\star(z)
\end{equation}
Next, we note that from Proposition \ref{prop:sec_stag}, for each $q \in [k]$, $b^\star(z)_q$ can be expressed as:
\begin{equation}
    b^\star(z)_q  = \sum_{i=1}^p \frac{(v_i^\top e^q)^2}{z-\lambda_i} + \mathcal{O}_\prec(\frac{1}{\sqrt{d}}).
\end{equation}

Theorem \ref{thm:det_eq} then follows by noting that by the standard properties of Stieltjes transforms \cite{bai2008large}, $b^\star(z), V^\star(z),  \nu^\star(z)$ admit unique analytic-continuations to $\mathbb{C}/\mathbb{R}^+$.

\section{Generalization Error}\label{sec:gen}

Having obtained the full-deterministic equivalent in Theorem \ref{thm:det_eq}, we now show how it can be exploited to yield the asymptotic generalization error after a gradient step.

We will proceed as follows:
\begin{itemize}
    \item Use the covariance approximation in Proposition \ref{prop:phi_moments} to identify certain ``statistics" involving the resolvent that characterize the asymptotic generalization error.
    \item Relate these statistics to projections of the extendent resolvent onto certain deterministic matrices. Introduce perturbation terms to extract more complicated functionals.
    \item Replace these statistics with the corresponding quantities obtained through the deterministic equivalent $\mathcal{G}_e$, and where necessary average over $\theta$ to obtain the asymptotic generalization as a function of the parameters $V^\star, \nu^\star, b^\star$ in Theorem \ref{thm:det_eq}.
\end{itemize}

% \change{TO-ADD: We refer to the appendix for the precise expression of:
% \begin{align}
%      \Lambda_{s,m}(\cdot) &=  \left(g(s) - \sum_{q=1}^k c_0(\kappa,\zeta^u_q) \tau_{0,q} - \sum_{q=1}^k c_1(\kappa, \zeta^u_q) \left(s \tau_{1,q} + m \tau_{2,q}\right)\right)^2 \nonumber \\ 
%     &- \mathbb{E}\left[\left(\sum_{q=1}^k c_1(\kappa, \zeta^u_q)\tau_{1,q}\right)^2 + \left(\sum_{q=1}^k c_1(\kappa, \zeta^u_q)\tau_{2,q}\right)^2\right]  + \tau_3 + \tau_4 
% \end{align}}
\subsection{Order Parameters for Generalization Error}

We start by using the covariance approximation \ref{prop:phi_moments} to simplify the expression for the generalization error with an arbitrary fixed choice of $a \in \R^p$

\begin{lemma}\label{lemma:eg_non_un} 

Let  $a \in \R^p$ be a fixed vector such that $\norm{a}/\sqrt{p} = \mathcal{O}(1)$ and let $e^1, \cdots, e^k$ denote the ``spike" directions defined in Equation \eqref{eq:def:eq}.
Define for 
$i \in [k]$:
\begin{equation}\label{eq:tau_observables}
    \tau^d_{0,i} \bydef  \frac{a^\top e_i}{\sqrt{p}}, \qquad \tau^d_{1,i} \bydef \frac{a^\top e_i \odot w^\star}{\sqrt{p}}, \qquad \tau^d_{2} \bydef \frac{ a^\top C \odot WW^\top a}{p} \quad \text{and}\quad \tau^d_4 \bydef \frac{a^\top D a}{p},
\end{equation}
where $C = \Ea{c_1(\kappa, u)c_1(\kappa, u)^\top}$ and $D_{j,j,} = \Eb{\kappa}{\sum_{k\ge 2} c_k^2(\kappa, u_j)}$. Then, the generalization error can be expressed as:
    \begin{equation}\label{eq:eg}
    \begin{aligned}
        \Ea{e_g} &= \Ea{\left[\sigma_\star(s) - \sum_{j=1}^k c_0(\kappa,u_j) \tau_{0,j} - \sum_{j=1}^k c_1(\kappa, u_j) \kappa \tau_{1,j})\right]^2} + \tau_3\\ &- \Ea{(\sum_{j=1}^k c_1(\kappa, u_j)\tau_{1,j})^2 + (\sum_{j=1}^k c_1(\kappa, u_j)\tau_{2,j})^2} + \tau_3  + \mathcal{O}_\prec(d^{-1/2}).
        \end{aligned}
    \end{equation}
\end{lemma}

\begin{proof}

The prediction at a point $x$ under the simplified updated weights is given by:
\begin{equation}
    f(x,\tilde{W},a) = \sum_{j=1}^k \tau_{0,j} \bar{\phi}^j_x + a^\top \tilde{\phi}_x
\end{equation}
The generalization error is then given by:
\begin{align*}
        \Ea{e_g(a)} &= \Ea{[\sigma_\star(s)-f(x,\tilde{W},a)]^2}
\end{align*}
$\Ea{e_g}$ can be equivalent expressed through the following quadratic form applied to the extended features $\phi_\mu^e$:
\begin{align*}
     \Ea{e_g(a)} &= \Ea{(\phi_\mu^e)^\top u_au_a^\top \phi_\mu^e}\\
     & = \Ea{\Tr(u_au_a^\top \phi^e_\mu(\phi_\mu^e)^\top)}\\
\end{align*}
where $u_a = [1,-\tau_0, \tilde{a}] \in \mathbb{R}^{p+k+1}$. By Proposition \ref{prop:phi_moments}, we have:
\begin{equation}
    \Ea{\Tr(u_au_a^\top \phi^e_\mu(\phi_\mu^e)^\top)} = \Ea{\Tr(u_a u_a^\top R^\star_\kappa)} + \mathcal{O}(\frac{\polylog d}{\sqrt{d}}).
\end{equation}

Expanding each of the terms in $R^\star_\kappa$ then yields Equation \eqref{eq:eg}.
\end{proof}
\subsection{Extended Resolvent to Generalization Error}

From Lemma \ref{lemma:eg_non_un}, we note that obtaining the limiting generalization error requires characterizing the functionals $\tau^d_{0,i},\tau^d_{1,i},\cdots$ when $\hat{a}$ is set as the ridge-regression estimator:
\begin{align}
        \label{eq:def:erm}
        \hat{\ba}_{\lambda}&=\underset{\ba\in\mathbb{R}^{p}}{\rm argmin} \sum\limits_{i\in[n]}\left(y_{i}-f(\bx_{i};\ba,\bW^{1})\right)^{2}+\lambda||\ba||_{2}^{2}\notag\\
        &=\left(\Phi^{\top}\Phi+\lambda \bI_{n}\right)^{-1}\Phi^{\top} \by .
    \end{align}

To extract these relevant functionals, we introduce certain perturbation terms in the extended resolvent.
\begin{align}
\label{eq:G_e}
    \vec{G}_e(\rho_1, \rho_2) \bydef \left((\Phi^e)^{\top}(\Phi^e)+   \lambda \vec{I} + \rho_1\vec{L}+\rho_2 \vec{D'}\right)^{-1}\in\mathbb{R}^{(p+1)\times (p+1)},
\end{align}
where the matrices $\vec{D}', \bL\in\mathbb{R}^{(p+1)\times (p+1)}$ are given by:
\begin{align}
    \vec{L} = \begin{bmatrix}
       \vec{0}_{k+1,k+1} & \vec{0}_{k+1, p}\\
        \vec{0}_{p, k+1}&   \vec{C}\odot \vec{W} \vec{W}^{\top} 
    \end{bmatrix},&&
    \vec{D}' = \begin{bmatrix}
        \vec{0}_{k+1, k+1} & \vec{0}_{k+1, p} \\
        \vec{0}_{p,k+1} & \Vec{D}
    \end{bmatrix},
\end{align}
with $C = \Ea{c_1(\kappa, u)c_1(\kappa, u)^\top}$ and $D_{j,j,} = \Eb{\kappa}{\sum_{k\ge 2} c_k^2(\kappa, u_j)}$.
We will demonstrate that the introduction of the perturbation terms $\vec{L},\vec{D}'$ allows extraction of all the relevant functionals in Lemma \ref{lemma:eg_non_un}.

\begin{proposition}\label{prop:gen_e_par}
Let $\hat{a}_\lambda$ denote the ridge-regression minimizer with regularization $\lambda > 0$ i.e:
\begin{align}
        \label{eq:def:erm}
        \hat{\ba}_{\lambda}&=\underset{\ba\in\mathbb{R}^{p}}{\rm argmin} \sum\limits_{i\in[n]}\left(y_{i}-f(\bx_{i};\ba,\bW^{1})\right)^{2}+\lambda||\ba||_{2}^{2}\notag\\
        &=\left(\Phi^{\top}\Phi+\lambda \bI_{n}\right)^{-1}\Phi^{\top} \by .
    \end{align}

Let $\tilde{G}(\rho_1,\rho_2)$ denote the resolvent of the ``bulk" component i.e:
\begin{equation}
    \tilde{G}(\rho_1,\rho_2) = (\tilde{\phi}\tilde{\phi}^\top+\lambda I_n)^{-1}
\end{equation}
Consider the following Schur-complement decomposition of $G_e$:
\begin{equation}\label{eq:Ge_block}
    G_e(\rho_1,\rho_2) = \begin{bmatrix}
        C & -C Q^\top\\
        -Q C & P
    \end{bmatrix},
\end{equation}
where 
\begin{equation}
    P = \tilde{G}(\rho_1,\rho_2) + Q C Q^\top,
\end{equation}
\begin{equation}
    Q = \tilde{G}(\rho_1,\rho_2) \Sigma_{21}/p,
\end{equation}
where $\Sigma_{12} =  \tilde{\Phi}\begin{bmatrix} y & \bar{\Phi}\end{bmatrix} $
with $y \bydef [y^1, y^2, \ldots, y^n]$, $\bar{\Phi} \bydef [\bar{\phi}^1, \bar{\phi}^2, \ldots, [\bar{\phi}^n]$, and 

\begin{equation}\label{eq:c1}
    C^{-1} = \begin{bmatrix}
        \norm{y}^2/p + \lambda & y^\top \bar{\Phi}/p\\
        y^\top \bar{\Phi}/p & \bar{\Phi}^\top \bar{\Phi}/p + \lambda
    \end{bmatrix} - \frac 1 p \begin{bmatrix} y^\top\\
    \end{bmatrix} \tilde{\Phi}^\top G(\rho_1, \rho_2) \tilde{\Phi} \begin{bmatrix} y & \bar{\Phi}\end{bmatrix}.
\end{equation}
Let $\hat{\tau_0}, \hat{\tau_1}_i, \hat{\tau_2}_i, \hat{\tau_3}$ be as defined in Lemma \ref{lemma:eg_non_un}, then: 
\begin{equation}\label{eq:tau_0}
  \hat{\tau}_0 = (C^{-1})_{01}((C^{-1})_{11}+\lambda (1-\frac{1}{p}\operatorname{diag}(\pi_u)))^{-1}
\end{equation}
\begin{equation}
    \hat{\tau}_{1,j} = (\theta \odot e^j)^\top Q \begin{bmatrix}
        1 \\  -\hat{\tau}_0
    \end{bmatrix} + \mathcal{O}(\frac{\hat{\tau}_0}{\sqrt{p}})
\end{equation}
\begin{equation}\label{eq:tau3_asymp}
    \hat{\tau}_2 =  \begin{bmatrix} 1 & -\tau^\ast_0\end{bmatrix} \frac{\partial}{\partial \rho_1} \left[(C)^{-1}\right]_{\rho = 0}\begin{bmatrix} 1 \\ -\hat{\tau}_0\end{bmatrix}+ \mathcal{O}(\frac{\hat{\tau}^2_0}{p})
\end{equation}
and
\begin{equation}\label{eq:tau4_asymp}
    \hat{\tau}_3 = \begin{bmatrix} 1 & -\tau^\ast_0\end{bmatrix} \frac{\partial}{\partial \rho_2} \left[(C)^{-1}\right]_{\rho= 0}\begin{bmatrix} 1 \\ -\hat{\tau}_0\end{bmatrix}+ \mathcal{O}(\frac{\hat{\tau}^2_0}{p}).
\end{equation},
where $\mathcal{O}(\hat{\tau}_0)$ denotes error bounded as $C\norm{\hat{\tau}_0}$ for some constant $C>0$.
\end{proposition}

\begin{proof}
    We decompose $\hat{a}_\lambda$ into projections along $\vec{e}^1, \cdots, \vec{e}^k$ and the orthogonal complement, denoted as $\tilde{a} = (I-\Pi)a$, where $\Pi$ denotes the projection operator along $\vec{e}^1, \cdots, \vec{e}^k$.

The ridge-regression objective can be re-expressed as:
\begin{equation}\label{eq:ridge_H_nonun}
   \mathcal{R}(\tau_0, \tilde{a})= \underset{\tau_0 \in \R^k, \tilde a}{\min}\sum_{\mu = 1}^n \big(y_\mu - \bar{\phi}^\top_\mu \tau_0 - \tilde a^\top  \bar{\phi}_\mu\big)^2 + \lambda \norm{\frac{1}{\sqrt{\pi}}\cdot\tau_0}^2  + \lambda \norm{\tilde a}^2,
\end{equation}
where $\cdot$ denotes element-wise multiplication by  $\pi \in \mathbb{R}^k$, accounting for the variability in the number of entries with value $\zeta^u_q$ for $q \in [k]$. 

The optimality condition can similarly be expressed in terms of $\tau_0, \tilde{a}$.
 Differentiating the objective in Eq. \eqref{eq:ridge_H_nonun} w.r.t $\tau_0$ yields:

\begin{align*}
    \hat{\tau}_0&=(\sum_\mu \bar{\phi}_\mu y_\mu - \sum_{\mu, \nu} \bar{\phi}_\mu \tilde{\phi}_\mu^\top \tilde{G}  \tilde{\phi}_\mu y_\nu)(\sum_\mu \bar{\phi}_\mu \bar{\phi}_\mu^\top - \sum_{\mu, \nu} \bar{\phi}_\mu \tilde{\phi}_\mu \tilde{G}\tilde{\phi}_\mu \bar{\phi}_\mu + \lambda \operatorname{diag}(\pi))^{-1}.
\end{align*}
Similarly, differentiating w.r.t $\tilde{a}$, we obtain:
\begin{equation}
    \sum_\mu \bar{\phi}_\mu \bar{\phi}^\top_\mu \tau_0 + \lambda \operatorname{diag}{(1/\pi)}\tau_0= \textstyle\sum_\mu \bar{\phi}_\mu y_\mu - \textstyle\sum_\mu \bar{\phi}_\mu \tilde{\phi}_\mu  \tilde a
\end{equation}
Simplifying, we obtain that orthogonal component of the minimizer $\tilde a$ is given by:
\begin{equation}\label{eq:a_grad_nonu}
    \tilde a = \tilde{G} \sum_\mu \tilde{\phi}_\mu (y_\mu -  \bar{\phi}^\top_\mu \tau_{0}),
\end{equation}
where $\tilde{G}$ is defined as before:
\begin{equation}
    \tilde{G} \bydef (\sum_{\mu}  \tilde{\phi}_\mu (\tilde{\phi}_\mu)^\top  + \lambda I)^{-1}.
\end{equation}

Therefore, we obtain:
\begin{equation}
   \tilde a = Q \begin{bmatrix}
        1 \\  -\hat{\tau}_0
    \end{bmatrix}, 
\end{equation}
Isolating the contributions from the projections along $e^1, \cdots, e^k$ in $\tau_1,i$, we obtain: 
\begin{align*}
    \frac{a^\top e_i \odot w^\star}{\sqrt{p}} =  \frac{\tilde{a}^\top e_i \odot w^\star}{\sqrt{p}} + \mathcal{O}(\frac{\hat{\tau}_0}{\sqrt{p}})
\end{align*}

Comparing the above with Equation \eqref{eq:c1}, we obtain Equation \eqref{eq:tau_0}. 

The expressions for $\hat{\tau_2}_i, \hat{\tau_3}$ then follow directly from the expressions for $\frac{\partial}{\partial \rho_1} \left[(C)^{-1}\right]_{\rho = 0}, \frac{\partial}{\partial \rho_2} \left[(C)^{-1}\right]_{\rho = 0}$
\end{proof}

\subsection{Deterministic Equivalent for the perturbed Resolvent}

Note that the relations established in Proposition \ref{prop:gen_e_par} still involve the full high-dimensional matrices $X,W$. Our goal now will be to use the deterministic equivalence we established earlier to obtain deterministic limits for each of the quantities.

We begin by incorporating the perturbation terms into our deterministic equivalence. Fortunately, this simply corresponds to rescaling certain terms by constants, as we explain below:

\begin{theorem}\label{thm:pertub}
   Let $V^{\star} \in \mathbb{C}^{k\times k}, \nu^{\star} \in \mathbb{C}^k, b^\star \in \mathbb{C}^k$ be uniquely defined through the following conditions:
\begin{itemize}
    \item $V^{\star},\nu^{\star},b^\star$ satisfy the following set of self-consistent equations:
    \begin{align*}
       &V^{\star}_{qq'}(z) = \Eb{\kappa}{\alpha\frac{c_1(\kappa, \zeta^{u}_{q})c_1(\kappa, \zeta^{u}_{q'}) + \rho_1 + \rho_1 \chi_{V,d,b}}{1+\chi_{V,d,b}\left(\kappa\right)}} \\
    &\nu^{\star}_{q}(z) =\Eb{\kappa}{\sum_{\ell \geq 2}\frac{\alpha c^2_{\ell}(\kappa, \zeta^{u}_{q})}{1+\chi_{V,d,b}(\kappa)} +\rho_2 \sum_{\ell \geq 2} c^2_{\ell}(\kappa, \zeta^{u}_{q})} ..,\\
    & b^\star_q(z) = \pi_{q} \beta ((V^\star)^{-1}+ \operatorname{diag}({b^\star})+ (\operatorname{diag}(\nu^{\star}) -zI_p))^{-1}_{q,q}
    \end{align*}
where $\kappa \sim\mathcal{N}(0,1)$   $(c_{\ell}(\kappa,\zeta))_{\ell >0}$ are defined in~\ref{def:main:shifted_hermite} and
where $\kappa\sim\mathcal{N}(0,1)$, $(c_{\ell}(\kappa,\zeta))_{\ell >0}$ are defined in~\ref{def:main:shifted_hermite} and $(\chi(z;\kappa), L(z))$ read as follows:
    \begin{align*}
        &\beta \chi(z; \kappa) = \sum_{q,q' \in[k]} \psi_{qq'} c_1(\kappa,\zeta^{u}_{q})c_1(\kappa,\zeta^{u}_{q'})+ \sum_{q\in[k]}b^\star_{q}\sum_{\ell\ge 2}  c_{\ell}^2(\kappa,\zeta^{u}_{q}), \\
        &L(z) =\left(V^\star(z)^{-1} + \operatorname{diag}({b^\star(z)})\right)^{-1},
    \end{align*}
    where $ \psi(z) \in \mathbb{R}^{k \times k}$ is defined as:
    \begin{align}
        \psi(z)&=b^\star(z) - L(z) \odot (b^\star(z) (b^\star(z))^\top), \notag \\
    \end{align}
\item $V^{\star}, \nu^{\star}, b^\star$ are analytic mappings satisfying $V^{\star}_{i,j}: \mathbb{C}^+ \rightarrow \mathbb{C}^-$ for $i,j \in [k], \nu^{\star}_i : \mathbb{C}^+ \rightarrow \mathbb{C}^-$ for $i\in [k]$,  $b^\star:\mathbb{C}^+_i \rightarrow \mathbb{C}^+$ for $i\in [k]$.
\item $\abs{b^{\star}_i(z)} \leq \frac{\pi_i}{\beta\zeta}$.

\end{itemize}

Define ${\mathcal{G}}_e(z,\rho_1,\rho_2) \in\mathbb{R}^{(p+1)\times (p+1)}$ as:
\begin{align}
     \label{eq:def:detequiv}
        &{\mathcal{G}}_e(z) =\begin{bmatrix}
            A^\ast_{11}-zI_{k} & (A^\ast_{21})^\top \odot \theta^\top \\
          \theta \odot {A}^{\ast}_{21} & A^\ast_{22}+\alpha S_e^* \odot \theta \theta^\top
        \end{bmatrix}^{-1}, 
\end{align}
where $A^\ast_{11},A^\ast_{12}, S_e$ are identical to Theorem \ref{thm:det_eq}.
 
 Then, for any $z \in \mathbb{C}/\mathbb{R}^+$ and sequence of deterministic matrices ${A} \in \mathbb{C}^{(p+1) \times (p+1)}$ with $\norm{A}_{\operatorname{tr}}=\Tr(({A}{A}^*)^{1/2})$ uniformly bounded in $d$:
\begin{equation}
{\rm Tr}(A G_e(z,\rho_1,\rho_2))  \xrightarrow[d \rightarrow \infty]{a.s} {\rm Tr}(A \mathcal{G}_e(z,\rho_1,\rho_2))). 
\end{equation}
\end{theorem}

\begin{proof}
    We observe that the introduction of perturbation terms $\rho_1, \rho_2$ simply amounts to shifting $V^\star, \nu^\star$ by constants. To see this, first note that  the perturbation matrices $C\odot WW^\top, D$ do not depend on $X$. Therefore the proof of Proposition \ref{prop:first_stag} applies to $G_e(z,\rho_1,\rho_2)$ with the only modification being the inclusion of $L,D'$ along with $-z I$ as constant terms. 
    Subsequently, in the second stage, the structure of $L,D'$ allows them to be directly absorbed into the component $M^\star$ with changes upto constants in $V_d, \nu_d$.
\end{proof}

As a direct corollary, we obtain the deterministic equivalents of the Schur-complement Representations in the following sense:
\begin{corollary}\label{cor:eq_order_par}
Let $\mathcal{G}_e(\rho_1,\rho_2)$ denote the deterministic equivalent for the perturbed resolvent as per Theorem \ref{thm:pertub} with $z$ set as $-\lambda$.
Define the following block-matrix representations for $G_e, \mathcal{G}_e$:
\begin{equation}\label{eq:Ge_block}
    G_e(\rho_1,\rho_2) = \begin{bmatrix}
        C & -C Q^\top\\
        -Q C & P
    \end{bmatrix}.
\end{equation}
\begin{equation}\label{eq:Ge_equivalent_block}
   \mathcal{G}_e(\rho_1,\rho_2) = \begin{bmatrix}
        \mathcal{C} & -\mathcal{C} \mathcal{Q}^\top\\
        - \mathcal{Q} \mathcal{C} & \mathcal{P}
    \end{bmatrix}.
\end{equation}
Then:
\begin{equation}
    C = \mathcal{C} + \cO_\prec{(\frac{1}{\sqrt{d}})},
\end{equation}
Furthermore, for any fixed $\vec{r} \in \R^{p-k}$ with $\norm{\vec{r}}=\mathcal{O}(1)$:
\begin{equation}
   Q\vec{r} = \mathcal{Q} \vec{r}  + \cO_\prec{(\frac{1}{\sqrt{d}})},
\end{equation}
\end{corollary}
\begin{proof}
    The above equivalences follow from Theorem \ref{thm:det_eq} by setting $A$ as matrices acting on the individual blocks of $ G_e(\zeta,\rho)$. Concretely, to obtain $C_{i,j}$ we set $A$ as the matrix with $A_{i,j}=1$ and $0$ otherwise.
\end{proof}

We derive the resulting expressions for $\mathcal{C}, \mathcal{Q}$ below for subsequent usage:
\begin{lemma}\label{lem:eq_schur}
Consider the Schur-decomposition of the equivalent resolvent $\mathcal{G}_e$ (with $z =-\lambda$) defined by Equation \eqref{eq:Ge_equivalent_block} and let $\psi, S$ be as defined in Theorem \ref{thm:pertub}.
The matrices $\mathcal{C}, \mathcal{Q}$ satisfy:
\begin{align*}
    \mathcal{C}^{-1} &= A^\ast_{11} + \lambda I_{k+1} -(A^\ast_{21})^\top ((\psi)^{-1}+\alpha S)^{-1}A^\ast_{21} + \mathcal{O}_\prec({\frac{1}{\sqrt{p}}})
\end{align*}
\begin{equation}
    \mathcal{Q} =  ((\operatorname{diag}(\frac{b^\star}{\pi\beta})_e  -\frac{1}{\beta^2 }(\operatorname{diag}(b^\star) - \psi(V_d, \nu_d))_e  \odot \frac{\theta}{\pi} \frac{\theta}{\pi}^\top  )^{-1}+\alpha S_e \odot\theta \theta^\top)^{-1}  \theta \odot A^\star_{21}  +\mathcal{O}_\prec({\frac{1}{\sqrt{p}}}),
\end{equation}
where in the expression for $\mathcal{C}$, we've further averaged over $\theta$.
\end{lemma}
\begin{proof}

Throughout, we shall use the following simplification:
\begin{equation}\label{eq:theta_sim}
    (e^i \odot \theta) (e^j \odot \theta) = \pi \delta_{i=j}+\mathcal{O}_\prec(\frac{1}{\sqrt{d}})
\end{equation},
which follows from the definition of $e^i,\pi$ and  $\theta_i$ being independent sub-Gaussian random variables.

By Theorem \ref{thm:pertub} and the definition of $\mathcal{C}$, we have:
    \begin{align*}
     \mathcal{C} &= A^\ast_{11} + \lambda I_{k+1}\\ 
     &-\underbrace{(A^\star_{21})^\top \odot \theta^\top ((\operatorname{diag}(\frac{b^\star}{\pi\beta})_e  -\frac{1}{\beta^2 }(\operatorname{diag}(b^\star) - \psi(V_d, \nu_d))_e  \odot \frac{\theta}{\pi} \frac{\theta}{\pi}^\top  )^{-1}+\alpha S_e \odot\theta \theta^\top)^{-1} \theta \odot A^\star_{21}}_T\\
\end{align*},

To simplify the above expression, we apply the matrix-Woodbury identity to the term $T$.  We first recognize from the definition of $\psi$ in Equation \eqref{eq:def_psi} that the following relation holds:
\begin{equation}
(\bV_d^{-1}+\operatorname{diag}(b^\star))^{-1} \odot (\frac{b^\star}{\pi \beta})_e (\frac{b^\star}{\pi \beta})_e^\top  \odot \theta \theta^\top = \frac{1}{\beta^2 }(\operatorname{diag}(b^\star) - \psi(V_d, \nu_d))_e  \odot \frac{\theta}{\pi} \frac{\theta}{\pi}^\top 
\end{equation}

Further, noting that $\alpha S_e \odot\theta \theta^\top$ corresponds to a finite-rank perturbation $E_\theta S_e E_\theta^\top$ along with additional simplifications by Equation \eqref{eq:theta_sim}, we obtain:
\begin{align*}
    T&= (A^\star_{21})^\top \odot \theta^\top ((\operatorname{diag}(\frac{b^\star}{\pi\beta})_e  -\frac{1}{\beta^2 }(\operatorname{diag}(b^\star) - \psi(V_d, \nu_d))_e  \odot \frac{\theta}{\pi} \frac{\theta}{\pi}^\top  ) \theta \odot A^\star_{21}\\
    & - (A^\star_{21})^\top \psi ((\alpha S)^{-1}+\psi)^{-1} \psi A^\star_{21} + \mathcal{O}_\prec(\frac{1}{\sqrt{d}})\\
    &= (A^\star_{21})^\top \psi A^\star_{21}- (A^\star_{21})^\top \psi ((\alpha S)^{-1}+\psi)^{-1} \psi A^\star_{21} + \mathcal{O}_\prec(\frac{1}{\sqrt{d}})\\
    &= (A^\star_{21})^\top ((\alpha S)+\psi^{-1}) A^\star_{21} + \mathcal{O}_\prec(\frac{1}{\sqrt{d}})
\end{align*}

Next, for $\mathcal{Q}$, the expression follows by direction substitution.
Recall from Proposition \ref{prop:gen_e_par},  that $Q=\tilde{G}(\rho_1,\rho_2)\Sigma_{21}/p$. From the structure of $\mathcal{G}(\rho_1,\rho_2)$ in Theorem \ref{thm:pertub}, we obtain that the corresponding matrix $\mathcal{Q}$, in the equivalent $\mathcal{G}(\rho_1,\rho_2)$ is similarly given by:
\begin{equation}
    ((\operatorname{diag}(\frac{b^\star}{\pi\beta})_e  -\frac{1}{\beta^2 }(\operatorname{diag}(b^\star) - \psi(V_d, \nu_d))_e  \odot \frac{\theta}{\pi} \frac{\theta}{\pi}^\top  )^{-1}+\alpha S_e \odot\theta \theta^\top)^{-1}  \theta \odot A^\star_{21}.
\end{equation}
\end{proof}

\subsection{Extracting order-parameters through the  deterministic equivalent}

Armed with the deterministic equivalent for the perturbed resolvent (Theorem \ref{thm:pertub}), it remains to justify and extract the quantities $\tau_0,\tau_1,\tau_2,\tau_3$.

\begin{lemma}\label{lem:pertub}
Let $C^{-1},\mathcal{C}^{-1}$ denote the corresponding blocks of $G_e$ and $\mathcal{G}_e$ respectively. Then, w.h.p as $d \rightarrow \infty$
\begin{equation}
    \frac{\partial}{\partial \rho_1} \left(C^{-1}\right)_{\rho = 0} = \frac{\partial}{\partial \rho_1} \left(\mathcal{C}^{-1}\right)_{\rho = 0} +\mathcal{O}_\prec(\frac{1}{d^{1/4}}) 
\end{equation}
\begin{equation}
    \frac{\partial}{\partial \rho_2} \left(C^{-1}\right)_{\rho = 0} = \frac{\partial}{\partial \rho_2} \left(\mathcal{C}^{-1}\right)_{\rho = 0} + \mathcal{O}(\frac{1}{d^{1/4}}) 
\end{equation}

\end{lemma}
\begin{proof}
We have:
\begin{equation}
  \frac{\partial^2}{\partial \zeta^2} \left(C^{-1}\right)_{\zeta} = \frac 1 p \begin{bmatrix} y^\top\\
    \bar{\Phi}^\top\end{bmatrix} \tilde{\Phi}^\top \tilde{G} (W W^\top )^2 G \tilde{\Phi}\begin{bmatrix} y & \bar{\Phi}\end{bmatrix} + \frac 1 p \begin{bmatrix} y^\top\\
    \bar{\Phi}^\top\end{bmatrix} \tilde{\Phi}^\top \tilde{G} (W W^\top) G ( W W^\top ) \tilde{\Phi} \begin{bmatrix} y & \bar{\phi}\end{bmatrix}.
\end{equation}
Recall that $\norm{G}_\mu < \frac{1}{\gamma}$ and $\norm{W W^\top}  =  \mathcal{O}_\prec(1)$ by Lemma \ref{lem:op_norm}.  By the submultiplicity of operator norms, we have that $ norm{\frac{\partial^2}{\partial \zeta^2} \left(C^{-1}\right)_{\zeta}}$ is uniformly bounded in $\zeta, d$ with high-probability as $d \rightarrow \infty$.
Therefore, applying the mean-value-theorem entry-wise to $C^{-1}$ and $C_\star^{-1}$ yields that for any $\zeta >0$
:\begin{equation}
    \frac{\partial}{\partial \zeta} \left(C^{-1}\right)_{\zeta = 0} = \frac{(C^{-1}(\zeta)-C^{-1}(0)}{\zeta} + K\zeta,
\end{equation}
and 
\begin{equation}
    \frac{\partial}{\partial \zeta} \left(\mathcal{C}^{-1}\right)_{\zeta = 0} = \frac{(\mathcal{C}^{-1}(\zeta)-\mathcal{C}^{-1}(0)}{\zeta} + K\zeta,
\end{equation}
for some constant $K > 0$. 

Combining the above with the bounds from Corollary \ref{cor:eq_order_par} yields:
\begin{equation}
    \frac{\partial}{\partial \rho_1} \left( C^{-1}\right)_{\rho = 0} =  \frac{\partial}{\partial \rho_1} \left(\mathcal{C}^{-1}\right)_{\rho = 0} + K\zeta + \frac{1}{\rho}\mathcal{O}_\prec(\frac{1}{\sqrt{d}}) 
\end{equation}
Therefore, setting $\zeta = \frac{1}{d^{1/4}}$ yields:

\begin{equation}
    \frac{\partial}{\partial \rho_2} \left(C^{-1}\right)_{\rho = 0} =  \frac{\partial}{\partial \rho_2} \left(\mathcal{C}^{-1}\right)_{\rho = 0} + \mathcal{O}_\prec(\frac{1}{d^{1/4}}) 
\end{equation}
\end{proof}

\subsection{Proof of Theorem \ref{thm:gen_e}}

The proof of Theorem \ref{thm:gen_e} follows directly from Lemma \ref{lemma:eg_non_un} combined with the next result, which defines the parameters $\hat{\tau}_0, \hat{\tau}_1, \hat{\tau}_2, \hat{\tau}_3$ using the fixed point parameters $V^\star, \nu^\star$:

\begin{proposition}
Let $\hat{\ba}_{\lambda}=\underset{\ba\in\mathbb{R}^{p}}{\rm argmin} \sum\limits_{i\in[n]}\left(y_{i}-f(\bx_{i};\ba,\bW^{1})\right)^{2}+\lambda||\ba||_{2}^{2}$ denote the ridge-regression estimator after a gradient update to the first layer. Let $\hat{\tau}_0, \hat{\tau}_1, \hat{\tau}_2, \hat{\tau}_3$ be defined as in \ref{lemma:eg_non_un} with $a =\hat{a}_\lambda$ and let $\mathcal{C}^\star(V^\star, \nu^\star), \psi(V^\star, \nu^\star)$ be functions of $V^\star, \nu^\star$ defined in Theorems \ref{thm:det_eq}, Lemma \ref{lemma:block_res} with $z$ set as $-\lambda$  and let $S \in \mathbb{R}^{k \times k}$ be as defined in Theorem \ref{thm:det_eq}.
    Then, with $V^\star, \nu^\star$ being the unique solutions to the fixed-point equations defined in Theorem \ref{thm:gen_e}.

    \begin{equation}
        \hat{\tau}_0 = (\mathcal{C}^{-1})_{01}((\mathcal{C}^{-1})_{11}+\lambda (1-\frac{1}{p}\operatorname{diag}(\pi_u)))^{-1}  + \mathcal{O}_\prec(\frac{1}{\sqrt{d}}) 
    \end{equation}
    \begin{equation}\label{eq:tau1_asymp}
    \begin{pmatrix}
        \hat{\tau}_{1}
    \end{pmatrix} 
     = ((\psi)^{-1}+\alpha  S)^{-1}  \tilde{A^\ast_{21}}\begin{bmatrix}
        1 \\  -\hat{\tau}_{0,1} \\ \vdots \\ -\hat{\tau}_{0,k} 
    \end{bmatrix} + \mathcal{O}_\prec(\frac{1}{\sqrt{d}})
\end{equation}

\begin{equation}\label{eq:tau2_asymp}
    \hat{\tau}_2 =  \begin{bmatrix} 1 & -\hat{\tau}_0\end{bmatrix} \frac{\partial}{\partial \rho_1} \left[(\mathcal{C})^{-1}\right]_{\rho = 0}\begin{bmatrix} 1 \\ -\hat{\tau}_0\end{bmatrix} + \mathcal{O}_\prec(\frac{1}{\sqrt{d}})
\end{equation}
and
\begin{equation}\label{eq:tau3_asymp}
    \hat{\tau}_2 =  \begin{bmatrix} 1 & -\hat{\tau}_0\end{bmatrix} \frac{\partial}{\partial \rho_2} \left[(\mathcal{C})^{-1}\right]_{\rho = 0}\begin{bmatrix} 1 \\ -\hat{\tau}_0\end{bmatrix} + \mathcal{O}_\prec(\frac{1}{\sqrt{d}})
\end{equation}
where:
\begin{equation}
    \mathcal{C}^{-1} = A^\ast_{11} + \lambda I_{k+1} -(A^\ast_{21})^\top ((\psi)^{-1}+\alpha S)^{-1}A^\ast_{21},
\end{equation}
and $\tilde{A^\ast_{21}} \in \mathbb{R}^{k \times k+1}$ is defined analogous to $A^\ast_{21}$ in Theorem \ref{thm:det_eq} but with $u_1, \cdots, u_p$ replaced by $\zeta^u_1, \cdots, \zeta^u_k$ i.e:
\begin{equation}
    A^\ast_{21}[j,:] = \alpha \, \Eb{\kappa}{\frac{c_1(\kappa,\zeta^u_j)}{1 + \chi(z;\kappa)}\kappa \iota^\top}, \  \forall j \in [k]
\end{equation}

\end{proposition}
\begin{proof}
For $\tau^*_0$, the result follows directly from Proposition \ref{prop:gen_e_par}, Corollary \ref{cor:eq_order_par} and Lemma \ref{lem:eq_schur}. This further implies that $\hat{\tau}_0$ has entries $\mathcal{O}(1)$.

Therefore, we may again apply Proposition \ref{prop:gen_e_par} to obtain:

\begin{equation}\label{eq:tau3_asymp}
    \hat{\tau}_2 =  \begin{bmatrix} 1 & -\tau^\ast_0\end{bmatrix} \frac{\partial}{\partial \rho_1} \left[(C)^{-1}\right]_{\rho = 0}\begin{bmatrix} 1 \\ -\hat{\tau}_0\end{bmatrix}+ \mathcal{O}(\frac{1}{p})
\end{equation}
and
\begin{equation}\label{eq:tau4_asymp}
    \hat{\tau}_3 = \begin{bmatrix} 1 & -\tau^\ast_0\end{bmatrix} \frac{\partial}{\partial \rho_2} \left[(C)^{-1}\right]_{\rho= 0}\begin{bmatrix} 1 \\ -\hat{\tau}_0\end{bmatrix}+ \mathcal{O}(\frac{1}{p}).
\end{equation}
 The expressions for $\hat{\tau}_2, \hat{\tau}_3$ in Equations \eqref{eq:tau2_asymp},\eqref{eq:tau3_asymp} are then direct consequences of Corollary \ref{cor:eq_order_par}, Lemma \ref{lem:eq_schur} and 
Lemma \ref{lem:pertub} (to justify differentiating through the deterministic equivalent).
It remains to obtain the resulting expression for  $\hat{\tau}_1$. Applying Proposition \ref{prop:gen_e_par}, and using $\hat{\tau}_0 = \mathcal{O}(1)$,
we obtain:
\begin{equation}
    \hat{\tau}_{1,j} = (\theta \odot e^j)^\top Q \begin{bmatrix}
        1 \\  -\hat{\tau}_0
    \end{bmatrix} + \mathcal{O}(\frac{1}{\sqrt{p}}).
\end{equation}
Next, applying Lemma \ref{lem:eq_schur} and Corollary \ref{cor:eq_order_par} with $r=(\theta \odot e^j)$ yields:
\begin{equation}
   \hat{\tau}_{1,j} =  (\theta \odot e^j)^\top((\operatorname{diag}(\frac{b^\star}{\pi\beta})_e  -\frac{1}{\beta^2 }(\operatorname{diag}(b^\star) - \psi(V_d, \nu_d))_e  \odot \frac{\theta}{\pi} \frac{\theta}{\pi}^\top  )^{-1}+\alpha S_e \odot\theta \theta^\top)^{-1}  \theta \odot A^\star_{21}
\end{equation}
Subsequently, analogous to the derivation of $\mathcal{C}$ in Lemma \ref{lem:eq_schur}, we simplify the above expression using the Woodbury identity to obtain Equation \eqref{eq:tau1_asymp}.
\end{proof}

\section{Proofs of auxilary results}\label{app:extra}
\subsection{Proof of Lemma \ref{lemma:block_res}}

We proceed using the leave-one out argument for the Wishart spectrum. Let  $W_{-i}$ denote the weight matrix with the $i_{th}$ column removed and define:
\begin{equation}
    E_{w_i} \coloneqq w_i \odot [e_1, \cdots, e_k] \in \mathbb{R}^{p \times k}.
\end{equation}

Then:
\begin{equation}
    C_e \odot W W ^\top =  C_e \odot W_{-i} W_{-i}^\top +  C_e \odot w_i w_i^\top.
\end{equation}

Similar to the term $V_e\odot \theta \theta^\top$ term in Section \ref{app:2ndstage}, we observe that $C_e  \odot w_i w_i^\top$ is a rank $k$ matrix generated through the vectors  $e_1, \cdots, e_k$:
\begin{equation}
    C_e \odot w_i w_i^\top = E_{w_i}  C E_{w_i}^\top,
\end{equation}

The Woodbury matrix identity yields:
\begin{equation}\label{eq:woodburry}
    R = R_{-i}-R_{-i}E_{w_i} ((C)^{-1} + E_{w_i}^\top R_{-i} E_{w_i})^{-1} E_{w_i}^\top R_{-i},
\end{equation}
where:
\begin{equation}
   R_{-i} = (C_e \odot W_{-i} W_{-i} ^\top + (\diag(D)_e-z I_p))^{-1}
\end{equation}

We substitute the above into the 
relation:
\begin{equation}
      R (C_e \odot W W ^\top + (\operatorname{diag}(D)_e -zI_p)) = I
\end{equation}
to obtain:
\begin{equation}
   \sum_{i=1}^d  R_{-i} E_{w_i}  C E_{w_i}^\top  - \sum_{i=1}^d R_{-i} E_{w_i} ((C)^{-1} + B_m)^{-1}B_m C E_{w_i}^\top + R(\operatorname{diag}(D)_e -z I_p) = I
\end{equation}

Therefore, we obtain:
\begin{equation}
   d R \Ea{E_{w} C E_{w}^\top}- \sum_{i=1}^d R_{-i} E_{w_i} (C^{-1} + B_m)^{-1}B_m C E_{w_i}^\top + R(\operatorname{diag}(D)_e-z I_p) = I + \Delta
\end{equation}
where $\Delta$ includes error terms and $B_m$ is an $k \times k$ matrix given by:
\begin{equation}\label{eq:b_def}
    B_m=\Ea{E_w^\top  R_{-i} E_w}.
\end{equation}
Bounding the error terms exactly as in Proposition \ref{prop:first_stag}, we obtain that contributions from $\delta$ are of order $\mathcal{O}_\prec (\frac{\norm{A}_\text{Tr}}{\sqrt{d}})$

Now, since $(C^{-1} + B_m)^{-1}B_m = I-(C^{-1} + B_m)^{-1}(C)^{-1}$,we obtain:
\begin{align*}
d R E_{w_i} ((C)^{-1} + B_m)^{-1} E_{w_i}^\top + R(\diag{D}_e-z I_p) = I+\mathcal{O}_\prec (\frac{\norm{A}_\text{Tr}}{\sqrt{d}}).
\end{align*}
which leads to the following deterministic equivalent for $R(C,D)$:
\begin{equation}
    \mathcal{R} = (\mathcal{M}+ (\diag{D}_e-z I_p))^{-1},
\end{equation}
where:
\begin{equation}\label{eq:R_b}
    \mathcal{M} = d\Ea{E_w( C^{-1}+B_m)^{-1}E_w^\top}, 
\end{equation}

To obtain the deterministic equivalent, it remains to next derive a self-consistent equation for the matrix $B_m$ itself.
Substituting Equation \eqref{eq:R_b}
into Equation \eqref{eq:b_def} we obtain that $B_m$ is diagonal with entries satisfying
\begin{equation}
    b_q \coloneqq B_{q,q} = \pi_q\beta (((C)^{-1}+ \operatorname{diag}({b}))^{-1}+ (\operatorname{diag}(D)-z I_p))^{-1}_{q,q},
\end{equation}
for $q \in [p]$.

Lastly, it remains to establish the uniqueness of the fixed points of $b_q$:

\begin{lemma}\label{lem:inner_contr}
Let $F_b:\mathbb{C}^+\rightarrow \mathbb{C}^+$ denote the mapping:
\begin{equation}
    F(b)_q =  \pi_{q} \beta (((C)^{-1}+ \operatorname{diag}({b}))^{-1}+ (\operatorname{diag}(D)-z I_p))^{-1}_{q,q},
\end{equation}
for $q \in [k]$. Then for large enough $\zeta$ and $C,D$ with entries in $\mathbb{C}^{-}$
satisfying $\abs{C}_{ij} \leq K$ for some constant $K$ independent of $\zeta$
:
    \begin{equation}
        \norm{F(b')-F(b)}_2 < \frac{K'}{\zeta^2}\norm{b-b}_2
    \end{equation},
for $b, b'$ satisfying $\abs{b}_q \leq \frac{\pi}{\beta \zeta}$ and some constant $K>0$.
\end{lemma}
\begin{proof}
Lemma \ref{lem:diff_inv} and Lemma \ref{lem:operator-norms} along with the bounds on $b,C, D$ imply:
\begin{equation}
     \norm{F(b')-F(b)}_2 \leq \frac{K_1}{\zeta^2} \norm{((C)^{-1}+ \operatorname{diag}({b}'))^{-1}-((C)^{-1}+ \operatorname{diag}({b}))^{-1}},
\end{equation}
for some constant $K_1 > 0$. Again applying Lemma \ref{lem:diff_inv} and using the bounds on entries of $C$ yields:
\begin{equation}
     \norm{F(b')-F(b)}_2 \leq \frac{K_2}{\zeta^2}\norm{b'-b}_2
\end{equation}
\end{proof}

% \begin{proof}
%     By assum
% \end{proof}

The control of all the error terms and concentration bounds follows ezactly as in Proposition \ref{prop:first_stag} 

\subsection{Equivalence between gradient update and the spiked random feature model}
 \label{sec:app:mappting_sfrm}
 % {\color{red} Note: this section needed to be integrated somewhere into the main text.}
In this section, we discuss the regime $n_0 = \alpha_0 d$ for some constant $\alpha_0$. In this regime, we show that the ``bulk" in the gradient update possesses a non-isotropic component.

 The starting point is the following relationship between the initial weight matrix $W^0$ and the new version $W^1$ after one gradient step:
 \begin{equation}
     W^1 = W^0 + (\eta p) G,
 \end{equation}
 where $G$ is the gradient matrix constructed as
 \begin{equation}
     G = \frac{1}{n \sqrt p } \diag\{a_1, \ldots a_p\} \sigma'(W^0 X) \diag\{y_1, \ldots, y_n\} X^T,
 \end{equation}
 where $X$ is an $d \times n_0$ matrix whose columns are the data vectors in the first batch, $\set{y_i}_{i \in [n]}$ are the corresponding labels, and $\sigma'(\cdot)$ is the derivative of the student activation function. 

 Consider the Hermite expansion of $\sigma$:
 \begin{equation}
     \sigma(x) = c_0 + c_1 h_1(x) + c_2 h_2(x) + \ldots.
 \end{equation}
 We can then write
 \begin{equation}\label{eq:dsigma_decomp}
     \sigma'(x) = c_1 + \sigma'_{>1}(x),
 \end{equation}
 where
 \begin{equation}
     \sigma'_{>1}(x) \bydef \sum_{k \ge 2} \sqrt{k!}c_{k} \, h_{k-1}(x).
 \end{equation}
 Using the decomposition in \eqref{eq:dsigma_decomp}, we can rewrite the weight matrix $W^1$ as
 \begin{equation}
     W^1 = \underbrace{W^0 + \frac{\eta}{n_0} \diag\{\sqrt p a_1, \ldots \sqrt p a_p\} \sigma'_{>1}(W^0 X) \diag\{y_1, \ldots, y_n\} X^T}_{\widetilde W}  + u v^\top,
 \end{equation}
 where $u = c_1 c_1^\star \eta \sqrt{p} a$ and $v = \frac{1}{c^\star_1}Xy / n_0$, where $c^\star_1$ denotes the first Hermite coefficient of the target activation $g$.
 Next, we examine $\widetilde W$, which is the ``bulk'' of the weight matrix $W^1$ after one gradient step.

 Observe that, conditioned on $X$ and $y$, the rows of $\widetilde W$ are independent centered random vectors. For $i \in [n]$, the $i$th row of $\widetilde W$, denoted by $b_i \in \R^d$, is
 \begin{equation}\label{eq:W_tilde_row}
     b_i = w_i + \frac{\eta (\sqrt{p} a_i)}{n_0} X \diag\set{y_1, \ldots, y_n} \sigma'_{>1}(X^\top w_i).
 \end{equation}

 \begin{proposition}
     Let $\Eb{w_i}{\cdot}$ denote the conditional expectation with respect to $w_i$, with $X$, $\set{y_j}_{j \in [n_0]}$ and $a_i$ kept fixed. Then
     \begin{equation} \label{eq:Delta_cov}
 \begin{aligned}
     &d\cdot \Eb{w_i}{b_i b_i^\top} = I + \frac{2\sqrt{2} c_2 \eta(\sqrt{p}a_i)}{n_0} X \diag\set{y_1, \ldots, y_n} X^\top\\
    &\qquad+ 2\eta^2c_2^2 (\sqrt{p}a_i)^2 \left(\frac 1 n X \diag\set{y_1, \ldots, y_{n_0}} X^\top\right)^2  + \frac{6c_3^2 \eta^2(\sqrt{p}a_i)^2}{n_0^2} X y y^\top X^\top\\
     &\qquad+ \left(\sum_{k \ge 3} k! c_k^2\right) \eta^2(\sqrt{p}a_i)^2 (d/n_0) \left(\frac 1 n_0 X \diag\set{y_1^2, \ldots, y_{n_0}^2} X^\top\right) + \Delta,
 \end{aligned}
 \end{equation}
     where $\Delta \in \R^{d \times d}$ is small error term such that
    \begin{equation}\label{eq:Ebb_Delta}
         \norm{\Delta}_\mathsf{op} \prec d^{-1/2}.
     \end{equation}
 \end{proposition}
 \begin{proof}
 From \eqref{eq:W_tilde_row}, 
 \begin{equation}\label{eq:Ebb_1}
 \begin{aligned}
     &\Eb{w_i}{b_i b_i^\top} = \Eb{w_i}{w_i w_i^\top} + \underbrace{\frac{\eta (\sqrt{p} a_i)}{n_0} X \diag\set{y_1, \ldots, y_{n_0}} \Eb{w_i}{\sigma'_{>1}(X^\top w_i) w_i^\top}}_{(*)} + (*)^\top\\ 
     &\quad + \frac{\eta^2 (\sqrt{p} a_i)^2}{n_0^2} X \diag\set{y_1, \ldots, y_{n_0}} \Eb{w_i}{\sigma'_{>1}(X^\top w_i) (\sigma'_{>1}(X^\top w_i))^\top}\diag\set{y_1, \ldots, y_{n_0}} X^\top.
 \end{aligned}
 \end{equation}
 By symmetry, it is straightforward to check that $\Eb{w_i}{w_i w_i^\top} = I/d$. To compute the last three terms on the right-hand side of the above equation, we first make the simplifying assumptions that $w_i \sim \mathcal{N}(0, I/d)$ and that each column of $Z$ has a fixed norm equal to $\sqrt{d}$. Note that these assumptions are asymptotically accurate as $d \to \infty$. For finite $d$, we can absorb the errors introduced by these assumptions into the error matrix $\Delta$. Under the above assumptions, it is easy to check that:
 \begin{equation}
     \Eb{w_i}{\sigma'_{>1}(X^\top w_i) w_i^\top} = \frac{\sqrt{2}c_2}{d} X^\top
 \end{equation}
 and
 \begin{equation}
     \Eb{w_i}{\sigma'_{>1}(X^\top w_i) (\sigma'_{>1}(X^\top w_i))^\top} = \frac{2c_2^2}{d} X^\top X + \frac{6c_3^2}{d} 1 1^\top + \left(\sum_{k \ge 3} k! c_k^3\right) I_n + \widetilde{\Delta},
 \end{equation}
 where $\norm{\widetilde \Delta}_\mathsf{op} \prec d^{-1/2}$. Substituting these estimates into \eqref{eq:Ebb_1} gives us \eqref{eq:Delta_cov}, with
 \begin{equation}
     \Delta = \eta^2(\sqrt{p}a_i)^2 (d/n_0) \left(\frac 1 n_0 X \diag\set{y_1, \ldots, y_n} \widetilde \Delta \diag\set{y_1, \ldots, y_n}X^\top\right).
 \end{equation}

 Let us further evaluate the trace of the expression \eqref{eq:Delta_cov} in the setting of \cite{cui2024asymptotics}, for odd activations (implying in particular $\mu_2=0$), and uniform initializations $\forall i,~~\sqrt{p}a_i=1$. Since
\begin{align}
    \frac{1}{d}\Tr[\frac 1 {n_0} X \diag\set{y_1^2, \ldots, y_n^2} X^\top] &= \frac{1}{d}\sum\limits_{i=1}^d\mathbb{E}_{x\sim\mathcal{N}(0,I_d)}[g(w_\star^\top x)^2x_i^2]\notag\\
    &=\frac{1}{d}\sum\limits_{i=1}^d \left(2w^\star_i\mathbb{E}_{x\sim\mathcal{N}(0,I_d)}[g(w_\star^\top x)g^\prime(w_\star^\top x)x_i]
    +\mathbb{E}_{z\sim\mathcal{N}(0,I_d)}[g(w_\star^\top z)^2]
    \right)\notag\\
    &\asymp\mathbb{E}_{\xi\sim\mathcal{N}(0,1)}[g(\xi)^2]
\end{align}
Thus,
\begin{align}
    \frac{1}{d}\Tr[\Eb{w_i}{b_i b_i^\top}]=1+\left(\sum\limits_{k\ge 3}k!c_k^2\right)\eta^2 \frac{1}{\alpha_0}\mathbb{E}_{\xi\sim\mathcal{N}(0,1)}[g(\xi)^2]=1+\mathbb{E}_{\xi\sim\mathcal{N}(0,1)}[\sigma^\prime_{>1}(\xi)^2]\eta^2 \frac{1}{\alpha_0}\mathbb{E}_{\xi\sim\mathcal{N}(0,1)},
\end{align}
which corresponds to the term $c$ in \cite{cui2024asymptotics} up to conventions.
\end{proof}
\subsection{Proof of Lemma \ref{lemma:weak_correlation}}

\begin{proof}
    We start by rewriting $g = z + w' (s - z^\top w')$, where $s \sim \mathcal{N}(0, 1)$ and $z \sim \mathcal{N}(0, I)$ are independent. It follows that
    \begin{align}
        f_1(w_1^\top g) f_2(w_2^\top g) g^\top w' &= s f_1\big(w_1^\top z + (w_1^\top w')(s - z^\top w')\big)f_2\big(w_2^\top z + (w_2^\top w')(s - z^\top w')\big)\\
        &= s f_1(w_1^\top z) f_2(w_2^\top z) + s f_1(w_1^\top z) f_2'(w_2^\top x) (w_2^\top w')(s - z^\top w')\\
        &\qquad\qquad\qquad+ s f_2(w_2^\top z) f_1'(w_1^\top x) (w_1^\top w')(s - z^\top w') + \mathcal{O}_\prec(d^{-1}).
    \end{align}
Taking expectation, and applying Lemma~\ref{lemma:Hermite}, we get \eqref{eq:f1f2_linear}. The estimate in \eqref{eq:f1f2_quadratic} can be treated similarly. Note that
\begin{equation}
    f_1(w_1^\top g) f_2(w_2^\top g) (\theta^\top g)^2 = s^2 f_1(w_1^\top x) f_2(w_2^\top z) + \mathcal{O}_\prec(d^{-1/2}),
\end{equation}
which immediately leads to \eqref{eq:f1f2_quadratic}. 
\end{proof}

% \subsection{General covariance}

% We first note that by rotational invariance of Gaussians, we may assume without loss of generality that $\Sigma$ is diagonal. We obtain:
% \begin{equation}
%     C_e \odot W\Sigma W ^\top =  C_e \odot W_{-i} W_{-i}^\top +  \sigma_i C_e\odot w_i w_i^\top.
% \end{equation}

% Again using the notation $E_{w_i}= w_i\odot E$, we 
% obtain:
% \begin{equation}
%     C_e \odot w_i w_i^\top = E_{w_i}  C E_{w_i}^\top,
% \end{equation}

% The Woodbury matrix identity yields:
% \begin{equation}
%     R = R_{-i}-\sigma_iR_{-i}E_{w_i} ((C)^{-1} + \sigma_i E_{w_i}^\top R_{-i} E_{w_i})^{-1} E_{w_i}^\top R_{-i}
% \end{equation}

% \begin{equation}
%    \sum_{i=1}^d  \sigma_i R_{-i} E_{w_i}  C E_{w_i}^\top  - \sigma^2_i\sum_{i=1}^d R_{-i} E_{w_i} ((C)^{-1} + \sigma B_m)^{-1}B_m C E_{w_i}^\top + R(D_e+\lambda I_p) = I
% \end{equation}

% Again using  $(C^{-1} + B_m)^{-1}B_m = I-(C^{-1} + B_m)^{-1}(C)^{-1}$,we obtain: 
% \begin{equation}
%     \mathcal{R} = (\mathcal{M}+ (D^*_p+\lambda I_p))^{-1},
% \end{equation}
% where:
% \begin{equation}\label{eq:R_b}
%     \mathcal{M} = d\Eb{w,\sigma}{E_w( C^{-1}+\sigma \operatorname{diag}(b))^{-1}E_w^\top}, 
% \end{equation}

% \begin{equation}
%     b_q = \pi_{u_i} \beta \Eb{\sigma}{(((C)^{-1}+ \sigma\operatorname{diag}({b}))^{-1}+ (\operatorname{diag}(D)+\lambda I_p))^{-1}}_{p,p},
% \end{equation}
%  \end{proof}

 % \section{Additional Numerical Investigation}
 % \label{sec:numerics}